\documentclass[11pt]{article}

\usepackage{fullpage,times,url,bm}

\usepackage{amsthm,amsfonts,amsmath,amssymb,epsfig,color,float,graphicx,verbatim}
\usepackage{algorithm,algorithmic}
\usepackage{bbm}
\usepackage{enumerate}
\usepackage{nicefrac}
\usepackage{thmtools}
\usepackage{thm-restate}

\usepackage[style = alphabetic,citestyle=alphabetic,maxbibnames=99,sorting=nyt,natbib=true,backref=true]{biblatex}
\DefineBibliographyStrings{english}{%
  backrefpage = {Cited on page},
  backrefpages = {Cited on pages},
}

\usepackage{graphicx}
\usepackage{subcaption}

\usepackage{array}

\usepackage[dvipsnames]{xcolor}

\usepackage[colorlinks,linkcolor = RawSienna, urlcolor  = Green, citecolor = Green, anchorcolor = ForestGreen,bookmarks=False]{hyperref}
\newtheorem{theorem}{Theorem}[section]
\newtheorem{proposition}[theorem]{Proposition}
\newtheorem{lemma}[theorem]{Lemma}

\newtheorem{definition}[theorem]{Definition}

\renewcommand{\eqref}[1]{Eq.~(\ref{#1})}

\usepackage{amssymb}

\usepackage{bbm}

\newcommand{\stam}[1]{}

\usepackage{mathtools}

\DeclareMathOperator*{\argmax}{argmax}
\DeclareMathOperator*{\argmin}{argmin}

\newcommand{\reals}{{\mathbb R}}

\newcommand{\diag}{\mathrm{diag}}

\DeclareMathOperator*{\E}{\mathbb{E}}

\newcommand{\inner}[1]{\langle #1 \rangle}
\newcommand{\norm}[1]{\left\|#1\right\|}
\newcommand{\snorm}[1]{\|#1\|} 

\newcommand{\rmax}{R_{\text{max}}} 
\newcommand{\rmin}{R_{\text{min}}}
\newcommand{\rratio}{R}

\makeatletter
\newcommand{\printfnsymbol}[1]{%
  \textsuperscript{\@fnsymbol{#1}}%
}
\makeatother

\newcommand{\N}{\mathbb{N}}

\newcommand{\R}{\mathbb{R}}

\let\P\BP

\let\hat\widehat

\newcommand{\f}{\frac}
\newcommand{\p}{\partial}

\renewcommand{\r}{\right}
\renewcommand{\l}{\left}
\newcommand{\sip}[2]{\langle #1, #2 \rangle}
\newcommand{\ip}[2]{\left\langle #1, #2 \right\rangle}

\newcommand{\ddx}[1]{\frac{\mathrm{d}}{\mathrm{d} #1}}

\newcommand{\ind}{{\mathbbm{1}}}

\newcommand{\iid}{\stackrel{\mathrm{ i.i.d.}}{\sim}}

\newcommand{\summ}[2]{\sum_{#1 = 1}^{#2}}

\newcommand{\sgn}{\operatorname{sign}}

\newcommand{\tr}{{\rm tr}}

\newcommand{\calC}{\mathcal{C}}

\newcommand{\calN}{\mathcal{N}}

\makeatletter
\renewcommand*{\eqref}[1]{%
  \hyperref[{#1}]{\textup{\tagform@{\ref*{#1}}}}%
}
\makeatother

\newcommand{\stablerank}{\mathsf{StableRank}}

\newcommand{\minnormmu}{\min_{q} \snorm{\muq q}}
\newcommand{\maxnormmu}{\max_{q} \snorm{\muq q}}
\newcommand{\maxcorrmu}{\max_{q\neq r} |\sip{\muq q}{\muq r}|}
\newcommand{\pclust}{\mathsf{P}_{\mathsf{clust}}}

\newcommand{\pnoise}{\mathsf{P}}
\newcommand{\poppnoise}{\mathsf{P}_{\mathsf{opp}}}

\newcommand{\pclustnoise}{\mathsf{P}_{\mathsf{clust}}}

\newcommand{\psubgnoise}{\mathsf{P}_{\mathsf{sg}}}

\newcommand{\pgaus}{\mathsf{P}_{\mathsf{gaus}}}
\newcommand{\pz}{\mathsf{P}_{z}}
\newcommand{\pzz}{\mathsf{P}_{z}'}
\newcommand{\px}{\mathsf{P}_{x}}

\newcommand{\tyj}[1]{\tilde{y}^{(#1)}}
\newcommand{\muq}[1]{\mu^{(#1)}}
\newcommand{\iq}[1]{I^{(#1)}}
\newcommand{\iqc}[1]{I^{(#1)}_\calC}
\newcommand{\iqn}[1]{I^{(#1)}_\calN}
\newcommand{\yqi}[1]{y^{(#1)}}
\newcommand{\yq}[1]{y^{(#1)}}
\newcommand{\clusteri}[1]{\mathrm{cluster}(#1)}
\newcommand{\cluster}[1]{\mathrm{cluster}}
\newcommand{\cov}{\Sigma}
\newcommand{\sgnormz}{\sigma_z}

\newcommand{\xk}{[x]_{1}}

\newcommand{\zk}{[z]_{1}}

\newcommand{\xik}{[x_i]_{1}}
\newcommand{\zik}{[z_i]_{1}}

\newcommand{\unif}{\mathsf{Unif}}
\newcommand{\orthog}{p}
\newcommand{\unifclass}{\tau}

\usepackage{enumitem} 
\DeclareMathSizes{9}{8}{7}{5}

\newcommand\numberthis{\addtocounter{equation}{1}\tag{\theequation}}

 \author{
 Spencer Frei\footnote{Equal contribution.}\phantom{$^\ast$}\\
 UC Berkeley\\
 \texttt{frei@berkeley.edu}\\
 \and 
 Gal Vardi$^\ast$\\
 TTI-Chicago and Hebrew University\\ 
 \texttt{galvardi@ttic.edu}
 \and
 Peter L. Bartlett\\
 UC Berkeley and Google\\
 \texttt{peter@berkeley.edu}\\
 \and
 Nathan Srebro\\
 TTI-Chicago\\
 \texttt{nati@ttic.edu}
 \and
{Collaboration on the Theoretical Foundations of Deep Learning (\href{https://deepfoundations.ai}{deepfoundations.ai})}
}

\title{\textbf{Benign Overfitting in Linear Classifiers and Leaky ReLU Networks from KKT Conditions for Margin Maximization}}

\begin{document}
\maketitle
\begin{abstract}

Linear classifiers and leaky ReLU networks trained by gradient flow on the logistic loss have an implicit bias towards solutions which satisfy the Karush--Kuhn--Tucker (KKT) conditions for margin maximization.    In this work we establish a number of settings where the satisfaction of these KKT conditions implies benign overfitting in linear classifiers and in two-layer leaky ReLU networks: the estimators interpolate noisy training data and simultaneously generalize well to test data.  The settings include variants of the noisy class-conditional Gaussians considered in previous work as well as new distributional settings where benign overfitting has not been previously observed.  The key ingredient to our proof is the observation that when the training data is nearly-orthogonal, both linear classifiers and leaky ReLU networks satisfying the KKT conditions for their respective margin maximization problems behave like a 
nearly uniform
average of the training examples.  
\end{abstract}

\section{Introduction}

The phenomenon of `benign overfitting'---referring to settings where a model achieves a perfect fit to noisy training data and still generalizes well to unseen data---has attracted significant attention in recent years. Following the initial experiments of~\citet{zhang2017rethinkinggeneralization}, researchers have sought to understand how this phenomenon can occur despite the long-standing intuition from statistical learning theory that overfitting to noise should result in poor out-of-sample prediction performance.

In this work, we provide several new results on benign overfitting in classification tasks, for both linear classifiers and two-layer leaky-ReLU neural networks.
We consider gradient flow on the empirical risk with exponentially-tailed loss functions, such as the logistic loss.
Under certain assumptions on the data distribution, we prove that gradient flow converges to solutions that exhibit benign overfitting: the predictors interpolate noisy training data and simultaneously generalize well to unseen test data. 
Our results extend existing work 
in two aspects:
First, we prove benign overfitting in two-layer leaky ReLU networks, while existing results do not cover such models.\footnote{\citet{frei2022benign} showed benign overfitting in two-layer networks with \textit{smooth} leaky ReLU activations, as we discuss later.} Second, we characterize benign overfitting in new distributional settings (i.e., assumptions on the data distributions).

The first distributional setting we consider is a noisy sub-Gaussian distribution $(x,y)\sim \psubgnoise$ where labels are generated by a single component of $x$ and are flipped to the opposite sign with probability $\eta$.  We show that if the variance from this component is sufficiently large relative to the variance of the other components, and if the covariance matrix has a sufficiently high rank relative to the number of samples, then 
linear classifiers and leaky ReLU networks trained by gradient flow exhibit benign overfitting. 
In our second distributional setting, we 
consider 
a distribution $\pclustnoise$ where 
inputs
$x$ are drawn uniformly from $k$ nearly-orthogonal clusters, 
and labels are determined by the cluster and are flipped to the opposite sign with probability $\eta$. 
We show that under some assumptions on the scale and correlation of the clusters, 
gradient flow on linear classifiers and leaky ReLU networks produces classifiers which
exhibit benign overfitting. This is a setting not covered by prior work on benign overfitting, 
and it essentially generalizes some previous results on benign overfitting in linear classification \citep{chatterji2020linearnoise,wang2021binary} and neural networks with \emph{smooth} leaky activations~\citep{frei2022benign}.

Our proofs follow by analyzing the implicit bias of gradient flow.
\citet{lyuli2020implicitbias,ji2020directional}
showed that when training homogeneous neural networks with exponentially-tailed loss functions,
gradient flow is biased towards solutions that maximize the margin in parameter space. 
Namely, if the empirical risk reaches a small enough value, then gradient flow 
converges in direction to a solution that satisfies the Karush--Kuhn--Tucker (KKT) conditions for the margin-maximization problem.
We develop new proof techniques which show that in the aforementioned distributional settings, benign overfitting occurs for any solution that satisfies these KKT conditions. 
In a bit more detail, we show that every KKT point in our settings has a linear decision boundary, even in the case of leaky ReLU networks. This linear decision boundary can be expressed by a weighted sum of the training examples, where the weights of all examples are approximately balanced. Using this balancedness property, we are able to prove that benign overfitting occurs.

\subsection*{Related work}

\paragraph{Benign overfitting.}

The benign overfitting phenomenon has recently attracted intense attention and was studied in various settings, such as linear regression \citep{hastie2020surprises,belkin2020two,bartlett2020benignpnas,muthukumar2020harmless,negrea2020defense,chinot2020robustness,koehler2021uniform,wu2020optimal,tsigler2020benign,zhou2022non,wang2022tight,chatterji2021interplay,bartlett2021failures,shamir2022implicit}, kernel regression \citep{liang2020just,mei2019generalization,liang2020multipledescent,mallinar2022benign,rakhlin2019consistency,belkin2018overfittingperfectfitting}, and classification \citep{chatterji2020linearnoise,wang2021binary,cao2021benign,muthukumar2021classification,montanari2020maxmarginasymptotics,shamir2022implicit,frei2022benign,cao2022benign,mcrae2022harmless,liang2021interpolating,thrampoulidis2020theoretical,wang2021benignmulticlass,donhauser2022fastrates}. 
Below we discuss several works on benign overfitting in classification which are most relevant to our results.

In contrast to linear regression, in linear classification the solution to which gradient flow is known to converge, namely, the max-margin predictor, does not have a closed-form expression. Hence, analyzing benign overfitting in linear classification is more challenging.
\citet{chatterji2020linearnoise,wang2021binary} prove benign overfitting in linear classification for a high-dimensional sub-Gaussian mixture model. Our results imply as a special case benign overfitting in sub-Gaussian mixtures similar to their results. 
\citet{cao2021benign} also study benign overfitting in a sub-Gaussian mixture model, but they do not consider label flipping noise.
\citet{muthukumar2021classification} study the behavior of the overparameterized
max-margin classifier in a discriminative classification model with label-flipping noise, by connecting the behavior of the max-margin classifier to the ordinary least squares solution. They show that under certain conditions, all training data points become support vectors of the maximum margin classifier (see also \citet{hsu2021proliferation}).
\citet{montanari2020maxmarginasymptotics} studies a setting where the inputs are Gaussian, and the labels are generated according to a logistic link function. They derive an expression for the asymptotic prediction error of the max-margin linear classifier, assuming the ratio of the dimension and the sample size converges to some fixed positive limit.
\citet{shamir2022implicit} also studies linear classification and proves benign overfitting under a distributional setting which is different from the aforementioned works and from our setting.

Benign overfitting in nonlinear neural networks is even less well-understood.
\citet{frei2022benign} show benign overfitting in 
two-layer networks with \textit{smooth} leaky ReLU activations 
for a high-dimensional sub-Gaussian mixture model; at the end of Section~\ref{sec:clustered} we compare our results with theirs. 
\citet{cao2022benign} study benign overfitting in training a two-layer \emph{convolutional} neural network 
using the logistic loss, but they do not consider label-flipping noise as we do.

\paragraph{Implicit bias.}

The literature on implicit bias in neural networks has rapidly expanded in recent years (see \citet{vardi2022implicit} for a survey).
In what follows, we discuss results that apply either to linear classification using gradient flow, or to nonlinear two-layer networks trained with gradient flow in classification settings.

\citet{soudry2018implicitbias} showed that gradient descent on linearly-separable binary classification problems with exponentially-tailed losses (e.g., the exponential loss and the logistic loss), converges to the maximum $\ell_2$-margin direction. This analysis was extended to other loss functions, tighter convergence rates, non-separable data, and variants of gradient-based optimization algorithms   \citep{nacson2019convergence,ji2018risk,ji2020gradient,gunasekar2018characterizing,shamir2020gradient,ji2021characterizing,nacson2019stochastic,ji2021fast}.

\citet{lyuli2020implicitbias} and~\citet{ji2020directional} showed that homogeneous neural networks (and specifically two-layer leaky ReLU networks, which are the focus of this paper) trained with exponentially-tailed classification losses converge in direction to a KKT point of the maximum-margin problem. 
We note that the aforementioned KKT point may not be a global optimum of the maximum-margin problem \citep{vardi2021margin,lyu2021gradient}.
Recently, \citet{kunin2022asymmetric} extended this result by showing bias towards margin maximization in a broader family of networks called \emph{quasi-homogeneous}.
\citet{lyu2021gradient,sarussi2021towards,frei2023implicit} studied implicit bias in two-layer leaky ReLU networks with linearly-separable data, and proved that under some additional assumptions, gradient flow converges to a linear classifier. Specifically, \citet{frei2023implicit} analyzed the implicit bias in leaky ReLU networks trained with nearly-orthogonal data, and our analysis of leaky ReLU networks builds on their result (see Section~\ref{sec:kkt.nearly.orthogonal} for details). Moreover, implicit bias with nearly-orthogonal data was studied for ReLU networks in \citet{vardi2022gradient}, where the authors prove bias towards networks that are not adversarially robust.
Other works which consider the implicit bias of classification using gradient flow in nonlinear two-layer networks include \cite{chizat2020implicit,phuong2020inductive,safran2022effective,timor2022implicit}.

\section{Preliminaries}

\paragraph{Notation.} We use $\snorm{x}$ to denote the Euclidean norm of a vector $x$, while for matrices $W$ we use $\snorm{W}_F$ to denote its Frobenius norm and $\snorm{W}_2$ its spectral norm.  We use $\ind(z)$ to denote the indicator function, so $\ind(z) = 1$ if $z\geq 0$ and 0 otherwise.  We use $\sgn(z)$ as the function that is $1$ when $z>0$ and $-1$ otherwise.  For integer $n\in \N$, we use $[n] = \{1, \dots, n\}$.  The Gaussian with mean $a$ and variance $\sigma^2$ is denoted $\mathsf N(a, \sigma^2)$, while the multivariate Gaussian with mean $\mu$ and covariance matrix $\cov$ is denoted $\mathsf N(\mu, \cov)$. 
We denote the minimum of two numbers $a,b$ as $a\wedge b$, and the maximum $a\vee b$.  For a vector $x\in \R^d$, we use $[x]_i\in \R$ to denote the $i$-th component of the vector, and $[x]_{i:j}\in \R^{j-i+1}$ as the vector with components $[x]_i, [x]_{i+1}, \dots, [x]_j$.  We use the standard big-Oh  notation $O(\cdot), \Omega(\cdot)$ to hide universal constants, with $\tilde{O}(\cdot), \tilde \Omega(\cdot)$ hiding logarithmic factors.  We refer to quantities that are independent of the dimension $d$, number of samples $n$, the failure probability $\delta$ or number of neurons $m$ in the network as constants. 

\paragraph{The setting.}  
We consider classification tasks where the training data $\{(x_i, y_i)\}_{i=1}^n$ are drawn i.i.d. from a distribution $\pnoise$ over $(x,y)\in \R^d \times \{ \pm 1\}$.
We 
study
two distinct models in this work.  In the first, we consider maximum-margin \textit{linear classifiers} $x\mapsto \sgn(\sip{w}{x})$, which are solutions to the following constrained optimization problem:
\begin{equation}\label{eq:linear.max.margin}
\min_{w\in \R^d} \snorm{w}^2\quad \text{such that for all $i\in [n]$,}\quad y_i \sip{w}{x_i} \geq 1.
\end{equation} 
By~\citet{soudry2018implicitbias}, gradient descent on exponentially-tailed losses such as the logistic loss has an implicit bias towards such solutions.
We shall show that in a number of settings, any solution to Problem~\eqref{eq:linear.max.margin} will exhibit benign overfitting.

As our second model,
we consider two-layer neural networks with leaky ReLU activations, where the first layer $W\in \R^{m\times d}$ is trained but the second layer weights $\{a_j\}_{j=1}^m$ fixed at random initialization:
\begin{equation} \label{eq:leaky.relu.network}
    f(x;W) :=  \summ j m a_j \phi(\sip{w_j}{x}),\quad \phi(q) = \max(\gamma q, q),\quad \gamma \in (0,1).
\end{equation}
For simplicity we assume $m$ is an even number and that for half of the neurons, $a_j=1/\sqrt m$, and the other half of the the neurons satisfy $a_j=-1/\sqrt m$.   We consider a binary classification task with training data $S = \{(x_i, y_i)\}_{i=1}^n \subset \R^d \times \{\pm 1 \}$.   We define the \textit{margin-maximization problem} for the neural network $f(x;W)$ over training data $S$ as
\begin{equation}\label{eq:margin.maximization.problem}
    \min_{W\in \R^{m\times d}} \snorm{W}_F^2 \quad \text{such that for all $i\in [n]$,}\quad  y_i f(x_i;W) \geq 1.
\end{equation}
Recall the definition of the Karush--Kuhn--Tucker (KKT) conditions for non-smooth optimization problems (cf. \citet{lyuli2020implicitbias,dutta2013approximate}).
Let $h: \reals^p \to \reals$ be a locally Lipschitz function. The Clarke subdifferential \citep{clarke2008nonsmooth} at $\theta \in \reals^p$ is the convex set
\[
	\partial^\circ h(\theta) := \text{conv} \left\{ \lim_{s \to \infty} \nabla h(\theta_s) \; \middle| \; \lim_{s \to \infty} \theta_s = \theta,\; h \text{ is differentiable at } \theta_s  \right\}~.
\]
If $h$ is continuously differentiable at $\theta$ then $\partial^\circ h(\theta) = \{\nabla h(\theta) \}$.
Given locally Lipschitz functions $h, g_1, \dots, g_n:\R^p \to \R$, we say that $\theta \in \reals^p$ is a \emph{feasible point} of the problem 
\[ \min h(\theta) \quad \text{s.t.}\quad \text{for all $n\in [N]$,}\,\, g_n(\theta) \leq 0,\]
 if $\theta$ satisfies $g_n(\theta) \leq 0$ for all $n \in [N]$. We say that a feasible point $\theta$ is a \emph{KKT point} if there exists $\lambda_1,\ldots,\lambda_N \geq 0$ such that 
\begin{enumerate}
	\item $0 \in \partial^\circ h(\theta) + \sum_{n \in [N]} \lambda_n \partial^\circ g_n(\theta)$;
	\item For all $n \in [N]$ we have $\lambda_n g_n(\theta) = 0$.
\end{enumerate}
We shall show that in a number of settings, any KKT point of Problem~\eqref{eq:margin.maximization.problem} will generalize well, even when a constant fraction of the training labels are uniformly random labels.  Since any feasible point of Problem~\eqref{eq:margin.maximization.problem} interpolates the training data, this implies the network exhibits \textit{benign overfitting}.

KKT points of Problem~\eqref{eq:margin.maximization.problem} appear naturally in the training of neural networks.  For a loss function $\ell:\R\to [0,\infty)$ and for parameters $W$ of the neural network $f(x;W)$, define the empirical risk under $\ell$ as
\[ \hat L(W) := \f 1 n \summ i n \ell\big(y_i f(x_i;W)\big).\]
\textit{Gradient flow} for the objective function $\hat L(W)$ is the trajectory $W(t)$ defined by an initial point $W(0)$, and is such that $W(t)$ satisfies the differential equation $\ddx t W(t) \in - \p^{\circ} \hat L(W(t))$ with initial point $W(0)$.  Since the network $f(x;\cdot)$ is 1-homogeneous, recent work by~\citet{lyuli2020implicitbias} and~\citet{ji2020directional} show that if $\ell$ is either the exponential loss $\ell(q)=\exp(-q)$ or logistic loss $\ell(q) = \log(1+\exp(-q))$, then provided there exists a time $t_0$ for which $\hat L(W(0)) < \log(2)/n$, gradient flow converges in direction to a KKT point of Problem~\eqref{eq:margin.maximization.problem}, in the sense that for some KKT point $W^*$ of Problem~\eqref{eq:margin.maximization.problem} it holds that
$\f{ W(t)}{\snorm{W(t)}} \to \f{ W^*}{\snorm{W^*}}$.  Thus, although there exist many neural networks which could classify the training data correctly, if gradient flow reaches a point with small enough loss then it will only produce networks which converge in direction to networks which satisfy the KKT conditions of Problem~\eqref{eq:margin.maximization.problem}.  Note that this need not imply that $W(t)$ converges in direction to a global optimum of Problem~\eqref{eq:margin.maximization.problem}~\citep{vardi2021margin,lyu2021gradient}. 
This is in contrast to the margin-maximization problem in linear classification given in Eq.~\eqref{eq:linear.max.margin}, where the constraints and objective function are linear, and hence the KKT conditions are necessary and sufficient for global optimality.

\section{Properties of KKT Points for Nearly Orthogonal Data}\label{sec:kkt.nearly.orthogonal}

In this section we show that when the training data is nearly-orthogonal (in a sense to be formalized momentarily), then the decision boundaries of both (i) KKT points of the \textit{linear} max-margin problem~\eqref{eq:linear.max.margin} and (ii) KKT points of the \textit{nonlinear} leaky ReLU network~\eqref{eq:margin.maximization.problem} take the form of a weighted-average estimator $w =\summ i n s_i y_i x_i$ where $\{s_i\}_{i=1}^n$ are strictly positive and all of the same order, namely, $w$ is a \textit{nearly uniform average} of the training data. 
We will use this property in the next sections to show benign overfitting under certain distributional assumptions.
We begin with our definitions of $\orthog$-orthogonality and $\unifclass$-uniform classifiers.

\begin{definition}\label{def:nearlyorthogonal}
Denote $\rmin^2 = \min_i \snorm{x_i}^2$, $\rmax^2 = \max_i \snorm{x_i}^2$, and $\rratio^2 = \rmax^2/\rmin^2$.  We call the training data \emph{$\orthog$-orthogonal} if $\rmin^2 \geq \orthog \rratio^2 n \max_{i\neq j} |\sip{x_i}{x_j}|$.  
\end{definition}
Clearly, if the training data is exactly orthogonal then it is $\orthog$-orthogonal for every $p>0$.  In contrast to exact orthogonality, $\orthog$-orthogonality allows for the possibility that training data sampled i.i.d. from a broad class of distributions is $\orthog$-orthogonal, as we shall see later.  

\begin{definition}\label{def:uniformclassifier}
    We say that $w \in \R^d$ is \emph{$\unifclass$-uniform w.r.t. $\{(x_i,y_i)\}_{i=1}^n \subset \R^d \times \{-1,1\}$} if $w = \sum_{i=1}^n s_i y_i x_i$, where the coefficients $\{s_i\}_{i=1}^n$ are strictly positive and $\frac{\max_i s_i}{\min_i s_i} \leq \unifclass$.
\end{definition}

Our first lemma shows that if the training data is $\orthog$-orthogonal for large $\orthog$ and the norms of the training examples are all of the same order, then the linear max-margin classifier is given by a $\unifclass$-uniform vector.

\begin{proposition} \label{prop:bound.lambdas.max.margin.v2}
Suppose the training data 
are $\orthog$-orthogonal for $\orthog \geq 3$.  Denote 
$\rratio^2 = \max_{i,j} \nicefrac{\snorm{x_i}^2}{\snorm{x_j}^2}$.
Let $\hat{w} = \argmin \{ \snorm{w}^2: y_i\sip w{x_i}\geq 1\, \forall i\}$ be the max-margin linear classifier. Then, $\hat{w}$ is $\unifclass$-uniform w.r.t. $\{(x_i,y_i)\}_{i=1}^n$ for 
\[
    \unifclass
    = \rratio^2 \left(1 + \frac{2}{\orthog \rratio^2 - 2} \right)~.
\]
\end{proposition}

The proof for this proposition comes from an analysis of the KKT conditions for the max-margin problem and is provided in Appendix~\ref{appendix:lambdas.linear}.  Observe that as $\orthog \to \infty$ and as $\rratio^2 \to 1$, we see that the linear max-margin becomes proportional to $\summ i n y_i x_i$, i.e. the classical sample average estimator.

Next, we show that when the training data are $\orthog$-orthogonal for large enough $\orthog$, then any KKT point of the leaky ReLU network margin maximization problem~\eqref{eq:margin.maximization.problem} has the same decision boundary as a $\unifclass$-uniform linear classifier, despite the fact that two-layer leaky ReLU networks are in general nonlinear.  
The proof relies on a recent work by~\citet{frei2023implicit}, and is given in Appendix~\ref{appendix:lambdas.leaky}.

\begin{restatable}{proposition}{kktleakynearlyorthogonal}
\label{prop:kkt.leaky.nearly.orthogonal}
Denote 
$\rratio^2 = \max_{i,j} \nicefrac{\snorm{x_i}^2}{\snorm{x_j}^2}$.
Let $f$ denote the leaky ReLU network~\eqref{eq:leaky.relu.network} and let $W$ denote a KKT point of Problem~\eqref{eq:margin.maximization.problem}.  Suppose the training data are $\orthog$-orthogonal for $\orthog \geq 3 \gamma^{-3}$.  Then, there exists $z \in \R^d$ such that for any $x\in \R^d$,
\begin{equation*}
 \sgn\l( f(x;W) \r) = \sgn \l(\ip{z}{x}\r),
\end{equation*}
and $z$ is $\unifclass$-uniform w.r.t. $\{(x_i,y_i)\}_{i=1}^n$ for 
\[
    \unifclass
    = \frac{\rratio^2}{\gamma^2} \left(1 + \frac{2}{\gamma \orthog \rratio^2 - 2} \right)~.
\]
Moreover, for any initialization $W(0)$, gradient flow on the logistic or exponential loss converges in direction to such a KKT point.
\end{restatable}

Proposition~\ref{prop:kkt.leaky.nearly.orthogonal}
 identifies an explicit formula for the limiting behavior of a neural network classifier trained by gradient flow in a \textit{non-convex} setting.  It is worth emphasizing that Proposition~\ref{prop:kkt.leaky.nearly.orthogonal} does not make any assumptions on the width of the network or the initialization, and thus the characterization holds for neural networks in the feature-learning regime. Finally, note that as $p\to \infty$ and $\rratio^2 \to 1$, KKT points of Problem~\eqref{eq:margin.maximization.problem} have the same decision boundary as a $\tau$-uniform classifier for $\tau \to \gamma^{-2}$.  In particular, if additionally the leaky parameter $\gamma\to 1$, the KKT points of leaky ReLU network margin-maximization problems become proportional as the sample average $\summ i n y_i x_i$, just as the linear max-margin predictor does. 

Putting Proposition~\ref{prop:bound.lambdas.max.margin.v2} and~\ref{prop:kkt.leaky.nearly.orthogonal} together, we see that by understanding the behavior of $\unifclass$-uniform classifiers $x\mapsto \sgn\l(\sip{\summ i n s_i y_i x_i}{x}\r)$, we can capture the behavior of both linear max-margin estimators as well as those of leaky ReLU networks trained by gradient flow with nearly-orthogonal data. In the following sections, we describe two distributional settings where we show that this estimator can exhibit \textit{benign overfitting}: 
it achieves 0 training error on noisy datasets while simultaneously achieving test error near the noise rate.

\section{Benign Overfitting for Sub-Gaussian Marginals}
In this section we consider a distribution $\psubgnoise$ over $(x,y)$ such that $x$ has independent sub-Gaussian components, with a single high-variance component which determines the label $y$, while the remaining components of $x$ have small variance.  Let $\px$ be a distribution over $\R^d$.  We assume the covariates $x\sim \px$ are mean-zero with covariance matrix $\cov = \E_{x\sim \px}[xx^\top]$ satisfying $\cov = \diag(\lambda_1,\dots, \lambda_d)$ where $\lambda_1\geq \cdots \geq \lambda_d$.  We assume that $z:=\cov^{-1/2}x \sim \pz$ where $\pz$ is a sub-Gaussian random vector with independent components and sub-Gaussian norm at most $\sgnormz$ (see~\citet{vershynin.high} for more details on sub-Gaussian distributions).  Given $x\sim \px$, labels are generated as follows.  For some label noise parameter $\eta \in (0,1/2)$, we have $y = \sgn(\xk)$ with probability $1-\eta$ and $y = -\sgn(\xk)$ with probability $\eta$, where $\xk$ denotes the first component of $x$.  Finally, we assume that for some absolute constant $\beta>0$, 
we have $\P(|\zk| \leq t) \leq \beta t$ 
for all $t \geq 0$.  In the remainder, we will assume that $\sgnormz, \eta,$ and $\beta$ are absolute constants, and our results will hold provided $d$ and $n$ are large enough relative to these and other universal constants.

The reader may be curious about the requirement that $\P_{z\sim \pz} (|\zk| \leq t) \leq \beta t$.  This is a technical assumption that ensures that the `signal' in the model is large as it prevents the possibility that the mass of $\zk$ is highly concentrated near zero. 
Additionally, note that this assumption is satisfied if the distribution of either $z$ or $\zk$ is (isotropic) log-concave
by the anti-concentration property of isotropic log-concave distributions~\citep[Theorem 5.1 and Theorem 5.14]{lovasz}.\footnote{For $z\sim \pz$ where $\pz$ is log-concave and isotropic, \citet[Theorem 5.1]{lovasz} implies the one-dimensional marginal $[z]_1$ is isotropic and log-concave. Theorem 5.14 of the same reference shows that the density function of the (one-dimensional) $[z]_1$ is bounded from above by a constant, which implies $\P(|\zk|\leq t) \leq \beta t$ for an absolute constant $\beta>0$.} This assumption also implies that $\E[|\zk|] \geq 1/(4\beta)$, since $\P(|z| \geq 1/(2\beta)) \geq 1/2$ by taking $t=1/(2\beta)$.  
We can in principle accommodate more general conditions, such as $\P(|\zk| \leq t ) \leq \beta t^p$ for some $p>0$; this is a type of `soft margin' condition which has been utilized in previous work on learning noisy halfspaces~\citep{frei2020halfspace,frei2021provable}.

We assume access to $n$ i.i.d. training examples $\{(x_i,y_i)\}\iid \psubgnoise$.  For a desired probability of failure $\delta \in (0,1/2)$, we make the following assumptions on the problem parameters for a sufficiently large constant $C>1$. 
\begin{enumerate}[label=(SG\arabic*)]
\item \label{a:samples.sg} The number of samples satisfies $n\geq C \log(6/\delta)$.
\item \label{a:srank.sg} The covariance matrix satisfies $\stablerank(\cov_{2:d}) \geq C \log(6n/\delta)$, where $\cov_{2:d}$ denotes the matrix $\diag(\lambda_2, \dots, \lambda_d)$.   
\item \label{a:covnlogn} The covariance matrix satisfies $\f{ \tr(\cov)}{ \sqrt{\tr(\cov^2)}} \geq C n \log(6n^2/\delta)$.
\end{enumerate} 
We remind the reader that the stable rank of a matrix $M\in \R^{m\times d}$ is $\stablerank(M):=\snorm{M}_F^2/\snorm{M}_2^2$.  
We note that the quantity $\tr(\cov)/\sqrt{\tr(\cov^2)}$ in~\ref{a:covnlogn} has appeared in previous work on benign overfitting: it is the square root of the ``effective rank'' $R_0(\cov)$ from~\citet{bartlett2020benignpnas}.  Indeed, this quantity is large if $\sqrt{\stablerank(\cov^{1/2})}$ is large, since 
\begin{align*}
\f{ \tr(\cov)}{\sqrt{\tr(\cov^2)}} &\geq \f{ \tr(\cov)}{\sqrt{\snorm{\cov}_2 \tr(\cov)}} = \f{ \sqrt{\tr(\cov)}}{\snorm{\cov^{1/2}}_2} = \f{\snorm{\cov^{1/2}}_F}{\snorm{\cov^{1/2}}_2} = \sqrt{\stablerank(\cov^{1/2})}.
\end{align*}
Thus the Assumption~\ref{a:covnlogn} can be roughly understood as requiring that the 
matrix $\cov^{1/2}$
has sufficiently large rank.  
Additionally, we note that it is possible to have $\stablerank(\cov_{2:d}) = \Theta(1)$ while $\tr(\cov)/\sqrt{\tr(\cov^2)}=\Theta(\sqrt d)$ (take $\cov = \diag(\sqrt d, \sqrt d, 1,\dots, 1)$), and it is also possible for $\stablerank(\cov_{2:d})=\Theta(\sqrt d)$ while $\tr(\cov)/\sqrt{\tr(\cov^2)} = \Theta(1)$ (take $\cov = \diag(d, d^{1/4}, 1, \dots, 1)$).  Thus the Assumptions~\ref{a:srank.sg} and~\ref{a:covnlogn} are independent.

Our first lemma states that as the constant $C$ in the preceding assumptions becomes larger, the training data becomes more orthogonal.  
\begin{restatable}{lemma}{sgnearlyorthogonal}\label{lemma:sg.nearlyorthogonal}
There exists an absolute constant $C_1>0$ (depending only on $\sgnormz$) such that for every large enough constant $C>0$, for any $\delta \in (0,\nicefrac 12)$, under Assumptions~\ref{a:samples.sg} through~\ref{a:covnlogn} (defined for these $C$ and $\delta$), with probability at least $1-2\delta$ over $\psubgnoise^n$, the training data is $\nicefrac C{C_1}$-orthogonal, and $\max_{i,j} \nicefrac{\snorm{x_i}^2}{\snorm{x_j}^2} \leq (1 + \nicefrac {C_1}{\sqrt{C}})^4$. 
\end{restatable}
The proof of Lemma~\ref{lemma:sg.nearlyorthogonal}, as well as all proofs for this section, appears in Appendix~\ref{appendix:subgaussian}.
Recall from Propositions~\ref{prop:bound.lambdas.max.margin.v2} and~\ref{prop:kkt.leaky.nearly.orthogonal} that for $\orthog$-orthogonal training data, as $\rratio^2 = \max_{i,j} \nicefrac{\snorm{x_i}^2}{\snorm{x_j}^2} \to 1$ and $p\to \infty$, 
solutions
the linear max-margin problem~\eqref{eq:linear.max.margin} become $\unifclass$-uniform for $\unifclass\to 1$.  Similarly, KKT points of the leaky ReLU max margin problem behave like $\unifclass$-uniform linear classifiers for $\unifclass\to \gamma^{-2}$ as $p\to \infty$ and $\rratio\to 1$.  In our main theorem for this section, we show that $\unifclass$-uniform linear classifiers exhibit benign overfitting.    We remind the reader that we refer to quantities that are independent of the dimension $d$, number of samples $n$, the failure probability $\delta$ or number of neurons $m$ in the network as constants.

\begin{restatable}{theorem}{subgaussiansumsiyixi}\label{thm:subgaussian.sumsiyixi}
Let $\unifclass \geq 1$ be a constant, and suppose $\eta \leq \f 1 {2\unifclass} - \Delta$ for some absolute constants $ \eta, \Delta>0$.  There exist constants $C, C'>0$  (depending only on $\eta, \sgnormz, \beta, \unifclass$, and $\Delta$) such that for any $\delta \in (0,\nicefrac 17)$, under Assumptions~\ref{a:samples.sg} through~\ref{a:covnlogn} (defined for these $C$ and $\delta$), with probability at least $1-7\delta$ over $\psubgnoise^n$, if $u\in \R^d$ is $\unifclass$-uniform w.r.t. 
$\{(x_i,y_i)\}_{i=1}^n$, 
then
\begin{align*}
\text{for all $k\in [n]$}, \quad y_k &= \sgn\big(\sip{u}{x_k}\big), \qquad
\text{while simultaneously,} \\
 \eta \leq \P_{(x,y)\sim \psubgnoise}\big(y \neq  \sgn (\sip{u}{x}) \big) &\leq \eta + C'\sqrt{  \f { \tr(\cov_{2:d}^2)} {\lambda_1^2} } \l( 1 +  \sqrt{0 \vee \f 12\log\l( \f{\lambda_1^2}{ \tr(\cov_{2:d}^2)}\r)}\r). 
\end{align*}
In particular, if $\tr(\cov_{2:d}^2)/\lambda_1^2 = o(1)$, then the linear classifier $x \mapsto \sgn(\sip ux)$ exhibits benign overfitting. 
\end{restatable}

Theorem~\ref{thm:subgaussian.sumsiyixi} shows that any $\unifclass$-uniform estimator 
will exhibit benign overfitting, with the level of noise tolerated determined by the quantity $\unifclass$.   Moreover, by considering the $1$-uniform estimator $\summ in y_i x_i$, we see that there exists an estimator which can tolerate noise levels close to $\nicefrac 12$.   

Using Lemma~\ref{lemma:sg.nearlyorthogonal} and Proposition~\ref{prop:bound.lambdas.max.margin.v2},  we can use Theorem~\ref{thm:subgaussian.sumsiyixi} to characterize the linear max-margin predictor. 
\begin{restatable}{corollary}{subgaussianlinearmaxmargin}\label{cor:subgaussian.linearmaxmargin}
Suppose $0<\eta \leq 0.49$.  There exist constants $C, C'>0$ such that for any $\delta \in (0, \nicefrac 19)$, under Assumptions~\ref{a:samples.sg} through~\ref{a:covnlogn} (defined for these $C$ and $\delta$), with probability at least $1-9\delta$ over $\psubgnoise^n$, the max-margin linear classifier $w = \argmin \{ \snorm{w}^2: y_i \sip{w}{x_i}\geq 1 \, \forall i\}$ satisfies
\begin{align*}
\text{for all $k\in [n]$}, \quad y_k &= \sgn\big(\sip{w}{x_k}\big),  \qquad
\text{while simultaneously,} \\
\eta \leq \P_{(x,y)\sim \psubgnoise}\big(y \neq  \sgn (\sip{w}{x}) \big) &\leq \eta + C'\sqrt{  \f { \tr(\cov_{2:d}^2)} {\lambda_1^2} } \l( 1 +  \sqrt{0 \vee \f 12\log\l( \f{\lambda_1^2}{\tr(\cov_{2:d}^2)}\r)}\r). 
\end{align*} 
In particular, if $\tr(\cov^2_{2:d}) / \lambda_1^2 =o(1)$ then $w$ exhibits benign overfitting. 
\end{restatable}
The proof of Corollary~\ref{cor:subgaussian.linearmaxmargin} is the result of a simple calculation (for completeness it is provided in Appendix~\ref{appendix:subgaussian}): By Theorem~\ref{thm:subgaussian.sumsiyixi}, we can tolerate noise rates $\eta$ close to $\f 12$ if $\unifclass$ is close to one.  By Lemma~\ref{lemma:sg.nearlyorthogonal}, as $C$ gets larger the training data becomes more orthogonal and the ratio of the norms of the examples becomes closer to one.  By Proposition~\ref{prop:bound.lambdas.max.margin.v2} this implies $\unifclass \to 1$ as $C$ increases.

We can similarly use Lemma~\ref{lemma:sg.nearlyorthogonal} and Proposition~\ref{prop:kkt.leaky.nearly.orthogonal} to show that KKT points of the max-margin problem for leaky ReLU networks from Problem~\eqref{eq:margin.maximization.problem} also exhibit benign overfitting.  

\begin{restatable}{corollary}{subgaussianleakykkt}
\label{cor:subgaussian.leakykkt}
Suppose that $0 < \eta \leq \f{49\gamma^2}{100}$.
There exist  constants $C, C'>0$ such that for any $\delta \in (0,\nicefrac 19)$, under Assumptions~\ref{a:samples.sg} through~\ref{a:covnlogn} (defined for these $C$ and $\delta$), with probability at least $1-9\delta$ over $\psubgnoise^n$,  any KKT point $W$ of Problem~\eqref{eq:margin.maximization.problem} satisfies 
\begin{align*}
\text{for all $k\in [n]$}, \quad y_k &= \sgn\big( f(x_k;W)\big), \quad \text{while simultaneously,}\\
\eta \leq \P_{(x,y)\sim \psubgnoise}\Big(y \neq  \sgn \big( f(x;W) \big) \Big) &\leq \eta +  C'\sqrt{  \f { \tr(\cov_{2:d}^2)} {\lambda_1^2} } \l( 1 +  \sqrt{0 \vee \f 12\log\l( \f{\lambda_1^2}{\tr(\cov_{2:d}^2)}\r)}\r). 
\end{align*}
In particular, if $\tr(\cov^2_{2:d})/\lambda_1^2 = o(1)$ then the neural network $f(x;W)$ exhibits benign overfitting.  
Moreover, for any initialization $W(0)$, gradient flow converges in direction to a network satisfying the above.  
\end{restatable}
The proof of Corollary~\ref{cor:subgaussian.leakykkt} similarly requires a small calculation which we provide in Appendix~\ref{appendix:subgaussian}.  It is noteworthy that the only difference in the behavior of KKT points of the leaky ReLU max-margin problem~\eqref{eq:margin.maximization.problem} and the linear max-margin~\eqref{eq:linear.max.margin} is the level of noise that is tolerated: in the leaky ReLU case, smaller leaky parameters $\gamma$ result in less noise tolerated, and as $\gamma\to 1$, we recover the behavior of the linear max-margin predictor from Corollary~\ref{cor:subgaussian.linearmaxmargin}.  Additionally, the generalization bound in Corollary~\ref{cor:subgaussian.leakykkt} does not depend on the number $m$ of neurons in the network.  

From the above results, we see that in order for benign overfitting to occur in either the linear max-margin classifier or in two-layer leaky ReLU networks trained by gradient flow, the data needs to simultaneously satisfy two constraints: (1) the covariance matrix is sufficiently high rank in the sense of Assumptions~\ref{a:srank.sg} and~\ref{a:covnlogn}, and (2) the variance in the first coordinate must be large relative to the variance of the last $d-1$ coordinates.  There is a tension here as can be seen by considering the covariance matrix $\cov = (\xi, 1, \dots, 1)$ for $\xi\geq 1$: as $\xi \to \infty$, $\tr(\cov)/\sqrt{\tr(\cov^2)} \to 1$, and hence as the signal-to-noise ratio $\lambda_1^2/\tr(\cov^2_{2:d})$ increases, it becomes more difficult to satisfy assumption~\ref{a:covnlogn}.  However, it is indeed possible to satisfy both (1) and (2).  
Consider the distribution $\pgaus$ over $(x,y) \in \R^d \times \{ \pm 1\}$ where $x\sim \mathsf N(0,\cov)$ with covariance matrix $\cov = \diag (d^\rho, 1, \dots, 1)$ for some $\rho >0$, and where $y = \sgn([x]_1)$ with probability $1-\eta$ and $y = -\sgn([x]_1)$ with probability $\eta$ for some constant $\eta >0$. In the following corollary, we show that if $\rho \in (1/2, 1)$, then (1) and (2) are satisfied and so KKT points of the leaky ReLU max-margin problem~\eqref{eq:margin.maximization.problem} exhibit benign overfitting (an analogous result for the linear max-margin classifier holds as well).

\begin{restatable}{corollary}{pgausresult}
\label{cor:pgausresult}
Suppose $0 < \eta \leq \f{49\gamma^2}{100}$.  Then for the distribution $\pgaus$, for any $\delta \in (0, \nicefrac 19)$, if  $\rho \in \l( \nicefrac 12, 1\r)$, $d= \tilde{\Omega}(n^{1/(1-\rho)})$, and $n=\tilde \Omega(1)$, then Assumptions~\ref{a:samples.sg} through~\ref{a:covnlogn} are satisfied.  Moreover, with probability at least $1-9\delta$ over $\pgaus^n$, KKT points of Problem~\eqref{eq:margin.maximization.problem} exhibit benign overfitting:
\begin{align*}
\text{for all $k\in [n]$}, \quad y_k &= \sgn\big( f(x_k;W)\big), \\
\text{while simultaneously,} \quad  \eta &\leq \P_{(x,y)\sim \pgaus}\Big(y \neq  \sgn \big( f(x_k;W) \big) \Big) \leq \eta + \tilde O\l( d^{\f 1 2(1-2\rho)} \r) = \eta + o_d(1).
\end{align*}
Furthermore, for any initialization $W(0)$, gradient flow converges in direction to a network satisfying the above. 
\end{restatable}
 We note that a similar result on benign overfitting for the linear max-margin classifier for data coming from $\pgaus$ has been shown by~\citet{muthukumar2021classification} with a rather different proof technique.

\section{Benign Overfitting for Clustered Data}\label{sec:clustered}

In this section we consider a distribution where data comes from multiple clusters and data from each cluster initially share the same label but then are flipped with some constant probability $\eta$.  In particular, we consider a distribution $\pclust$ over $(x,y)\in \R^d\times \{\pm 1\}$ defined as follows.   Let $k\geq 2$ and $Q := [k]$.  We are given cluster means $\muq 1, \dots, \muq k$ with cluster labels $\tyj 1, \dots, \tyj k \in \{\pm 1 \}$.   Cluster indices are sampled $q\sim \unif(Q)$, after which $x|q \sim \muq {q}+ z$ where $z \sim \pzz$ is such that: the components of $z$ are mean-zero, independent, sub-Gaussian random variables with sub-Gaussian norm at most one; and $\E[\snorm{z}^2]= d$.\footnote{We can easily accommodate well-conditioned clusters, e.g. $\kappa d \leq \E \snorm{z}^2 \leq d$ for some absolute constant $\kappa>0$, although the noise rate tolerated will then depend upon $\kappa$ (smaller $\kappa$ will require smaller $\eta$).  We do not do so for simplicity of exposition.}  Finally, the (clean) label of $x$ is $\tilde y = \yqi {q}$, and the observed label is $y = \tilde y$ with probability $1-\eta$ and $y = -\tilde y$ with probability $1-\eta$.

For a given $\delta \in (0,\nicefrac 12)$, we make the following assumptions on the parameters, for a sufficiently large constant $C>1$:
\begin{enumerate}[label=(CL\arabic*)]
    \item \label{a:samples}Number of samples $n\geq C k^2 \log(k/\delta)$. 
    \item \label{a:dimension}Dimension $d\geq C \max\{n \maxnormmu^2, n^2 \log(n/\delta)\}$.
    \item \label{a:min.norm.mu}The cluster means satisfy: $\minnormmu \geq C k \sqrt{\log(2nk/\delta)}$.
    \item \label{a:orthog.mu} The cluster means are nearly-orthogonal in the sense that: $\minnormmu^2 \geq C k \maxcorrmu$.
\end{enumerate}
We shall show below that under these assumptions, the training data is linearly separable with high probability. 

Our first lemma shows that under the preceding assumptions, the training data become more orthogonal and the ratio of the norms of the examples tends to one as $C$ increases. 
\begin{restatable}{lemma}{clusternearlyorthogonal}\label{lem:cluster.nearly.orthogonal}
There exists an absolute constant $C_2>0$ such for every large enough constant $C>0$, for any $\delta \in (0,\nicefrac 17)$, under Assumptions~\ref{a:samples} through~\ref{a:orthog.mu} (defined for these $C$ and $\delta$), with probability at least $1-7\delta$ over $\pclustnoise$, the training data is $\nicefrac C{C_2}$-orthogonal, and $\max_{i,j} \nicefrac{\snorm{x_i}^2}{\snorm{x_j}^2}\leq (1 + \nicefrac {C_2}{\sqrt C})^2$. 
\end{restatable}

The proof of the above lemma, as well as all proofs for this section, appears in Appendix~\ref{appendix:clustered}.
As before, Lemma~\ref{lem:cluster.nearly.orthogonal} allows for us to utilize Propositions~\ref{prop:bound.lambdas.max.margin.v2} and~\ref{prop:kkt.leaky.nearly.orthogonal} to show that both KKT points of the linear max-margin problem~\eqref{eq:linear.max.margin} and of the leaky ReLU network max-margin problem~\eqref{eq:leaky.relu.network} take the form $\summ i n s_i y_i x_i$.  The following theorem characterizes the performance of this predictor.   

\begin{restatable}{theorem}{clustersumsiyixi}\label{thm:cluster.sumsiyixi}
Let $\unifclass \geq 1$ be a constant, and suppose $\eta \leq \f{1}{1+\unifclass} - \Delta$ for some absolute constants $\eta,\Delta>0$.  There exist constants $C,C'>0$ (depending only on $\eta, \unifclass$, and $\Delta$) such that for any $\delta \in (0,\nicefrac 1{14})$, under Assumptions~\ref{a:samples} through~\ref{a:orthog.mu} (defined for these $C$ and $\delta$), with probability at least $1-14\delta$ over $\pclustnoise^n$, 
if $u\in \R^d$ is $\unifclass$-uniform w.r.t. $\{(x_i,y_i)\}_{i=1}^n$, then 
\begin{align*}
\text{for all $k\in [n]$}, \quad y_k &= \sgn\big(\sip{u}{x_k}\big), \\
\text{while simultaneously,} \quad  \eta &\leq \P_{(x,y)\sim \pclustnoise}\big(y \neq  \sgn (\sip{u}{x}) \big) \leq \eta + \exp \l( - \f{  n \minnormmu^4}{C'k^2 d}\r). 
\end{align*}
In particular, if $n \minnormmu^4 = \omega(k^2 d)$, then the linear classifier $x\mapsto \sgn(\sip{u}{x})$ exhibits benign overfitting. 
\end{restatable}

In order for benign overfitting to occur, the above theorem requires that Assumptions~\ref{a:samples} through~\ref{a:orthog.mu} are satisfied while simultaneously $\minnormmu^4 = \omega(k^2 d/n)$.  This can be satisfied in a number of settings, such as:
\begin{enumerate}
\item[$(i)$] Orthogonal clusters with $\snorm{\muq q} = \Theta(d^\beta)$ for each $q\in Q$, where $\beta \in (\nicefrac 14, \nicefrac 12)$, $k = O(1)$, $n = \tilde{\Omega}(1)$ and $d= \tilde{\Omega}(n^{1/(1-2\beta)})$.  In this setting the test error is at most $\eta + \exp\l( -\tilde{\Omega}(nd^{4\beta-1})\r)$. 
\item[$(ii)$] Non-orthogonal clusters where $\maxcorrmu = O(d^{\frac 35})$ and $\snorm{\muq q} = \Theta(d^{\frac 13})$ for each $q$, $n = \Theta(d^{\frac 15})$, and $k = \Theta(d^{0.05})$. In this setting the test error is at most $\eta + \exp\l( -\Omega(d^{0.43})\r)$. 
\end{enumerate}
Although neither of the above settings are explicitly covered by Theorem~\ref{thm:subgaussian.sumsiyixi}, the setting $(i)$ is similar in flavor to that theorem, in that the labels are determined by a constant number of high-variance directions. By contrast, the setting $(ii)$ is quite different, as it allows for (clean) labels to be determined by the output of a linear classifier over $k = \Theta(d^{0.05})$ components, namely $\sgn\big(\sip{\summ qk \yq q \muq q}{x}\big)$.

Just as in the case of Theorem~\ref{thm:subgaussian.sumsiyixi}, we have a number of corollaries of Theorem~\ref{thm:cluster.sumsiyixi}.    The first is a consequence of Proposition~\ref{prop:bound.lambdas.max.margin.v2}.

\begin{restatable}{corollary}{clusterlinearmaxmargin}
\label{cor:clusterlinearmaxmargin}
Suppose $0 < \eta \leq 0.49$.  There exist constants $C, C'>0$ such that for any $\delta \in (0,\nicefrac 1{21})$, under Assumptions~\ref{a:samples} through~\ref{a:orthog.mu} (defined for these $C$ and $\delta$), with probability at least $1-21\delta$ over $\pclustnoise^n$, the max-margin linear classifier $w = \argmin \{ \snorm{w}^2: y_i \sip{w}{x_i}\geq 1\, \forall i\}$ satisfies
\begin{align*}
\text{for all $k\in [n]$}, \quad y_k &= \sgn\big(\sip{w}{x_k}\big), \\
\text{while simultaneously,} \quad  \eta &\leq \P_{(x,y)\sim \pclustnoise}\big(y \neq  \sgn (\sip{w}{x}) \big) \leq \eta + \exp \l( - \f{  n \minnormmu^4}{C'k^2 d}\r). 
\end{align*}
In particular, if $n \minnormmu^4 = \omega(k^2 d)$ then $w$ exhibits benign overfitting.
\end{restatable}
Similarly, we can show that KKT points of the max-margin problem for leaky ReLU networks from Problem~\eqref{eq:margin.maximization.problem} also exhibit benign overfitting.

\begin{restatable}{corollary}{clusterleakykkt}
\label{cor:clusterleakykkt} 
Suppose that $0 < \eta \leq \f{49\gamma^2}{100}$.  There exist constants $C, C'>0$ such that for any $\delta \in (0,\nicefrac 1{21})$, under Assumptions~\ref{a:samples} through~\ref{a:orthog.mu} (defined for these $C$ and $\delta$), with probability at least $1-21\delta$ over $\pclustnoise^n$, any KKT point $W$ of Problem~\eqref{eq:margin.maximization.problem} satisfies
\begin{align*}
\text{for all $k\in [n]$}, \quad y_k &= \sgn\big(f(x_k;W)\big), \\
\text{while simultaneously,} \quad  \eta &\leq \P_{(x,y)\sim \pclustnoise}\big(y \neq  \sgn \big(f(x;W) \big) \leq \eta + \exp \l( - \f{  n \minnormmu^4}{C'k^2 d}\r). 
\end{align*}
In particular, if $n \minnormmu^4 = \omega(k^2 d)$ then the neural network $f(x;W)$ exhibits benign overfitting. 
Moreover, for any initialization $W(0)$, gradient flow converges in direction to a network which satisfies the above. 
\end{restatable}

We would like to note that Theorem~\ref{thm:cluster.sumsiyixi} (and the subsequent corollaries) does not explicitly cover the case that the data consists of two opposing clusters, i.e. when $\tilde y \sim \unif(\{\pm 1 \})$ and $x|\tilde y \sim \tilde y \mu + z$ for some random vector $z$ and labels are flipped $\tilde y \mapsto -\tilde y$ with some probability $\eta < 1/2$.  A number of recent works showed that the linear max-margin classifer~\citep{chatterji2020linearnoise,wang2021binary} and two-layer neural networks with \textit{smooth} leaky ReLU activations~\citep{frei2022benign} exhibit benign overfitting for this distributional setting.  However, our analysis can easily be extended to show benign overfitting of the linear max-margin and KKT points of the leaky ReLU network max-margin problem~\eqref{eq:margin.maximization.problem} for this distribution by using a small modification of the proof we use for Theorem~\ref{thm:cluster.sumsiyixi}.   We also wish to emphasize that the results of~\citet{frei2022benign} are specific to networks with smooth leaky ReLU activations, which are not 
homogeneous and for which gradient flow does not have a known implicit bias towards satisfying the KKT conditions for margin-maximization.  In particular, their analysis is based on tracking the generalization error of the neural network throughout the training trajectory, while ours relies upon the structure imposed by the KKT conditions for margin-maximization in homogeneous networks.  Another difference between our work and theirs concerns the label noise model.  We derive an explicit upper bound on the noise level tolerated, while their results hold for noise levels below an unspecified constant.  Our analysis holds for labels flipped with constant probability, while theirs permits adversarial label noise. 

\section{Proof intuition}
In this section we provide some intuition for how benign overfitting of the max-margin linear classifier is possible.  We consider a distribution $\poppnoise$ defined by a mean vector $\mu \in \R^d$ and label noise parameter $\eta \in (0,1/2)$, where examples $(x,y)\sim \poppnoise$ are sampled as follows: 
\begin{equation}\label{eq:gaussian.class.conditionals}
\tilde y \sim \unif(\{\pm 1\}),\quad   z\sim \mathsf N(0,I_d),\quad x|\tilde y \sim \tilde y \mu + z,\quad \begin{cases} y = \tilde y, & \text{w.p.}\, 1-\eta,\\
y = -\tilde y,& \text{w.p.}\, \eta.
\end{cases} 
\end{equation}
As we mentioned in the previous section, this distribution is not explicitly covered by Theorem~\ref{thm:cluster.sumsiyixi} but the intuition and proof are essentially the same.  
Our starting point is Proposition~\ref{prop:bound.lambdas.max.margin.v2}, which shows the max-margin linear classifier is $\unifclass$-uniform over the training data when the training data are nearly orthogonal.  For simplicity let us consider the simplest estimator of this form, the $1$-uniform vector $\hat \mu = \summ in y_i x_i$. Let us call the training examples $(x_i, y_i)$ for which $y_i = \tilde y_i$ the \textit{clean} examples, and denote the indices corresponding to such examples $\calC\subset [n]$, with the examples with $y_i = -\tilde y_i$ the \textit{noisy} examples, identified by $\calN\subset [n]$ (so $\calC\cup \calN = [n]$).  For the distribution~\eqref{eq:gaussian.class.conditionals} and training data $\{(x_i, y_i)\}_{i=1}^n$, the estimator $\hat \mu$ thus takes the form
\begin{align*}
\hat \mu &= \summ i n y_i x_i = \sum_{i\in \calC} (\mu + y_i z_i) + \sum_{i\in \calN} (-\mu + y_i z_i) = \l(|\calC| - |\calN| \r) \mu + \summ i n y_i z_i \propto \mu + \f{ 1}{|\calC| - |\calN|} \summ i n y_i z_i,
\end{align*}
where $z_i \iid \mathsf N(0,I_d)$.  
Now, provided $n$ is sufficiently large and if the noise rate is smaller than say $1/4$, then with high probability we have $9n/10 \geq |\calC| - |\calN| \geq n/10$.   In particular, the estimator $\hat \mu$ is proportional to a sum of two components: $\mu$, which is the linear classifier which achieves optimal accuracy for the distribution, and $(|\calC|-|\calN|)^{-1}\summ in y_iz_i$ which incorporates only the noise.  This latter component is useless for prediction on fresh test examples, but is quite useful for achieving small training error.  Indeed, for training example $(x_k, y_k) = (\tilde{y}_k \mu + z_k, y_k)$,
\begin{align*} 
\ip{y_k x_k}{\summ i n y_i z_i} &= \ip{y_k \tilde y_k \mu + y_k z_k}{y_k z_k + \sum_{i\neq k} y_i z_i} \\
&= \snorm{z_k}^2 + y_k \tilde y_k \ip{\mu}{\summ i n y_i z_i} + \ip{y_k z_k}{\sum_{i\neq k} y_i z_i}.
\end{align*} 
Since $z_k \sim \mathsf N(0,I_d)$ and $\summ i n y_i z_i \sim \mathsf N(0, n I_d)$, standard concentration bounds show that $\snorm{z_k}^2 \gtrsim d$ while $|\sip{\mu}{\summ in y_i z_i}| \lesssim \sqrt{n} \snorm{\mu}$, and $|\sip{y_kz_k}{\sum_{i\neq k} y_i z_i}|\lesssim \sqrt{nd}$ (ignoring log factors for simplicity).   Thus, provided $d \gg \sqrt{nd}$ and $d \gg \sqrt{n}\snorm{\mu}$, the estimator $\xi := (|\calC| - |\calN|)^{-1} \summ in y_i z_i$ satisfies
\begin{equation} \label{eq:overfitting.trainingdata.effect}
 \sip{\xi}{y_k x_k} \gtrsim d/n.
 \end{equation}
On the other hand, the effect of $\xi$ on an independent test example $(x,y)$ is, 
\begin{equation} \label{eq:overfitting.test.effect}
\l| \ip{yx}{\xi}\r| = \l| \f{1}{|\calC| - |\calN|} \ip{y \tilde y (z+\mu)}{\summ in y_iz_i} \r| \lesssim \f{1}{n} \l( \sqrt{n} \snorm{\mu} + \sqrt{dn} \r) = \f{ \snorm{\mu} + \sqrt{d}}{\sqrt n}.
\end{equation}
Putting~\eqref{eq:overfitting.trainingdata.effect} and~\eqref{eq:overfitting.test.effect} together, we see that $\xi$ has a significantly larger effect on the training data than on the test data performance as long as $d \gg \sqrt{n} \snorm{\mu} + \sqrt{dn}$.  Assuming $\snorm{\mu} < \sqrt d$ this holds when $d\gg n^2$.  In particular, the possibility of a component which enables benign overfitting becomes easier in high dimensions, at least when the signal $\snorm{\mu}$ is not too large.   

The above sketch shows that the overfitting component $\xi$ is useful for interpolating the (noisy) training data when $\snorm \mu < \sqrt d$ and $d\gg n^2$.  However, the estimator $\hat \mu \propto \mu + \xi$ also contains the component $\mu$ which is biased towards getting noisy training data \textit{incorrect}.  Thus, in order to show $\mu + \xi$ exhibits benign overfitting, we need to show $(i)$ the signal strength from $\mu$ is not so strong as to prevent overfitting the noisy training data, but $(ii)$ the signal strength from $\mu$ is large enough to enable good generalization from test data.  For part $(i)$, standard concentration bounds imply that 
\begin{equation} \label{eq:mu.trainingdata}
|\sip{\mu}{y_k x_k}| = |y_k \tilde y_k \snorm{\mu}^2 + \sip{y_k z_k}{\mu}| \lesssim \snorm{\mu}^2 + \snorm{\mu}.
\end{equation}
In light of~\eqref{eq:overfitting.trainingdata.effect}, the estimator $\mu+\xi$ will still interpolate the training data provided $d/n \gg \max(\snorm{\mu}, \snorm{\mu}^2)$.  For part $(ii)$, for a given \textit{clean} test example $(x,\tilde y)$, 
\begin{equation} \label{eq:mu.testdata}
\sip{\tilde yx}{\mu} = \sip{\mu + \tilde y z}{\mu} \gtrsim \snorm{\mu}^2 - C \snorm{\mu}.
\end{equation}
Thus, provided $\snorm{\mu}\gg C$ and $\snorm{\mu}^2 \gg n^{-1/2} (\snorm{\mu}+\sqrt d)$, we can be ensured that $\mu + \xi$ will also classify clean test examples correctly by putting together~\eqref{eq:mu.testdata} and~\eqref{eq:overfitting.test.effect}. 

To summarize, we have identified settings under which we can guarantee that benign overfitting occurs for the estimator $\hat \mu \propto \mu + \xi$ for the distribution $\poppnoise$.  First, the training data must be sufficiently high-dimensional to ensure that the overfitting component $\xi$ has a significant effect on the training data but little effect on future test data.  Second, the underlying signal of $\hat \mu$ (whose strength is measured by $\snorm{\mu}$) must not be so strong as to prevent overfitting to the noisy labels, yet must also be strong enough to ensure that future test data can be accurately predicted.

\section{Discussion}
We have characterized a number of new settings under which linear classifiers and two-layer neural networks can exhibit benign overfitting.  We showed how the implicit bias of gradient flow imposes significant structure on linear classifiers and neural networks trained by this method, and how this structure can be leveraged to understand the generalization of interpolating models in the presence of noisy labels.  

There are a number of directions for future research.  For instance, the larger class of homogeneous neural networks trained by the logistic loss also have an implicit bias towards satisfying the KKT conditions for margin-maximization.  Can this implicit bias be leveraged to show benign overfitting in neural networks with ReLU activations of depth $L\geq 2$?  Additionally, although two-layer leaky ReLU networks are in general nonlinear, our proof holds in settings where their decision boundaries are linear.  It would be interesting to understand benign overfitting in neural networks when the learned decision boundary is nonlinear. 

\subsection*{Acknowledgements}

SF, GV, PB, and NS acknowledge the support of the NSF and the Simons Foundation for the Collaboration on the Theoretical Foundations of Deep Learning through awards DMS-2031883 and \#814639.

\appendix
{\hypersetup{linkcolor=Black}\tableofcontents}

\section{Proof of Proposition~\ref{prop:bound.lambdas.max.margin.v2}} \label{appendix:lambdas.linear}

Since $\hat{w}$ 
satisfies the KKT conditions of the max-margin problem, we have $\hat{w} = \sum_{i = 1}^n \lambda_i y_i x_i$ where for all $i \in [n]$ we have $\lambda_i \geq 0$, and $\lambda_i = 0$ if $y_i \hat{w}^\top x_i \neq 1$.
We denote $\rmin = \min_i \snorm{x_i}$, $\rmax = \max_i \snorm{x_i}$, and $\rratio = \rmax/\rmin$.

In the following lemma, we obtain an upper bound for the $\lambda_i$'s:

\begin{lemma}\label{lem:lam.upper}
    Suppose the training data is $\orthog$-orthogonal.  Then for all $i \in [n]$ we have $\lambda_i \leq \frac{1}{\rmin^2 \left( 1 - \frac{1}{\orthog \rratio^2} \right)}$.
\end{lemma}
\begin{proof}
    Let $j \in \argmax_{i \in [n]} \lambda_i$.
    We have
    \begin{equation}\label{eq:lam.upper1}
        y_j \hat{w}^\top x_j
        = \sum_{i = 1}^n \lambda_i y_j y_i x_i^\top x_j
        = \lambda_j \norm{x_j}^2 + \sum_{i \neq j} \lambda_i y_j y_i x_i^\top x_j
        \geq \lambda_j \rmin^2 - n \left(\max_{i \in [n]} \lambda_i \right) \left( \max_{i \neq j} | \sip{x_i}{x_j} | \right)~.
    \end{equation}
    By the $\orthog$-orthogonality assumption, we also have 
    \begin{equation}\label{eq:lam.upper2}
        n \max_{i\neq j} |\sip{x_i}{x_j}| \leq \frac{\rmin^2}{\orthog \rratio^2}~.
    \end{equation}
    Suppose that 
    \begin{equation}\label{eq:lam.upper3}
        \lambda_j = \max_{i \in [n]} \lambda_i > \frac{1}{\rmin^2 \left( 1 - \frac{1}{\orthog \rratio^2} \right)}~.
    \end{equation}
    Combining \eqref{eq:lam.upper1},~\eqref{eq:lam.upper2} and~\eqref{eq:lam.upper3}, we get
    \begin{align*}
        y_j \hat{w}^\top x_j
        \geq \lambda_j \rmin^2 - \lambda_j \cdot \frac{\rmin^2}{\orthog \rratio^2}
        = \lambda_j \rmin^2 \left( 1 - \frac{1}{\orthog \rratio^2} \right)   
        > 1~.
    \end{align*}
    By the KKT conditions, if $y_j \hat{w}^\top x_j > 1$ then we must have $\lambda_j=0$, and thus we reach a contradiction.
\end{proof}

Next, we obtain a lower bound on the $\lambda_i$'s:
\begin{lemma}\label{lem:lam.lower}
Suppose the training data is $\orthog$-orthogonal.  Then 
for all $i \in [n]$ we have $\lambda_i \geq \frac{1}{\rmax^2} \left( 1 - \frac{1}{\orthog \rratio^2 - 1} \right)$.
\end{lemma}
\begin{proof}
    Let $j \in [n]$. By the definition of $\hat{w}$ we have 
    \begin{align}\label{eq:lam.lower1}
        1 
        &\leq y_j \hat{w}^\top x_j \nonumber
        \\
        &= \sum_{i = 1}^n \lambda_i y_j y_i x_i^\top x_j \nonumber
        \\
        &= \lambda_j \norm{x_j}^2 + \sum_{i \neq j} \lambda_i y_j y_i x_i^\top x_j \nonumber
        \\
        &\leq \lambda_j \rmax^2 + n \left(\max_{i \in [n]} \lambda_i \right) \left(\max_{i \neq j} | \sip{x_i}{x_j} | \right)~.
    \end{align}
    By the $\orthog$-orthogonality assumption, we have 
    \begin{equation}\label{eq:lam.lower2}
        n \max_{i\neq j} |\sip{x_i}{x_j}| \leq \frac{\rmin^2}{\orthog \rratio^2}~,
    \end{equation}
    and by Lemma~\ref{lem:lam.upper} we have 
    \begin{equation}\label{eq:lam.lower3}
        \max_{i \in [n]} \lambda_i \leq \frac{1}{\rmin^2 \left( 1 - \frac{1}{\orthog \rratio^2} \right)}~.
    \end{equation}
    Combining \eqref{eq:lam.lower1},~\eqref{eq:lam.lower2} and~\eqref{eq:lam.lower3}, we get
    \[
        1 
        \leq \lambda_j \rmax^2 + \frac{1}{\rmin^2 \left( 1 - \frac{1}{\orthog \rratio^2} \right)} \cdot \frac{\rmin^2}{\orthog \rratio^2}
        = \lambda_j \rmax^2 + \frac{1}{\orthog \rratio^2 - 1}~.
    \]
    Hence,
    \[
        \lambda_j \geq \left( 1 - \frac{1}{\orthog \rratio^2 - 1} \right) \frac{1}{\rmax^2}~.
    \]
\end{proof}

Combining Lemmas~\ref{lem:lam.upper} and~\ref{lem:lam.lower}, we conclude that $\hat{w} = \sum_{i = 1}^n \lambda_i y_i x_i$ where
\begin{align*}
    \frac{\max_i \lambda_i}{\min_i \lambda_i}
    &\leq \frac{1}{\rmin^2 \left( 1 - \frac{1}{\orthog \rratio^2} \right)} \cdot \rmax^2 \left( 1 - \frac{1}{\orthog \rratio^2 - 1} \right)^{-1} 
    \\
    &= \frac{\orthog \rratio^2}{\rmin^2 ( \orthog \rratio^2 -1 )} \cdot \rmax^2 \cdot \frac{\orthog \rratio^2 - 1}{\orthog \rratio^2 - 2}
    \\
    &= \frac{\rmax^2}{\rmin^2} \cdot \frac{\orthog \rratio^2}{\orthog \rratio^2 - 2}
    \\
    &= \rratio^2 \left(1 + \frac{2}{\orthog \rratio^2 - 2} \right)~.    
\end{align*}

\section{Proof of Proposition~\ref{prop:kkt.leaky.nearly.orthogonal}} \label{appendix:lambdas.leaky}

We start with some notations. 
For convenience, we will use different notations for positive neurons (i.e., where $a_j = 1/\sqrt{m}$) and negative neurons (i.e., where $a_j = -1/\sqrt{m}$). Namely,
\begin{equation*} \label{eq:notations for neurons}
    f(x; W) 
    = \sum_{j=1}^m a_j \phi(w_j^\top x)
    = \sum_{j =1}^{m/2}\frac{1}{\sqrt{m}} \phi(v_j^\top x) - \sum_{j =1}^{m/2} \frac{1}{\sqrt{m}} \phi(u_j^\top x)~.
\end{equation*}
We denote $\zeta = \max_{i \neq j}|\inner{x_i,x_j}|$. 
Thus, our near-orthogonality assumption can be written as $n \zeta \leq \frac{\rmin^2}{\orthog \rratio^2}$.
Since $W$ satisfies the KKT conditions of Problem~\eqref{eq:margin.maximization.problem}, then there are $\lambda_1,\ldots,\lambda_n$ (known as \emph{KKT multipliers}) such that for every $j \in [m/2]$ we have
\begin{equation}
\label{eq:kkt condition v}
	v_j 
	= \sum_{i \in [n]} \lambda_i \nabla_{v_j} \left( y_i f(x_i; W) \right) 
	= \frac{1}{\sqrt{m}} \sum_{i \in [n]} \lambda_i y_i \phi'_{i,v_j} x_i~,
\end{equation}
where $\phi'_{i,v_j}$ is a subgradient  of $\phi$ at $v_j^\top x_i$, i.e., if $v_j^\top x_i > 0$ then $\phi'_{i,v_j} = 1$, if $v_j^\top x_i < 0$ then $\phi'_{i,v_j} = \gamma$ and otherwise $\phi'_{i,v_j}$ is some value in $[\gamma,1]$. Also, we have $\lambda_i \geq 0$ for all $i$, and $\lambda_i=0$ if 
$y_i  f(x_i; W) \neq 1$. Likewise, for all $j \in [m/2]$ we have
\begin{equation}
\label{eq:kkt condition u}
	u_j 
	= \sum_{i \in [n]} \lambda_i \nabla_{u_j} \left( y_i f(x_i; W) \right) 
	= \frac{1}{\sqrt{m}} \sum_{i \in [n]} \lambda_i (- y_i) \phi'_{i,u_j} x_i~,
\end{equation}
where $\phi'_{i,u_j}$ is defined similarly to $\phi'_{i,v_j}$.

Our proof builds on the following lemma, which follows from \citet{frei2023implicit} (note that within their notation, in our setting we have $m_1=m_2=m/2$):
\begin{lemma}[\citet{frei2023implicit}, Theorem 3.2 \& Corollary 3.5] \label{lem:from.previous}
    Denote $\rmin^2 = \min_i \snorm{x_i}^2$ and $\rmax^2 = \max_i \snorm{x_i}^2$.  Let $f$ denote the leaky ReLU network~\eqref{eq:leaky.relu.network} and let $W$ denote a KKT point of Problem~\eqref{eq:margin.maximization.problem}.  
    Let $\lambda_1,\ldots,\lambda_n \geq 0$ denote the corresponding KKT multipliers.
    Suppose the training data are $\orthog$-orthogonal for $\orthog \geq 3 \gamma^{-3}$.  Then, 
    we have
    $\lambda_i \in \left( \frac{1}{2 \rmax^2}, \frac{3}{2 \gamma^2  \rmin^2} \right)$ for all $i \in [n]$, and for any $x\in \R^d$ we have $\sgn\l( f(x;W) \r) = \sgn \l(\ip{z}{x}\r)$, where $z = \frac{\sqrt{m}}{2} v - \frac{\sqrt{m}}{2} u$ for 
    \[
        v 
        = \frac{1}{\sqrt{m}} \sum_{i: y_i=1} \lambda_i x_i - \frac{\gamma}{\sqrt{m}} \sum_{i: y_i=-1} \lambda_i x_i~,     
    \]
    and
    \[
        u
        = \frac{1}{\sqrt{m}} \sum_{i: y_i=-1} \lambda_i x_i - \frac{\gamma}{\sqrt{m}} \sum_{i: y_i=1} \lambda_i x_i~.
    \]
    
    Moreover, for any initialization $W(0)$, gradient flow on the logistic or exponential loss converges in direction to such a KKT point.
\end{lemma}

Note that the above lemma implies that $\sgn\l( f(x;W) \r) = \sgn \l(\ip{z}{x}\r)$ for 
\begin{align} \label{eq:z.formula} 
    z 
    &= \frac{\sqrt{m}}{2} \left( \frac{1}{\sqrt{m}} \sum_{i: y_i=1} \lambda_i x_i - \frac{\gamma}{\sqrt{m}} \sum_{i: y_i=-1} \lambda_i x_i \right) - \frac{\sqrt{m}}{2} \left( \frac{1}{\sqrt{m}} \sum_{i: y_i=-1} \lambda_i x_i - \frac{\gamma}{\sqrt{m}} \sum_{i: y_i=1} \lambda_i x_i \right) \nonumber
    \\
    &= \frac{1+\gamma}{2} \sum_{i: y_i=1} \lambda_i x_i - \frac{1+\gamma}{2} \sum_{i: y_i=-1} \lambda_i x_i \nonumber
    \\
    &= \frac{1+\gamma}{2} \sum_{i=1}^n y_i \lambda_i x_i~. 
\end{align}
The lemma also implies that $\lambda_i \in \left( \frac{1}{2 \rmax^2}, \frac{3}{2 \gamma^2  \rmin^2} \right)$ for all $i$. However, these bounds are not accurate enough for us. In the following lemmas we obtain bounds which give the explicit dependence on $\orthog$ for $\orthog$-orthogonal data.  The proofs of the lemmas follow similar arguments to the proof from \citet{frei2023implicit}, with some required modifications. 

\begin{lemma}\label{lem:lambda.upper.bound}
    Denote $\rmin^2 = \min_i \snorm{x_i}^2$, $\rmax^2 = \max_i \snorm{x_i}^2$, and $\rratio^2 = \rmax^2/\rmin^2$. Let $f$ denote the leaky ReLU network~\eqref{eq:leaky.relu.network} and let $W$ denote a KKT point of Problem~\eqref{eq:margin.maximization.problem}.  
    Suppose the training data are $\orthog$-orthogonal for $\orthog \geq 3 \gamma^{-3}$. 
    Using the notation from Eq.~\eqref{eq:kkt condition v} and~\eqref{eq:kkt condition u},
    for all $i \in [n]$ we have 
    \[
        \sum_{j \in [m/2]}  \lambda_i \phi'_{i,v_j}  + \sum_{j \in [m/2]} \lambda_i \phi'_{i,u_j} \leq \frac{m}{\rmin^2 \left(\gamma - \frac{1}{\orthog \rratio^2} \right)}~,
    \]
    and 
    \[
        \lambda_i \leq \frac{1}{\rmin^2 \gamma \left(\gamma - \frac{1}{\orthog \rratio^2} \right)}~.
    \]
\end{lemma}
\begin{proof}
    Let $\xi = 	\max_{q \in [n]} \left(  \sum_{j \in [m/2]}  \lambda_q \phi'_{q,v_j}  + \sum_{j \in [m/2]} \lambda_q \phi'_{q,u_j} \right)$ and suppose that $\xi > \frac{m}{\rmin^2 \left(\gamma - \frac{1}{\orthog \rratio^2} \right)}$. Let $r = \argmax_{q \in [n]} \left(  \sum_{j \in [m/2]}  \lambda_q \phi'_{q,v_j}  + \sum_{j \in [m/2]} \lambda_q \phi'_{q,u_j} \right)$. Since by our assumption $\orthog \geq 3 \gamma^{-3} \geq \frac{3}{\gamma}$ and $\rratio \geq 1$, then $\xi > 0$ and therefore $\lambda_r > 0$. Hence, by the KKT conditions we must have $y_r f(x_r; W) = 1$.
	
	We consider two cases:
	
	{\bf Case 1: }
	Assume that $y_r = -1$. Using \eqref{eq:kkt condition v} and~\eqref{eq:kkt condition u}, we have
	\begin{align*}
		\sqrt{m} f(x_r; W)
		&= \sum_{j \in [m/2]} \phi(v_j^\top x_r) - \sum_{j \in [m/2]} \phi(u_j^\top x_r)
		\\
		&= \sum_{j \in [m/2]} \phi \left(\frac{1}{\sqrt{m}} \sum_{q \in [n]} \lambda_q y_q \phi'_{q,v_j} x_q^\top x_r \right) - \sum_{j \in [m/2]}\phi \left(\frac{1}{\sqrt{m}} \sum_{q \in [n]} \lambda_q (-y_q) \phi'_{q,u_j} x_q^\top x_r \right)
		\\
		&= \sum_{j \in [m/2]} \phi \left(\frac{1}{\sqrt{m}} \lambda_r y_r \phi'_{r,v_j} x_r^\top x_r + \frac{1}{\sqrt{m}} \sum_{q \in [n] \setminus \{r\}} \lambda_q y_q \phi'_{q,v_j} x_q^\top x_r \right) 
		\\
		&\;\;\;\; -  \sum_{j \in [m/2]} \phi \left(\frac{1}{\sqrt{m}} \lambda_r (-y_r) \phi'_{r,u_j} x_r^\top x_r + \frac{1}{\sqrt{m}} \sum_{q \in [n] \setminus\{r\}}  \lambda_q (-y_q) \phi'_{q,u_j} x_q^\top x_r \right)
		\\
		&\leq \sum_{j \in [m/2]} \phi \left( - \frac{1}{\sqrt{m}} \lambda_r \phi'_{r,v_j} \rmin^2 + \frac{1}{\sqrt{m}} \sum_{q \in [n] \setminus \{r\}} \lambda_q y_q \phi'_{q,v_j} x_q^\top x_r \right) 
		\\
		&\;\;\;\; - \sum_{j \in [m/2]} \phi \left(\frac{1}{\sqrt{m}} \lambda_r \phi'_{r,u_j} \rmin^2 + \frac{1}{\sqrt{m}} \sum_{q \in [n] \setminus\{r\}} \lambda_q (-y_q) \phi'_{q,u_j} x_q^\top x_r \right)~.
	\end{align*}
    Since the derivative of $\phi$ is lower bounded by $\gamma$, we know $\phi(z_1)-\phi(z_2) \geq \gamma(z_1-z_2)$ for all $z_1,z_2\in \R$.  Using this and the definition of $\xi$, the above is at most
	\begin{align*}
		&\sum_{j \in [m/2]}\l[  \phi \left( \frac{1}{\sqrt{m}} \sum_{q \in [n] \setminus \{r\}} \lambda_q y_q \phi'_{q,v_j} x_q^\top x_r \right)  - \frac{1}{\sqrt{m}} \gamma \lambda_r \phi'_{r,v_j} \rmin^2\r]
		\\
		&\;\;\;\;\;- \sum_{j \in [m/2]} \l[ \phi \left(\frac{1}{\sqrt{m}} \sum_{q \in [n] \setminus\{r\}} \lambda_q (-y_q) \phi'_{q,u_j} x_q^\top x_r \right) + \frac{1}{\sqrt{m}} \gamma \lambda_r \phi'_{r,u_j} \rmin^2\r]
		\\
		&\leq - \frac{1}{\sqrt{m}} \gamma \xi \rmin^2 + 
		\sum_{j \in [m/2]}  \left| \frac{1}{\sqrt{m}} \sum_{q \in [n] \setminus \{r\}} \lambda_q y_q \phi'_{q,v_j} x_q^\top x_r \right| 
		+ \sum_{j \in [m/2]} \left|\frac{1}{\sqrt{m}} \sum_{q \in [n] \setminus\{r\}} \lambda_q (-y_q) \phi'_{q,u_j} x_q^\top x_r \right| 
		\\
		&\leq - \frac{1}{\sqrt{m}} \gamma \xi \rmin^2 + 
		\frac{1}{\sqrt{m}} \sum_{j \in [m/2]}  \sum_{q \in [n] \setminus \{r\}} \left| \lambda_q y_q \phi'_{q,v_j} x_q^\top x_r \right| 
		+ \frac{1}{\sqrt{m}} \sum_{j \in [m/2]} \sum_{q \in [n] \setminus\{r\}} \left| \lambda_q (-y_q) \phi'_{q,u_j} x_q^\top x_r \right|~.
	\end{align*}
	Using $| x_q^\top x_r | \leq \zeta$ for $q \neq r$, the above is at most
	\begin{align*}
		&- \frac{1}{\sqrt{m}} \gamma \xi \rmin^2 + 
		\frac{1}{\sqrt{m}} \sum_{j \in [m/2]}   \sum_{q \in [n] \setminus \{r\}} \lambda_q \phi'_{q,v_j} \zeta + \frac{1}{\sqrt{m}} \sum_{j \in [m/2]} \sum_{q \in [n] \setminus \{r\}} \lambda_q \phi'_{q,u_j} \zeta 
		\\
		&= - \frac{1}{\sqrt{m}} \gamma \xi \rmin^2 + \frac{\zeta}{\sqrt{m}}  
		 \sum_{q \in [n] \setminus \{r\}} \l(  \sum_{j \in [m/2]}   \lambda_q \phi'_{q,v_j} + \sum_{j \in [m/2]} \lambda_q \phi'_{q,u_j} \r)
		\\
		&\leq - \frac{1}{\sqrt{m}} \gamma \xi \rmin^2 + \frac{\zeta}{\sqrt{m}}  \cdot n \cdot \max_{q \in [n]} \left( \sum_{j \in [m/2]}   \lambda_q \phi'_{q,v_j} +  \sum_{j \in [m/2]} \lambda_q \phi'_{q,u_j} 	\right)
		\\
		&= - \frac{1}{\sqrt{m}} \gamma \xi \rmin^2 + \frac{\zeta}{\sqrt{m}} n \xi
            \\
		&= - \frac{\xi}{\sqrt{m}} ( \gamma \rmin^2 - n \zeta )~.
	\end{align*}
	By our $\orthog$-orthogonality assumption, the above expression is at most
	\begin{align*}
		- \frac{\xi}{\sqrt{m}} \left( \gamma \rmin^2 - \frac{\rmin^2}{\orthog\rratio^2} \right)
		= - \frac{\xi \rmin^2}{\sqrt{m}} \left( \gamma  - \frac{1}{\orthog \rratio^2} \right)
		< - \frac{m}{\rmin^2 \left(\gamma - \frac{1}{\orthog \rratio^2} \right)} \cdot \frac{\rmin^2}{\sqrt{m}} \left( \gamma  - \frac{1}{\orthog \rratio^2} \right)
        = - \sqrt{m}~,
	\end{align*}
        where in the inequality we used the assumption on $\xi$, the assumption $\orthog \geq 3 \gamma^{-3} \geq \frac{3}{\gamma}$, and $\rratio \geq 1$.
	Thus, we obtain $f(x_r; W) < -1$ in contradiction to $y_r f(x_r; W)=1$.
	
	{\bf Case 2: }
	Assume that $y_r = 1$. A similar calculation to the one given in case 1 (which we do not repeat for conciseness) implies that $f(x_r; W)>1$, in contradiction to $y_r f(x_r; W)=1$. It concludes the proof of $\xi \leq \frac{m}{\rmin^2 \left(\gamma - \frac{1}{\orthog \rratio^2} \right)}$.
	
	Finally, since $\xi \leq \frac{m}{\rmin^2 \left(\gamma - \frac{1}{\orthog \rratio^2} \right)}$ and the derivative of $\phi$ is lower bounded by $\gamma$, then for all $i \in [n]$ we have
	\begin{align*}
		\frac{m}{\rmin^2 \left(\gamma - \frac{1}{\orthog \rratio^2} \right)}
		\geq \sum_{j \in [m/2]}  \lambda_i \phi'_{i,v_j}  + \sum_{j \in [m/2]} \lambda_i \phi'_{i,u_j} 
		\geq m \lambda_i \gamma~,
	\end{align*}
	and hence $\lambda_i \leq \frac{1}{\rmin^2 \gamma \left(\gamma - \frac{1}{\orthog \rratio^2} \right)}$.
\end{proof}

\begin{lemma} \label{lem:lambda.lower.bound}
    Denote $\rmin^2 = \min_i \snorm{x_i}^2$, $\rmax^2 = \max_i \snorm{x_i}^2$, and $\rratio^2 = \rmax^2/\rmin^2$. Let $f$ denote the leaky ReLU network~\eqref{eq:leaky.relu.network} and let $W$ denote a KKT point of Problem~\eqref{eq:margin.maximization.problem}.  
    Suppose the training data are $\orthog$-orthogonal for $\orthog \geq 3 \gamma^{-3}$. 
    Using the notation from Eq.~\eqref{eq:kkt condition v} and~\eqref{eq:kkt condition u},
    for all $i \in [n]$ we have 
    \[
        \sum_{j \in [m/2]} \lambda_i \phi'_{i,v_j} + \sum_{j \in [m/2]} \lambda_i \phi'_{i,u_j} \geq \frac{m(\gamma \orthog \rratio^2 - 2)}{\rmax^2(\gamma \orthog \rratio^2 - 1)}~,
    \]
    and
    \[
        \lambda_i \geq \frac{\gamma \orthog \rratio^2 - 2}{\rmax^2(\gamma \orthog \rratio^2 - 1)} =  \frac{1}{\rmax^2} \left(1 - \frac{1}{\gamma \orthog \rratio^2 - 1}\right)~.
    \]
\end{lemma}
\begin{proof}
    Suppose that there is $i \in [n]$ such that $\sum_{j \in [m/2]} \lambda_i \phi'_{i,v_j} + \sum_{j \in [m/2]} \lambda_i \phi'_{i,u_j} < \frac{m(\gamma \orthog \rratio^2 - 2)}{\rmax^2(\gamma \orthog \rratio^2 - 1)}$.
    Using \eqref{eq:kkt condition v} and~\eqref{eq:kkt condition u}, we have
        \begin{align*}
		\sqrt{m} 
		&\leq \left| \sqrt{m} f(x_i; W) \right|
		= \left| \sum_{j \in [m/2]} \phi(v_j^\top x_i) - \sum_{j \in [m/2]} \phi(u_j^\top x_i) \right|
		\leq  \sum_{j \in [m/2]} \left| v_j^\top x_i \right| + \sum_{j \in [m/2]} \left| u_j^\top x_i \right| 
		\\
		&= \sum_{j \in [m/2]} \left| \frac{1}{\sqrt{m}} \sum_{q \in [n]} \lambda_q y_q \phi'_{q,v_j} x_q^\top x_i \right| + \sum_{j \in [m/2]} \left|\frac{1}{\sqrt{m}} \sum_{q \in [n]} \lambda_q (-y_q) \phi'_{q,u_j} x_q^\top x_i \right| 
		\\
		&\leq \frac{1}{\sqrt{m}}  \sum_{j \in [m/2]}  \left( \left| \lambda_i y_i \phi'_{i,v_j} x_i^\top x_i \right| + \sum_{q \in [n] \setminus \{i\}} \left| \lambda_q y_q \phi'_{q,v_j} x_q^\top x_i \right| \right)
		\\
		&\;\;\;\;\; + \frac{1}{\sqrt{m}}  \sum_{j \in [m/2]} \left( \left| \lambda_i (-y_i) \phi'_{i,u_j} x_i^\top x_i \right|  + \sum_{q \in [n] \setminus \{i\}} \left| \lambda_q (-y_q) \phi'_{q,u_j} x_q^\top x_i \right| \right)~.
	\end{align*}
	Using $| x_q^\top x_i | \leq \zeta$ for $q \neq i$ and $x_i^\top x_i \leq \rmax^2$, the above is at most
	\begin{align*}
		&\frac{1}{\sqrt{m}} \sum_{j \in [m/2]}  \left( \lambda_i \phi'_{i,v_j} \rmax^2 + \sum_{q \in [n] \setminus \{i\}} \lambda_q \phi'_{q,v_j} \zeta \right) + \frac{1}{\sqrt{m}} \sum_{j \in [m/2]} \left( \lambda_i \phi'_{i,u_j} \rmax^2+ \sum_{q \in [n] \setminus \{i\}}  \lambda_q \phi'_{q,u_j} \zeta \right) 
		\\
		&= \frac{1}{\sqrt{m}} \left( \sum_{j \in [m/2]} \lambda_i \phi'_{i,v_j} \rmax^2 + \sum_{j \in [m/2]} \lambda_i \phi'_{i,u_j} \rmax^2 \right) + 
            \\
            &\;\;\;\;\;\;\;
		\frac{1}{\sqrt{m}} \sum_{q \in [n] \setminus \{i\}} \l( \sum_{j \in [m/2]}  \lambda_q \phi'_{q,v_j} \zeta + \sum_{j \in [m/2]}  \lambda_q \phi'_{q,u_j} \zeta  \r)
		\\
		&= \frac{\rmax^2}{\sqrt{m}} \left( \sum_{j \in [m/2]} \lambda_i \phi'_{i,v_j} + \sum_{j \in [m/2]} \lambda_i \phi'_{i,u_j} \right) + 		
		\frac{\zeta}{\sqrt{m}} \sum_{q \in [n] \setminus \{i\}} \l(\sum_{j \in [m/2]} \lambda_q \phi'_{q,v_j}  + \sum_{j \in [m/2]} \lambda_q \phi'_{q,u_j} \r) 
		\\
		&< \frac{\rmax^2}{\sqrt{m}} \cdot \frac{m(\gamma \orthog \rratio^2 - 2)}{\rmax^2(\gamma \orthog \rratio^2 - 1)} + \frac{\zeta}{\sqrt{m}} \cdot n \cdot \max_{q \in [n]} \left(  \sum_{j \in [m/2]}  \lambda_q \phi'_{q,v_j}  + \sum_{j \in [m/2]} \lambda_q \phi'_{q,u_j} \right)~,
	\end{align*}
        where in the last inequality we used our assumption on $i$.
	Combining the above with our $\orthog$-orthogonality assumption $n \zeta \leq \frac{\rmin^2}{\orthog \rratio^2}$, we get 
	\begin{align*}
		\max_{q \in [n]} \left(  \sum_{j \in [m/2]}  \lambda_q \phi'_{q,v_j}  + \sum_{j \in [m/2]} \lambda_q \phi'_{q,u_j} \right) 
            &> m \left(1 - \frac{\gamma \orthog \rratio^2 - 2}{\gamma \orthog \rratio^2 - 1} \right) \cdot \frac{1}{n \zeta}
            \\
            &\geq m \left(1 - \frac{\gamma \orthog \rratio^2 - 2}{\gamma \orthog \rratio^2 - 1} \right) \cdot \frac{\orthog \rratio^2}{\rmin^2}
            \\
            &= m \left(\frac{1}{\gamma \orthog \rratio^2 - 1} \right) \cdot \frac{\orthog \rratio^2}{\rmin^2}
            \\
            &= \frac{m}{\rmin^2} \left(\frac{1}{\gamma - \frac{1}{\orthog \rratio^2}} \right)~, 
	\end{align*}
	in contradiction to Lemma~\ref{lem:lambda.upper.bound}. It concludes the proof of $\sum_{j \in [m/2]} \lambda_i \phi'_{i,v_j} + \sum_{j \in [m/2]} \lambda_i \phi'_{i,u_j} \geq \frac{m(\gamma \orthog \rratio^2 - 2)}{\rmax^2(\gamma p \rratio^2 - 1)}$.
	
	Finally, since $\sum_{j \in [m/2]} \lambda_i \phi'_{i,v_j} + \sum_{j \in [m/2]} \lambda_i \phi'_{i,u_j} \geq \frac{m(\gamma \orthog \rratio^2 - 2)}{\rmax^2(\gamma \orthog \rratio^2 - 1)}$ and the derivative of $\phi$ is upper bounded by $1$, then for all $i \in [n]$ we have
	\begin{align*}
		\frac{m(\gamma \orthog \rratio^2 - 2)}{\rmax^2(\gamma \orthog \rratio^2 - 1)}
		\leq \sum_{j \in [m/2]}  \lambda_i \phi'_{i,v_j}  + \sum_{j \in [m/2]} \lambda_i \phi'_{i,u_j} 
		\leq m \lambda_i~,
	\end{align*}
	and hence 
    \[
        \lambda_i \geq \frac{\gamma \orthog \rratio^2 - 2}{\rmax^2(\gamma \orthog \rratio^2 - 1)} =  \frac{1}{\rmax^2} \left(1 - \frac{1}{\gamma \orthog \rratio^2 - 1}\right)~.
    \]
\end{proof}

Combining Eq.~\eqref{eq:z.formula} with Lemmas~\ref{lem:lambda.upper.bound} and~\ref{lem:lambda.lower.bound}, and letting $s_i=\frac{(1+\gamma)\lambda_i}{2}$ for all $i \in [n]$, we get that $\sgn\l( f(x;W) \r) = \sgn \l(\ip{z}{x}\r)$ for $z = \sum_{i=1}^n s_i y_i x_i$, where for all $i \in [n]$ we have
\[
    s_i \in \left[\frac{(1+\gamma)}{2} \cdot \frac{1}{\rmax^2} \left(1 - \frac{1}{\gamma \orthog \rratio^2 - 1}\right), 
    \frac{(1+\gamma)}{2} \cdot \frac{1}{\rmin^2 \gamma^2 \left(1 - \frac{1}{\gamma \orthog \rratio^2} \right)} \right]~.
\]
Therefore, $z$ is $\unifclass$-uniform with 
\begin{align*}
    \unifclass
    &= \frac{(1+\gamma)}{2} \cdot \frac{1}{\rmin^2 \gamma^2 \left(1 - \frac{1}{\gamma \orthog \rratio^2} \right)} \cdot \frac{2}{(1+\gamma)} \cdot \rmax^2 \left(1 - \frac{1}{\gamma \orthog \rratio^2 - 1}\right)^{-1}
    \\
    &= \frac{\gamma \orthog \rratio^2}{\rmin^2 \gamma^2 (\gamma \orthog \rratio^2 - 1)} \cdot \rmax^2 \frac{\gamma \orthog \rratio^2 - 1}{\gamma \orthog \rratio^2 - 2}
    \\
    &= \frac{\rmax^2}{\rmin^2 \gamma^2} \cdot \frac{\gamma \orthog \rratio^2}{\gamma \orthog \rratio^2 - 2}
    \\
    &= \frac{\rratio^2}{\gamma^2} \left(1 + \frac{2}{\gamma \orthog \rratio^2 - 2}  \right)~.
\end{align*}

\section{Proofs for Sub-Gaussian Marginals}\label{appendix:subgaussian}
In this section we prove Lemma~\ref{lemma:sg.nearlyorthogonal} and Theorem~\ref{thm:subgaussian.sumsiyixi} as well as Corollary~\ref{cor:subgaussian.linearmaxmargin}, Corollary~\ref{cor:subgaussian.leakykkt}, and 
Corollary~\ref{cor:pgausresult}.
A rough outline of our proof strategy is as follows.
\begin{enumerate}
\item First, we show in Lemma~\ref{lemma:sg.concentration.testerror} that in order for a linear classifier $x\mapsto \sgn(\sip wx)$ to achieve a test error near the noise rate, it suffices for $\nicefrac{\snorm{[\cov^{1/2} w]_{2:d}}}{\sqrt{\lambda_1} [w]_1}$ to be small.   
\item Next, we show in Lemma~\ref{lemma:subgaussian.nearly.orthogonal} a number of properties of the training data $\{(x_i,y_i)\}_{i=1}^n$ that hold with high probability under Assumptions~\ref{a:samples.sg} through~\ref{a:covnlogn}.  Lemma~\ref{lemma:sg.nearlyorthogonal} will hold as a deterministic consequence of this lemma, so that the training data are $\orthog$-orthogonal for large $p$ (recall Definition~\ref{def:nearlyorthogonal}) and the norms of all of the examples are close to each other.  This allows for us to apply Proposition~\ref{prop:bound.lambdas.max.margin.v2} and Proposition~\ref{prop:kkt.leaky.nearly.orthogonal}, which show that the KKT points of both the linear max-margin problem~\eqref{eq:linear.max.margin} and the leaky ReLU max-margin ~\eqref{eq:margin.maximization.problem} are $\unifclass$-uniform w.r.t. $\{(x_i,y_i)\}_{i=1}^n$.  %
\item By the first step, to prove Theorem~\ref{thm:subgaussian.sumsiyixi}, it suffices to show that a $\unifclass$-uniform $w\in \R^d$ is such that:
\begin{enumerate}
\item[$(i)$] the norm $\snorm{[\cov^{1/2} w]_{2:d}}$ is small, and
\item[$(ii)$] the first component $[w]_1$ is large and positive.
\end{enumerate} 
Recall that a $\unifclass$-uniform $w$ takes the form $\summ i n s_i y_i x_i$ where $\max_{i,j}\nicefrac{s_i}{s_j}\leq \unifclass$.   For $(i)$, Lemma~\ref{lemma:subgaussian.nearly.orthogonal} provides bounds on $\snorm{[\cov^{1/2} x_i]_{2:d}}$ for training examples $x_i$, which is the basic building block to this part.  For $(ii)$, note that for clean examples $i\in \calC\subset [n]$ (where $y_i=\tilde y_i$), $[s_i y_i x_i]_1 = s_i |\xik|$, while for noisy examples $i\in \calN \subset [n]$ (where $y_i=-\tilde y_i$), $[s_i y_i x_i]_1 = -s_i |\xik|$.  Thus, it suffices to characterize the following,
\begin{align*} 
 \l[\summ i n s_i y_i x_i\r]_1 &= \sum_{i\in \calC} s_i |\xik| - \sum_{i\in \calN} s_i |\xik| \\
 &= \summ i n s_i|\xik| - 2 \sum_{i\in \calN} s_i |\xik|\\
 &\geq \min_i s_i \summ i n|\xik| - 2\max_i s_i \sum_{i\in \calN}  |\xik|.
 \end{align*}
Lemma~\ref{lemma:sg.sum.siyixi.firstcomponent} directly bounds each of the terms above.  The proof of Theorem~\ref{thm:subgaussian.sumsiyixi} is then a direct calculation based on $(i)$ and $(ii)$ above.
\item Corollaries~\ref{cor:subgaussian.linearmaxmargin} and~\ref{cor:subgaussian.leakykkt} follow by a direct calculation based on Lemma~\ref{lemma:sg.nearlyorthogonal} and Theorem~\ref{thm:subgaussian.sumsiyixi}.  
Corollary~\ref{cor:pgausresult}
follows by a direct calculation that verifies $\pgaus$ satisfies the required properties. 
\end{enumerate}

\subsection{Preliminary concentration inequalities}
We first show that the test error of any linear classifier $w\in \R^d$ satisfying $[w]_1> 0$ is close to the noise rate whenever $\nicefrac{\snorm{[\cov^{1/2} w]_{2:d}}}{\sqrt {\lambda_1}[w]_1}$ is small.  
\begin{lemma}\label{lemma:sg.concentration.testerror}
There exists an absolute constant $c_1\geq 2$ such that 
provided $w\in \R^d$ is such that $[w]_1>0$, 
then the following holds.  If $[\cov^{1/2} w]_{2:d}=0$ then $\P_{(x,y)\sim \psubgnoise} \big(y \neq \sgn(\sip{w}{x})\big)\leq \eta$. Otherwise, 
\[ \P_{(x,y)\sim \psubgnoise} \big(y \neq \sgn(\sip{w}{x})\big)\leq \eta +  \f{ c_1\snorm{[\cov^{1/2}w]_{2:d}}}{\sqrt{\lambda_1} [w]_1} \cdot\l( 1 +  \sqrt{0 \vee \log \l( \f{  \sqrt{\lambda_1} [w]_1} {\snorm{[\cov^{1/2}w]_{2:d}} }\r)} \r) .\]
\end{lemma}
\begin{proof}
By definition of $\psubgnoise$, we have,
\begin{align*}
\P_{(x,y)\sim \psubgnoise} \big(y \neq \sgn(\sip{w}{x})\big)
&=\P_{(x,y)\sim \psubgnoise} (y \sip{w}{x} < 0) \\
&= \P_{(x,y)\sim \psubgnoise}(y \sip{w}{x}<0,\, y =- \sgn(\xk))\\
&\qquad + \P_{(x,y)\sim \psubgnoise}(y \sip{w}{x}<0,\, y = \sgn(\xk))\\
&\leq \eta + \P_{(x,y)\sim \psubgnoise}(y \sip{w}{x}<0,\, y = \sgn(\xk)).\numberthis \label{eq:coupling}
\end{align*}
Denoting $[u]_{2:d}\in \R^{d-1}$ as the last $d-1$ components of the vector $u\in \R^d$, we have,
\begin{align*}
    \P_{(x,y)\sim \psubgnoise}( y \sip{w}{x} < 0,\ y = \sgn(\xk) ) &= \P_{(x,y)\sim \psubgnoise}(|\xk| [w]_1 < - \sgn(\xk) \sip{[w]_{2:d}}{[x]_{2:d}}) \\
    &\overset{(i)}= \P ( \sqrt{\lambda_1} [w]_1 |\zk| < - \sgn(\zk) \sip{[w]_{2:d}}{[\cov^{1/2} z]_{2:d}})\\
    &\overset{(ii)}= \P\l( |\zk| < - \f{ \sgn(\zk) \sip{[\cov^{1/2} w]_{2:d}}{[z]_{2:d}}}{\sqrt{\lambda_1} [w]_1} \r)~. \numberthis \label{eq:sg.test.error.prelim}
\end{align*}
Equality $(i)$ uses that $x=\cov^{1/2} z$.
Equality $(ii)$ uses the assumption that $[w]_1> 0$.  From here, we see that if $[\cov^{1/2} w]_{2:d}=0$ then we have  $\P_{(x,y)\sim \psubgnoise}( y \sip{w}{x} < 0,\ y = \sgn(\xk) ) =  \P( |\zk| < 0) = 0$, which by~\eqref{eq:coupling} shows that $\P_{(x,y)\sim \psubgnoise} \big(y \neq \sgn(\sip{w}{x})\big) \leq \eta$.  Thus in the remainder of the proof we shall assume $[\cov^{1/2} w]_{2:d}\neq 0$.

Let us define the term 
\[ \rho := \f{ \sgn(\zk) \sip{[\cov^{1/2} w]_{2:d} }{[z]_{2:d}}}{\sqrt{\lambda_1} [w]_1} .\]
This term is small in absolute value when $\sqrt{\lambda_1}[w]_1 \gg \snorm{[\cov^{1/2} w]_{2:d}}$.  In particular,
since $z$ is a sub-Gaussian random vector with sub-Gaussian norm at most $\sgnormz$, by Hoeffding's inequality we have that for some $c>0$ and any $t\geq 0$,
\begin{align*}
    \P( |\rho| \geq t) &= \P\l( \f{ |\sip{[z]_{2:d}}{[\cov^{1/2} w]_{2:d}}|}{\sqrt{\lambda_1}[w]_1} \geq t \r) \leq 2 \exp \l( - \f{ c \lambda_1 [w]_1^2 t^2}{\sgnormz^2 \snorm{[\cov^{1/2} w]_{2:d}}^2} \r). \numberthis \label{eq:z2d.covw2d.hoeffding}
\end{align*}
Continuing from~\eqref{eq:sg.test.error.prelim}, we get for any $t \geq 0$,
\begin{align*}
    \P_{(x,y)\sim \psubgnoise}( y \sip{w}{x} < 0,\ y = \sgn(\xk) ) &= \P(|\zk| < - \rho) \\
    &= \P(|\zk| < - \rho,\, |\rho| \geq t) + \P(|\zk| < - \rho, |\rho| < t) \\
    &\leq \P(|\rho| \geq t) + \P(|\zk| < t) \\
    &\overset{(i)}\leq 2 \exp \l( - \f{ c \lambda_1 [w]_1^2 t^2}{\sgnormz^2 \snorm{[\cov^{1/2} w]_{2:d}}^2} \r) + \beta t.
\end{align*}
In inequality $(i)$ we have used~\eqref{eq:z2d.covw2d.hoeffding} as well as the assumption that $\P(|\zk| \leq t) \leq \beta t$ for any $t\geq 0$.  In particular, if we let
\[ \xi := \f{  \snorm{[\cov^{1/2} w]_{2:d}}}{\sqrt{\lambda_1} [w]_1}, \quad t := c^{-1/2} \sgnormz\xi\sqrt{0 \vee \log(1/\xi) },\]
then we have, 
\begin{align*}
    \P_{(x,y)\sim \psubgnoise}( y \sip{w}{x} < 0,\ y = \sgn(\xk) ) &\leq 2 \exp \l( - \f{ ct^2}{\sgnormz^2 \xi^2}\r) + \beta t\\
    &= 2(1\wedge \xi) +  \xi \beta c^{-1/2} \sgnormz \sqrt{0 \vee \log(1/\xi)})\\
    &\leq \xi \cdot  \max(2, \beta c^{-1/2} \sgnormz) ( 1+ \sqrt{0 \vee \log(1/\xi)}).
\end{align*} 
The proof follows by letting $c_1 = \max(2, \beta c^{-1/2} \sgnormz)$. 
\end{proof}

The following lemma characterizes a number of useful properties about the training data.
\begin{lemma}\label{lemma:subgaussian.nearly.orthogonal}
There exists an absolute constant $C_0>1$ such that for every large enough $C>1$ (with $C, C_0$ depending only on $\sgnormz$) 
and for any $\delta \in (0,\nicefrac 12)$, under Assumptions~\ref{a:samples.sg} through~\ref{a:covnlogn} (defined for these $C$ and $\delta$), the following holds with probability at least $1-2\delta$ over $\psubgnoise^n$:
\begin{enumerate}
\item The norms of the samples satisfy,
\[ \l| \f{ \snorm{x_i}}{\sqrt{\tr(\cov)}} -1 \r| \leq C_0  \sqrt{ \f{ \snorm{\cov}_2 \log(6n/\delta)}{\tr(\cov)}}, \quad\text{for all $i\in [n]$},\]
and
\begin{align*} 
\l| \f{ \snorm{[\cov^{1/2} x_i]_{2:d}}}{\sqrt{\tr(\cov_{2:d}^2)}} -1 \r| &\leq C_0  \sqrt{ \f{ \snorm{\cov_{2:d}^2}_2 \log(6n/\delta)}{\tr(\cov_{2:d}^2)}}.
\end{align*} 
\item The correlations of distinct samples satisfy,
\[|\sip{x_i}{x_j}| \leq C_0 \sqrt{\tr(\cov^2)} \log(6n^2/\delta),\quad \text{for all $i\neq j$}.\]
\item The samples satisfy,
\begin{align*}
    \min_i \snorm{x_i}^2  \geq \f{ \tr(\cov)}{C_0  \sqrt{\tr(\cov^2)} \log(6n^2/\delta)} \cdot \f{ \max_i \snorm{x_i}^2}{\min_i \snorm{x_i}^2} \cdot \max_{i\neq j}  |\sip{x_i}{x_j}|.
\end{align*}
\end{enumerate}
\end{lemma}
\begin{proof}

We prove the lemma in parts.

\paragraph{Part 1: norms of samples.}  We first show concentration of the norms.  Let $x \in \{ x_1, \dots, x_n\}$.  
Recall that $x = \cov^{1/2} z$ where the components of $z$ are independent, mean-zero, and $z$ has sub-Gaussian norm at most $\sgnormz$, with $\E[zz^\top] = I_d$.
We can thus apply Hanson-Wright inequality~\citep[Theorem 6.3.2]{vershynin.high}, so that there is an absolute constant $c>0$ such that we have for any $t \geq 0$,
\begin{align*}
\P\l( \l| \snorm{x} - \sqrt{\tr(\cov)}\r| > t\r) \leq 2 \exp \l( -  \f{ct^2}{\sgnormz^4 \snorm{\cov}_2}\r),
\end{align*}
where we have used that $\snorm{\cov^{1/2}}_F^2 = \summ i d \lambda_i = \tr(\cov)$.   Choosing $t = c^{-1/2} \sgnormz^2 \sqrt{ \snorm{\cov}_2 \log(6n/\delta)}$ and using a union bound over $x\in \{x_1,\dots, x_n\}$, we see that,
\begin{align*}
\P\l(\exists i\in [n]:\ \l|  \snorm{x_i}- \sqrt{\tr(\cov)}\r| > c^{-1/2} \sgnormz^2 \sqrt{ \snorm{\cov}_2 \log(6n/\delta)} \r)  < \delta/3. \numberthis \label{eq:norm.concentration.highp}
\end{align*}
We now show a bound on $\snorm{\cov^{1/2} x_i}^2 = z_i^\top \cov^2 z_i$.  Again fix $x\in \{x_1, \dots, x_n\}$.    We can employ a nearly identical argument to above: since $\snorm{\cov}_F^2 = \tr(\cov^2)$,
by Hanson-Wright inequality, for any $t \geq 0$,
\begin{align*}\numberthis \label{eq:hansonwright.cov1/2x}
\P\l( \l| \snorm{\cov^{1/2}x} - \sqrt{\tr(\cov^2)}\r| > t\r) \leq 2 \exp \l( - \f{ct^2}{\sgnormz^4 \snorm{\cov^2}_2} \r),
\end{align*}
Choosing $t=c^{-1/2} \sgnormz^2 \sqrt{\snorm{\cov^2}_2 \log(6n/\delta)}$
and noting that $\tr(\cov^2)\geq \snorm{\cov^2}_2$ implies
\[ \sqrt{\tr(\cov^2)} + c^{-1/2} \sgnormz^2 \sqrt{\snorm{\cov^2}_2 \log(6n/\delta)} \leq (1+ c^{-1/2} \sgnormz^2) \sqrt{\tr(\cov^2) \log(6n/\delta)}, \]
by a union bound we get
\begin{align*}
\P\l( \exists i\in [n]: \snorm{\cov^{1/2} x_i} > (1+c^{-1/2} \sgnormz^2) \sqrt{\tr(\cov^2) \log(6n/\delta)} \r) \leq  \delta/ 3. \numberthis \label{eq:norm.concentration.highp.extracov}
\end{align*}
Using a completely identical argument used to derive \eqref{eq:norm.concentration.highp}, we also have
\begin{align}
\P\l( \exists i\in [n]: \l| \snorm{[\cov^{1/2} x_i]_{2:d}} - \sqrt{\tr(\cov^2_{2:d})} \r|  >c^{-1/2} \sgnormz^2 \sqrt{\snorm{\cov_{2:d}^2} \log(6n/\delta)} \r) \leq  \delta/ 3. \label{eq:cov12xi.norm.ub}
\end{align}

\paragraph{Part 2: correlations of samples.}  We now bound the correlation between distinct samples.   
Let us fix $j\in [n]$ and consider $i\in [n]\setminus \{j\}$.  Then there is a sub-Gaussian random vector $z_i$ with sub-Gaussian norm at most $\sgnormz$ such that $\sip{x_i}{x_j} = z_i^\top \cov^{1/2} x_j$.  In particular, $\sip{x_i}{x_j} = z_i^\top \xi \cdot \snorm{\cov^{1/2} x_j}$ where $\xi$ is a unit-norm vector.  Since $z_i$ is a sub-Gaussian random vector, this means that for some $c>0$ and any $t>0$,
\begin{align*}
&\P\l( \l| \sip{x_i}{x_j} \r| > t\Big| \snorm{\cov^{1/2} x_j} \leq (1+c^{-1/2} \sgnormz^2) \sqrt{\tr(\cov^2) \log(6n/\delta)}\r) \\
&\quad \leq 2 \exp\l( - c \cdot \f{ t^2}{\sgnormz^2 (1+c^{-1/2} \sgnormz^2)^2 \tr(\cov^2) \log(6n/\delta)}\r).\numberthis\label{eq:pairwise.correlations.conditional}
\end{align*}
Letting $c' = c/[\sgnormz^2 (1+c^{-1/2}\sgnormz^2)^2]$, we can thus bound,
\begin{align*}
&\P\l( \exists i\neq j:\ |\sip{x_i}{x_j}| > t \r) \\
&\quad \leq \P\l(\exists i\neq j: |\sip{x_i}{x_j}| > t \Big| \snorm{\cov^{1/2} x_j} \leq (1+c^{-1/2} \sgnormz^2) \sqrt{\tr(\cov^2) \log(6n/\delta)} \r)\\
&\qquad + \P\l(\exists j: \snorm{\cov^{1/2} x_j} >  (1+c^{-1/2} \sgnormz^2) \sqrt{\tr(\cov^2) \log(6n/\delta)}  \r)\\
&\quad \overset{(i)}\leq 2n^2 \exp\l( -\f{ c't^2}{\tr(\cov^2) \log(6n/\delta)}\r) + \f \delta 3,\numberthis \label{eq:pairwise.correlations.pt2}
\end{align*}
where $(i)$ uses~\eqref{eq:norm.concentration.highp.extracov}.  Choosing $t = (c')^{-1/2} \sqrt{\tr(\cov^2)} \log(6n^2/\delta)$ and using~\eqref{eq:pairwise.correlations.conditional}, we get
\begin{align*}
&\P\l( \exists i\neq j:\ |\sip{x_i}{x_j}| > (c')^{-1/2} \sqrt{\tr(\cov^2)} \log(6n^2/\delta) \r) \leq \f {2\delta} 3.\numberthis \label{eq:pairwise.correlations.pt3}
\end{align*}
Combining the above display with~\eqref{eq:norm.concentration.highp} and~\eqref{eq:cov12xi.norm.ub}
and using a union bound, we get that with probability at least $1-2\delta$,
\begin{equation}
\begin{cases}
\l| \f{ \snorm{x_i}}{\sqrt{\tr(\cov)}} -1 \r| \leq c^{-1} \sgnormz^2 \sqrt{ \f{ \snorm{\cov}_2 \log(6n/\delta)}{\tr(\cov)}}, &\text{for all $i\in [n]$},\\
\l| \f{ \snorm{[\cov^{1/2} x_i]_{2:d}}}{\sqrt{\tr(\cov_{2:d}^2)}} -1 \r| \leq c^{-1} \sgnormz^2 \sqrt{ \f{ \snorm{\cov_{2:d}^2}_2 \log(6n/\delta)}{\tr(\cov_{2:d}^2)}}, &\text{for all $i\in [n]$},\\
|\sip{x_i}{x_j}| \leq (c')^{-1/2} \sqrt{\tr(\cov^2) }\log(6n^2/\delta),&\text{for all $i\neq j$}.
\end{cases}\label{eq:subg.norms.correlations.collected}
\end{equation}
This completes the first two parts of the lemma.

\paragraph{Part 3: near-orthogonality of samples.} We now show an upper bound on $\rratio = \max_{i,j} \snorm{x_i}/\snorm{x_j}$.  Note that with probability at least $1-2\delta$,~\eqref{eq:subg.norms.correlations.collected} holds, and we shall show that this implies that for an absolute constant $C_0>0$ we have,
\begin{align*}
    \f{ \min_i \snorm{x_i}^2}{\rratio^2 \max_{i\neq j} |\sip{x_i}{x_j}|} \geq \f{ \tr(\cov)}{C_0 \sgnormz \sqrt{\tr(\cov^2)} \log(6n^2/\delta)}.
\end{align*} 
By Assumption~\ref{a:covnlogn} we have
\begin{equation} \label{eq:xi.ub}
\xi := \f{ \snorm{\cov}_2 \log(6n/\delta)}{\tr(\cov)} \leq \f{ \snorm{\cov}_F \log(6n/\delta)}{\tr(\cov)} = \f{ \sqrt{\tr(\cov^2)} \log(6n/\delta)}{\tr(\cov)} \leq \f{1}{C}.
\end{equation}
Thus by~\eqref{eq:subg.norms.correlations.collected}, we see that the quantity $\rratio = \max_{i,j} \snorm{x_i}/\snorm{x_j}$ satisfies 
\begin{align*}
\rratio&= \max_{i,j} \f{\snorm{x_i}}{\snorm{x_j}}\\
&\leq \f{ 1 + c^{-1} \sgnormz^2 \sqrt{ \xi}}{ 1 - c^{-1} \sgnormz^2 \sqrt{\xi}}\\
&\overset{(i)}\leq \f{ 1 + c^{-1} \sgnormz^2/\sqrt{C}}{ 1 - c^{-1} \sgnormz^2/\sqrt{C}}\\
&\overset{(ii)} \leq \l( 1 + \f{ 2c^{-1} \sgnormz^2 }{\sqrt C}\r)^2. \numberthis \label{eq:sg.rratio.ub.origproof}
\end{align*}
Inequality $(i)$ follows by~\eqref{eq:xi.ub}, while $(ii)$ uses the inequality $1/(1-x)\leq 1+2x$ on $[0,1/2]$ and holds for $C>1$ large enough.
In particular, by taking $C$ larger we can guarantee $R$ is closer to one.  

Next, we have by part 1 and part 2 of this lemma,
\begin{align*}
\f{ \min_i \snorm{x_i}^2}{\max_{i\neq j} |\sip{x_i}{x_j}|} &\geq \f{ \tr(\cov) \l( 1 - c^{-1} \sgnormz^2 \sqrt{\xi} \r)}{ (c')^{-1/2}  \sqrt{\tr(\cov^2)} \log(6n^2/\delta)} \geq \f{ \tr(\cov) \l( 1 - c^{-1}\sgnormz^2/\sqrt{C} \r)}{ (c')^{-1/2}\sqrt{\tr(\cov^2)} \log(6n^2/\delta)}.
\end{align*}
For $C>100c^{-2} \sgnormz^4$,
by~\eqref{eq:sg.rratio.ub.origproof} we have $1-c^{-1} \sgnormz^2/\sqrt{C} \geq 0.9$ and $\rratio^{-2}\geq 0.68$.  We therefore see that for $C$ large enough,
\begin{align*}
    \f{ \min_i \snorm{x_i}^2}{\rratio^2 \max_{i\neq j} |\sip{x_i}{x_j}|} \geq \f{ \tr(\cov)}{2 (c')^{-1/2}\sqrt{\tr(\cov^2)} \log(6n^2/\delta)}.
\end{align*}

\end{proof}

\subsection{Proof of Lemma~\ref{lemma:sg.nearlyorthogonal}}

We next prove Lemma~\ref{lemma:sg.nearlyorthogonal}: as $C$ grows in Assumptions~\ref{a:samples.sg} through~\ref{a:covnlogn}, the training data become $\orthog$-orthogonal for large $p$ and $\max_{i,j} \nicefrac{\snorm{x_i}^2 }{\snorm{x_j}^2} \to 1$.  

\sgnearlyorthogonal*
\begin{proof}
First, note that all of the results in Lemma~\ref{lemma:subgaussian.nearly.orthogonal} hold with probability at least $1-2\delta$.  We shall show that the training data being $C/C_1$-orthogonal and that $\max_{i,j} \nicefrac{\snorm{x_i}^2}{\snorm{x_j}^2}$ are a deterministic consequence of this high-probability event.   By Lemma~\ref{lemma:subgaussian.nearly.orthogonal}, 
\begin{align*}
    \min_i \snorm{x_i}^2  \geq \f{ \tr(\cov)}{C_0  \sqrt{\tr(\cov^2)} \log(6n^2/\delta)} \cdot \f{ \max_i \snorm{x_i}^2}{\min_i \snorm{x_i}^2} \cdot \max_{i\neq j}  |\sip{x_i}{x_j}|.
\end{align*}
By Assumption~\ref{a:covnlogn}, this means
\begin{align*}
    \min_i \snorm{x_i}^2  \geq \f{ C}{C_0} \cdot  \f{ \max_i \snorm{x_i}^2}{\min_i \snorm{x_i}^2} \cdot n\max_{i\neq j}  |\sip{x_i}{x_j}|.
\end{align*}
In particular, the training data is $C/C_0$-orthogonal (see Definition~\ref{def:nearlyorthogonal}). 

For the ratio $\rratio = \max_{i,j} \snorm{x_i}/\snorm{x_j}$, if we let $\xi := \f{ \snorm{\cov}_2 \log(6n/\delta)}{\tr(\cov)}$ then by part 1 of Lemma~\ref{lemma:subgaussian.nearly.orthogonal} we have
\[ \sqrt{\tr(\cov)} ( 1 - C_0 \sqrt{\xi}) \leq \snorm{x_i} \leq \sqrt{\tr(\cov)}(1 + C_0 \sqrt{\xi}). \]
By Assumption~\ref{a:covnlogn} we know $\xi \leq 1/C$ (see~\eqref{eq:xi.ub}).  Therefore for $C>1$ large enough, using $1/(1-x)\leq 1+2x$ for $x\in [0,1/2]$,
\begin{align*}
\rratio&= \max_{i,j} \f{\snorm{x_i}}{\snorm{x_j}}\leq \f{ 1 + C_0\sqrt{\xi}}{1 - C_0 \sqrt \xi} \leq  \f{ 1 + C_0 /\sqrt C}{ 1 - C_0/ \sqrt{C}} \leq \l( 1 + \f{ 2C_0  }{\sqrt C}\r)^2~.
\end{align*}
This completes the claimed upper bound for $\rratio$. 
\end{proof}

\subsection{Proof of Theorem~\ref{thm:subgaussian.sumsiyixi}}

We now begin to prove that if $w$ is $\unifclass$-uniform, i.e. there are strictly positive $s_i$, $i=1,\dots, n$ such that $w = \summ i n s_i y_i x_i$ and $\max_{i,j} \nicefrac{s_i}{s_j}\leq \unifclass$, then the first component of $w$ is large and positive.  By Lemma~\ref{lemma:sg.concentration.testerror}, this is one step towards showing the test error of this linear predictor is close to the noise rate.  

To begin, note that since $y_i=\sgn(\xik)$ for $i\in \calC$ and $y_i=-\sgn(\xik)$ for $i\in \calN$, we have,
\[ \l[\summ in s_i y_i x_i\r]_1= \sum_{i\in \calC} s_i |\xik| - \sum_{i\in \calN} s_i |\xik| = \summ i n s_i |\xik| - 2 \sum_{i\in \calN} s_i |\xik|.\]
Thus, in order to show that this quantity is large and positive, we would like to show the first term is large and positive while the second term is not too negative.  We do so in the following lemma. 

\begin{lemma}\label{lemma:signal.noise.first.component}
There exists a universal constant $C_1'>1$ (depending only on $\eta$ and $\sgnormz$) such that for any $\delta \in (0, \nicefrac 13)$, if $n\geq C_1' \log(2/\delta)$ then with probability at least $1-3\delta$ over the training data $\{(x_i,y_i)\}_{i=1}^n\sim \psubgnoise^n$,
the following holds:
\begin{align*}
\summ i n |\xik| &\geq n \sqrt{\lambda_1} \E[|\zk|] \l( 1 - C_1' \beta \sqrt{\f{\log(2/\delta)}n}\r),\quad\text{and}\\
\sum_{i \in \calN} |\xik| &\leq n \sqrt{\lambda_1} \E[|\zk|]   \l( \eta + C_1' \beta \sqrt{\f{\log(2/\delta)}n}\r).  
\end{align*}
\end{lemma}
\begin{proof}
By definition, there are i.i.d. $z_i\sim \pz$ such that $x_i = \cov^{1/2} z_i$.  In particular, $\xik = \sqrt{\lambda_1} \zik$, so thus it suffices to bound the sum $\summ i n |\zik| = \lambda_1^{-1/2} \summ i n |\xik|$ from below 
and the sum $\sum_{i\in \calN} |\zik|$ from above.

Let us denote $\alpha := \E|\zik|$.  Note that since $\P(|\zk| \leq t)\leq \beta t$, by taking $t = 1/(2\beta)$ we see that
\begin{equation}\label{eq:alpha.lb}
\alpha = \E|\zik| \geq \f 1{4\beta}.
\end{equation}
The quantity $|\zik| - \alpha$ with sub-Gaussian norm at most $c_1\sgnormz$ for some absolute constant $c_1>0$~\citep[Lemma 2.6.8]{vershynin.high}, 
and is i.i.d. over indices $i\in [n]$.   Therefore, by Hoeffding's inequality, this means that for some absolute constant $c>0$ and any $t\geq 0$,
\begin{align*}
\P\l(\l| \f 1 {n} \summ i n \l(|\zik| - \alpha \r) \r| \geq t\r) &\leq 2 \exp \l( - \f{ cnt^2}{ \sgnormz^2}\r).
\end{align*}
Choosing $t = c^{-1/2} \sgnormz \sqrt{\log(2/\delta) / n}$ and using~\eqref{eq:alpha.lb} we get that with probability at least $1-\delta$,
\begin{equation}
\l| \f{1}{n} \summ i n ( |\zik| - \alpha)  \r| \leq c^{-1/2} \sgnormz \sqrt{\f{ \log(2/\delta)}{n} } \implies \summ i n |\zik| \geq n \alpha \l( 1 - 4 c^{-1/2} \sgnormz \beta\sqrt{\f{\log(2/\delta)}n}\r). \label{eq:clean.signal.variance}
\end{equation} 
Using the same argument (and assuming without loss of generality that $|\calN| > 0$, since otherwise we can just ignore this term entirely), we get with probability at least $1-\delta$, 
\begin{equation}
\l| \f{1}{|\calN|} \sum_{i\in \calN}( |{\zik}| - \alpha)  \r| \leq c^{-1/2} \sgnormz  \sqrt{\f{ \log(2/\delta)}{|\calN|} }. \label{eq:noisy.signal.variance.intermediate}
\end{equation}
From here we see it is necessary to control the number of noisy points.  The number of noisy points $|\calN|$ is the sum of $n$ independent, identically distributed random variables with mean $\eta$.  Thus, by Hoeffding's inequality, 
for any $u\geq 0$, 
\begin{align*}
\P\l( \l|  |\calN| - n \eta \r| \geq u\r) \leq 2 \exp \l( - \f{ 2 u^2}{n} \r). 
\end{align*}
In particular, selecting $u = \sqrt{n \log(2/\delta)/2}$,
we see that with probability at least $1-\delta$,
\begin{align*}
\l| \f{|\calN|}{n} - \eta \r| &\leq \sqrt{\f{ \log(2/\delta)}{n}}.
\end{align*}
Rearranging we see that
\[ \eta n - \sqrt{n \log(2/\delta)} \leq |\calN| \leq \eta n + \sqrt{n \log(2/\delta)}.\]
Since $\eta$ is an absolute constant, using the lemma's assumption that $n\geq C_1' \log(2/\delta)$ we get for $C_1'$ large enough relative to $\eta^{-2}$,
\[ \eta n - \sqrt{n\log(2/\delta)} = \eta n \l( 1 - \sqrt{ \f{ \eta^{-2} \log(2/\delta)}n} \r) \geq \f 1 2 \eta n,\]
and therefore
\begin{align*}\numberthis \label{eq:noisy.count}
\f 12 \eta n  \leq |\calN| \leq \eta n + \sqrt{n \log(2/\delta)}
\end{align*}
Substituting the two previous displays into~\eqref{eq:noisy.signal.variance.intermediate} we get
\begin{align*}
\sum_{i\in \calN} |\zik| &\leq |\calN| \alpha \l( 1 + c^{-1/2} \sgnormz \alpha^{-1} \sqrt{\f{\log(2/\delta)}{|\calN|}} \r) \\
&\leq n \alpha \l( \eta + \sqrt{\f{\log(2/\delta)}n} \r) \cdot \l( 1 + 4c^{-1/2} \sgnormz \beta \sqrt{\f{ 2\eta^{-1}\log(2/\delta)}n}\r) \\
&\leq n \alpha \l( \eta + 12 c^{-1/2} \sgnormz \beta \sqrt{\f{2 \eta^{-1} \log(2/\delta)}n} \r).
\end{align*}
The second inequality uses~\eqref{eq:alpha.lb}.  The last inequality uses the lemma's assumption that $n\geq C_1' \log(2/\delta)$ for a large enough $C_1'$ and that $\eta^{-1}$ is an absolute constant.    Taking a union bound over the three events and taking $C_1'$ large enough completes the proof since $\sgnormz$ and $\eta$ are absolute constants. 
\end{proof}

We now show that a $\unifclass$-uniform classifier $u$ has a large and positive first component while $\snorm{[\cov^{1/2} u]_{2:d}}$ is small with high probability.  By Lemma~\ref{lemma:cluster.concentration.testerror}, this suffices for showing generalization error near the noise rate. 

\begin{lemma}\label{lemma:sg.sum.siyixi.firstcomponent}
Let $\unifclass \geq 1$ be a constant, and suppose $\eta \leq \f{1}{2\unifclass}-\Delta$ for some absolute constants $\eta, \Delta >0$.  There exists an absolute constant $C>1$ (depending only on $\eta, \sgnormz, \beta, \unifclass$, and $\Delta$) such that for any $\delta \in (0, \nicefrac 15)$, under Assumptions~\ref{a:samples.sg} through~\ref{a:covnlogn} (defined for these $C$ and $\delta$), with probability at least $1-5\delta$ over $\psubgnoise^n$, if $u = \summ i n s_i y_i x_i \in \R^d$ is $\unifclass$-uniform w.r.t $\{(x_i, y_i)\}_{i=1}^n\iid \psubgnoise$ then 
\[ [u]_1 \geq \f { \tau \Delta n \alpha \sqrt{\lambda_1}}{8\beta} \l( \min_i s_i \r),\quad \text{and}\quad \snorm{[\cov^{1/2} u]_{2:d} } \leq \f 3 2 n \l( \max_i s_i \r) \sqrt{\tr(\cov^2_{2:d})}. \]
In particular, if $[\cov^{1/2} u]_{2:d} \neq 0$ then 
\[ \f{ \l[  u \r]_1 }{\snorm{[ \cov^{1/2} u]_{2:d}}} \geq \f {\Delta}{12 \beta}  \cdot \sqrt{  \f{\lambda_1}{\tr(\cov_{2:d}^2)   } }. \]
\end{lemma}
\begin{proof}
First, by a union bound, for $C$ sufficiently large the results of both Lemma~\ref{lemma:signal.noise.first.component}
and Lemma~\ref{lemma:subgaussian.nearly.orthogonal} hold with probability at least $1-5\delta$.   In the remainder of the proof we will work on this high-probability event and we shall show the lemma holds as a deterministic consequence of this. 

By definition, $y_i = \sgn(\xik)$ for $i\in \calC$ and $y_i =-\sgn(\xik)$ for $i\in \calN$. 
Since $u$ is $\unifclass$-uniform, there exist strictly positive numbers $s_i$ such that $u = \summ i n s_i y_i x_i$ with $\max_{i,j} \f{s_i}{s_j}\leq \unifclass$.   Thus we can write,
\begin{align*}
\l[u \r]_1 = \l[ \summ i n s_i y_i x_i \r]_1 &= \sum_{i\in \calC} s_i |\xik|  - \sum_{i\in \calN} s_i |\xik| \\
&= \summ i n s_i |\xik| - 2 \sum_{i\in \calN} s_i |\xik| \\
&\geq \min_i s_i \summ i n |\xik| - 2 \max_i s_i \sum_{i\in \calN} |\xik|. \numberthis \label{eq:estimator.decomposition}
\end{align*}
Let us denote $\alpha := \E|\zk|$.  Recall by~\eqref{eq:alpha.lb} that the assumption of anti-concentration on $\zk$ implies $\alpha \geq 1/(4\beta)$.  
Now using Lemma~\ref{lemma:signal.noise.first.component}, we have,
\begin{align*}
\l[ u\r]_1 &\geq n \alpha \sqrt{\lambda_1} \l( \min_i s_i \r) \l( 1 - C_1' \beta \sqrt{\f{\log(2/\delta)}n}\r) \\
&\qquad - n \alpha \sqrt{\lambda_1} \l( \max_i s_i\r)  \l( 2\eta + 2C_1' \beta \sqrt{\f{\log(2/\delta)}n} \r)\\
&\geq n \alpha \sqrt{\lambda_1} \l( \min_i s_i\r) \l[ 1 - C_1' \beta \sqrt{\f{\log(2/\delta)}n} - \unifclass \l( 2\eta + 2C_1' \beta \sqrt{\f{\log(2/\delta)}n} \r) \r]\\
&\overset{(i)}\geq n \alpha \sqrt{\lambda_1} \l( \min_i s_i\r) \l[ 1 - C_1' \beta \sqrt{\f{\log(2/\delta)}n} - \unifclass \l( \f{1}{\unifclass} - 2\Delta + 2C_1' \beta \sqrt{\f{\log(2/\delta)}n} \r) \r].
\end{align*}
The inequality $(i)$ uses the lemma's assumption that $\eta \leq 1/(2\unifclass)- \Delta$.  Rearranging the above, we see
\begin{align*}
\l[u \r]_1 &\geq n \alpha \sqrt{\lambda_1} \l( \min_i s_i\r) \l[ 2\unifclass \Delta -  C_1' \beta \sqrt{\f{\log(2/\delta)}n} -2\unifclass C_1' \beta \sqrt{\f{\log(2/\delta)}n} \r] \\
&\geq \f 1 2 \unifclass \Delta n \alpha \sqrt{\lambda_1} \l( \min_i s_i \r) >0.\numberthis\label{eq:subg.first.component.positive}
\end{align*}
The final inequality uses that $\unifclass,\Delta, \beta$ are absolute constants and by taking $C$ large enough so that $n\geq C \log(2/\delta)$ implies the inequality.  

Next, we want to bound $\snorm{[\cov^{1/2} u]_{2:d}}$.  We will use the first part of Lemma~\ref{lemma:subgaussian.nearly.orthogonal} to do so.  We have, 
\begin{align*}
\norm{ \l[ \cov^{1/2} u\r]_{2:d}} &= \norm{\l[ \textstyle \summ i n s_i y_i \cov^{1/2} x_i\r]_{2:d}} \\
&\leq n \l( \max_i s_i \r) \max_i \norm{\l[ \cov^{1/2} x_i \r]_{2:d}} \\
&\leq n \l( \max_i s_i \r) \sqrt{ \tr(\cov_{2:d}^2)} \l( 1 + C_0 \sqrt{\f{ \snorm{\cov_{2:d}^2}_2 \log(6n/\delta)}{\tr(\cov_{2:d}^2)} }\r)\\
&= n \l( \max_i s_i \r)  \sqrt{ \tr(\cov_{2:d}^2)} \l( 1 + C_0 \sqrt{\f{ \log(6n/\delta)}{\stablerank(\cov_{2:d})} }\r)\\
&\leq \f 3 2 n \l( \max_i s_i \r)  \sqrt{ \tr(\cov_{2:d}^2)}.\numberthis \label{eq:cov1/2u2d.ub}
\end{align*}
The final inequality uses Assumption~\ref{a:srank.sg}
so that $\stablerank(\cov_{2:d}) > C \log(6n/\delta)$ and follows by taking $C$ large enough.  
Putting~\eqref{eq:subg.first.component.positive} and the above together, if $[\cov^{1/2}u]_{2:d}\neq 0$ we get
\begin{align*}
\f{ \l[  u \r]_1 }{\snorm{[ \cov^{1/2}  u]_{2:d}}} &\geq \f 1{3} \unifclass \Delta \alpha \cdot \f{\min_i s_i}{\max_i s_i} \cdot \sqrt{  \f{\lambda_1}{\tr(\cov_{2:d}^2)}} \geq \f{ \Delta\alpha }{3} \cdot \sqrt{  \f{\lambda_1}{\tr(\cov_{2:d}^2)}}. 
\end{align*}
Since by~\eqref{eq:alpha.lb} we have $\alpha \geq 1/(4\beta)$, this completes the proof. 
\end{proof}

We are now in a position to prove Theorem~\ref{thm:subgaussian.sumsiyixi}.  For the reader's convenience, we re-state it below.

\subgaussiansumsiyixi*
\begin{proof}
By a union bound, with probability at least $1-7\delta$, the results of both Lemma~\ref{lemma:sg.sum.siyixi.firstcomponent} and Lemma~\ref{lemma:subgaussian.nearly.orthogonal} hold, and we showed previously that Lemma~\ref{lemma:sg.nearlyorthogonal} is a deterministic consequence of Lemma~\ref{lemma:subgaussian.nearly.orthogonal} and the Assumptions~\ref{a:samples.sg} through~\ref{a:covnlogn}.    In the remainder of the proof we will work on this event and show that the theorem holds as a consequence of these lemmas and Assumptions~\ref{a:samples.sg} through~\ref{a:covnlogn}.  

Since $u$ is $\unifclass$-uniform, there exist strictly positive constants $s_i$ such that $u = \summ i n s_i y_i x_i$.  We first show that $u$ interpolates the training data: for any $k\in [n]$ we have
\begin{align*}
\sip{u}{y_k x_k} &= s_k \snorm{x_k}^2 + \sum_{i\neq k} \sip{s_i y_i x_i}{y_k x_k} \\
&\geq s_k \snorm{x_k}^2 - n \max_i s_i \cdot \max_{i\neq j} |\sip{x_i}{x_j}| \\
&= s_k \snorm{x_k}^2 \l( 1 - \f{ n \max_i s_i \cdot \max_{i\neq j} |\sip{x_i}{x_j}|}{s_k \snorm{x_k}^2 } \r) \\
&\geq s_k \snorm{x_k}^2 \l( 1 - \f{ n \unifclass \max_{i\neq j} |\sip{x_i}{x_j}|}{\snorm{x_k}^2}\r)\\
&\overset{(i)}\geq s_k \snorm{x_k}^2 \l( 1 - \f{C_1 \unifclass}{C} \r)\\
&\overset{(ii)}\geq \f 12 s_k \snorm{x_k}^2 .\numberthis \label{eq:u.interpolate.intermediate}
\end{align*}
The inequality $(i)$ uses that the training data is $C/C_1$-orthogonal by Lemma~\ref{lemma:sg.nearlyorthogonal}, while $(ii)$ follows by taking $C\geq 2 C_1 \unifclass$.  
This last quantity is strictly positive by Lemma~\ref{lemma:subgaussian.nearly.orthogonal}.
Thus, $u$ interpolates the training data.

We now show the generalization error is close to the noise rate.  
By Lemma~\ref{lemma:sg.concentration.testerror}, if $[\cov^{1/2}u]_{2:d} =0$ then since $[u]_1>0$ by Lemma~\ref{lemma:sg.sum.siyixi.firstcomponent}, we have $\P_{(x,y)\sim\psubgnoise} \big(y \neq \sgn(\sip {u}{x}) \big) \leq \eta$ and the proof is complete. 

 Thus consider the case that $[\cov^{1/2} u]_{2:d} \neq 0$.  Let $c := \Delta/(12 \beta)$, where $c<1$ is an absolute constant (assuming w.l.o.g. $\beta \geq 1$) as $\Delta, \beta$ are absolute constants by assumption.    
Then by Lemma~\ref{lemma:sg.sum.siyixi.firstcomponent} we have, 
\begin{equation} 
\f{ [u]_1}{\snorm{[\cov^{1/2} u]_{2:d}}} \geq \f{\Delta}{12 \beta}\sqrt{\f{ \lambda_1}{ \tr(\cov^2_{2:d})}} = c \sqrt{\f{ \lambda_1}{ \tr(\cov^2_{2:d})}} \label{eq:u1.over.cov12u2d}.
\end{equation}
Applying Lemma~\ref{lemma:sg.concentration.testerror} 
there exists $c_1\geq 2$ such that
\begin{align*} \numberthis \label{eq:subg.testerr.prelim}
&\P_{(x,y)\sim\psubgnoise} \big(y \neq \sgn(\sip {u}{x}) \big) \leq \eta + \f{ c_1 \snorm{[\cov^{1/2} u]_{2:d}} }{\sqrt \lambda_1 [u]_1} \l( 1 + \sqrt{0\vee \log\l( \f{ \sqrt{\lambda_1} [u]_1}{\snorm{[\cov^{1/2} u]_{2:d}}} \r)}\r).
\end{align*}
We now consider two cases. 

\paragraph{Case 1: $c^{-1} \sqrt{\tr(\cov^2_{2:d})/\lambda_1^2} \leq 1/2$}.  Since the function $\xi \mapsto \xi (1 + \sqrt{\log(1/\xi)})$ is monotone increasing on the interval $(0,1/2]$, for any $\xi, \xi' \in [0,1/2]$
satisfying $\xi\leq \xi'$ we have $\xi (1 + \sqrt{\log(1/\xi)}) \leq \xi' (1 + \sqrt{\log(1/\xi')})$.  By~\eqref{eq:u1.over.cov12u2d} and the case assumption, we have
\begin{equation} \label{eq:case1.key}
\f{ \snorm{[\cov^{1/2} u]_{2:d}}}{\sqrt {\lambda_1} [u]_1} \leq  \sqrt{\f{ c^{-2} \tr(\cov_{2:d}^2)}{\lambda_1^2}} \leq \f 12.
\end{equation}
Thus continuing from~\eqref{eq:subg.testerr.prelim},
\begin{align*}
    \P_{(x,y)\sim\psubgnoise} \big(y \neq \sgn(\sip {u}{x}) \big) &\leq \eta + \f{ c_1 \snorm{[\cov^{1/2} u]_{2:d}} }{\sqrt \lambda_1 [u]_1} \l( 1 + \sqrt{\log\l( 0 \vee \f{ \sqrt{\lambda_1} [u]_1}{\snorm{[\cov^{1/2} u]_{2:d}}} \r)}\r)\\
&\overset{(i)}= \eta + \f{ c_1 \snorm{[\cov^{1/2} u]_{2:d}} }{\sqrt \lambda_1 [u]_1} \l( 1 + \sqrt{\log\l(\f{ \sqrt{\lambda_1} [u]_1}{\snorm{[\cov^{1/2} u]_{2:d}}} \r)}\r)\\
&\overset{(ii)}\leq \eta + c_1 \sqrt{\f{ c^{-2} \tr(\cov^2_{2:d})}{\lambda_1^2}} \l( 1 + \sqrt{ \log \l(  \sqrt{\f{\lambda_1^2}{c^{-2} \tr(\cov^2_{2:d})}}\r) } \r) \\
&\overset{(iii)}\leq \eta + c_1 c^{-1} \sqrt{\f{ \tr(\cov^2_{2:d})}{\lambda_1^2}} \l( 1 + \sqrt{ 0 \vee \f 12 \log \l(  \f{\lambda_1^2}{\tr(\cov^2_{2:d})}\r)} \r).
\end{align*}
Equality $(i)$ uses that $\log(x) \geq 0$ for $x \geq 1$.  Inequality $(ii)$ uses~\eqref{eq:case1.key}.  The final inequality $(iii)$ uses that $c < 1$ and $a \leq a \vee b$ for any $a,b\in \R$.

\paragraph{Case 2: $c^{-1} \sqrt{\tr(\cov^2_{2:d})/\lambda_1^2} > 1/2$.}  In this case it is trivially true that
\begin{align*}
    &\P_{(x,y)\sim\psubgnoise} \big(y \neq \sgn(\sip {u}{x}) \big) \leq \eta + c_1 c^{-1} \sqrt{\f{ \tr(\cov^2_{2:d})}{\lambda_1^2}} \l( 1 + \sqrt{ 0 \vee \f 12 \log \l(  \f{\lambda_1^2}{\tr(\cov^2_{2:d})}\r)} \r),
\end{align*} 
since $c_1\geq 2$ and $c^{-1} \sqrt{\tr(\cov^2_{2:d})/\lambda_1^2} > 1/2$ the right-hand-side is at least 1. 
From this we see that the theorem follows by taking $C' = c_1 c^{-1}$. 
\end{proof}

\subsection{Proof of Corollary~\ref{cor:subgaussian.linearmaxmargin},  Corollary~\ref{cor:subgaussian.leakykkt}, and Corollary~\ref{cor:pgausresult}}

This section contains proofs of Corollary~\ref{cor:subgaussian.linearmaxmargin},  Corollary~\ref{cor:subgaussian.leakykkt}, and Corollary~\ref{cor:pgausresult}.

\subgaussianlinearmaxmargin*
\begin{proof}
By a union bound, both Theorem~\ref{thm:subgaussian.sumsiyixi} and Lemma~\ref{lemma:sg.nearlyorthogonal} hold with probability at least $1-9\delta$ and any $\unifclass$-uniform linear classifier exhibits benign overfitting in the sense described in the theorem, with the noise tolerance determined by $\unifclass$.  Thus, we need only verify that working on this high-probability event and using the assumptions, the linear max-margin solution is $\unifclass$-uniform and that $\unifclass$ is small.

By Lemma~\ref{lemma:sg.nearlyorthogonal}, the training data is $C/C_1$-orthogonal and we have the following upper bound for $\rratio^2$,
\begin{align*}
\rratio^2 = \f{ \max_i \snorm{x_i}^2}{\min_i \snorm{x_i}^2} \leq \l(1 + \f{ C_1}{\sqrt C}\r)^4 \leq \f {100}{99}. \numberthis \label{eq:rratio.sg.ub}
\end{align*}
The last inequality follows by taking $C$ to be a large enough absolute constant.  
Therefore Proposition~\ref{prop:bound.lambdas.max.margin.v2} ensures that the linear max-margin $w$ is $\unifclass$-uniform with $\unifclass = \rratio^2 \l( 1 + \f{2}{p \rratio^2 -2}\r)$.  In particular, 
\begin{align*}
\unifclass &\leq \rratio^2 \l( 1 + \f{ 2}{C\rratio^2/C_1 - 2}\r) \leq \f{201}{198}.
\end{align*}
The final inequality uses that $\rratio^2 \leq 100/99$ and by taking $C>1$ large enough.   Thus the max-margin linear classifier is $\unifclass$-uniform with $\unifclass \leq \f{201}{198}$.  Since $\f 1{2\unifclass} \geq \f{198}{402} \geq 0.492$, if $\eta \leq 0.49 = 0.492 - 0.002$ we can apply Theorem~\ref{thm:subgaussian.sumsiyixi}.  
\end{proof}

\subgaussianleakykkt*
\begin{proof}  
Just as in the proof of the preceding corollary, by a union bound, with probability at least $1-9\delta$ both Theorem~\ref{thm:subgaussian.sumsiyixi} and Lemma~\ref{lemma:sg.nearlyorthogonal} hold and any $\unifclass$-uniform linear classifier exhibits benign overfitting with probability, with the noise tolerance determined by $\unifclass$.   By Lemma~\ref{lemma:sg.nearlyorthogonal}, the training data is $C/C_1$-orthogonal, and thus for $C > 3 C_1 \gamma^{-3}$, we may apply Proposition~\ref{prop:kkt.leaky.nearly.orthogonal} so that $\sgn(f(x;W)) = \sgn(\sip{z}{x})$ where $z$ is $\unifclass$-uniform w.r.t. $\{(x_i, y_i)\}_{i=1}^n$ for $\unifclass = \rratio^2 \gamma^{-2} \l( 1 + \f{ 2}{\gamma C \rratio^2/C_1 -2}\r)$.  Lemma~\ref{lemma:sg.nearlyorthogonal} also implies that $\rratio^2 \leq (1 + C_1/\sqrt C)^4 \leq \f{100}{99}$ for $C$ large enough.  Hence, for $C$ large enough, $\unifclass \leq \f{201}{198}\gamma^{-2}$. Since $\f{1}{2\unifclass} \geq \f{198\gamma^2}{402} > 0.492\gamma^2$, if $\eta \leq 0.49 \gamma^2 = 0.492\gamma^2 - 0.002\gamma^2$ we may apply Theorem~\ref{thm:subgaussian.sumsiyixi} since $\gamma$ is an absolute constant.  
\end{proof}

\pgausresult* 
\begin{proof}
First, it is clear that $\pgaus$ is an instance of $\psubgnoise$, since $\cov^{-1/2} x$ is an isotropic Gaussian which clearly satisfies the anti-concentration property $\P(|\zk| \leq t) \leq \beta t$ for $\beta=1/\sqrt{2\pi}$.   
We thus need only verify that assumptions~\ref{a:samples.sg} through~\ref{a:covnlogn} are satisfied and that $\tr(\cov^2_{2:d})/\lambda_1^2$ is small.  Clearly, $\stablerank(\cov_{2:d})=d-1$ and 
\[ \f{ \tr(\cov)}{\sqrt{\tr(\cov^2)}} = \f{ d^\rho + d-1}{\sqrt{d^{2\rho} + d-1}}.\]
By assumption, $\rho \in (1/2, 1)$, so $d^\rho + d-1 = \Theta(d)$, while $d^{2\rho} + d - 1 = \Theta(d^{2\rho})$.  Therefore,
\[ \f{ \tr(\cov)}{\sqrt{\tr(\cov^2)}} = \Theta(d^{1-\rho}).  \]
Thus, we see that if $n = \tilde \Omega(1)$ and $d = \tilde \Omega(n^{1/(1-\rho)})$, then assumptions~\ref{a:samples.sg} through~\ref{a:covnlogn} are satisfied and hence Theorem~\ref{thm:subgaussian.sumsiyixi} and Corollary~\ref{cor:subgaussian.leakykkt} apply under the stated assumptions on the noise rate $\eta$.   On the other hand,
\[ \f{\tr(\cov^2_{2:d})}{\lambda_1^2} = \f{ d-1}{d^{2\rho}} = \Theta(d^{1-2\rho}). \]
Since $\rho > 1/2$, we see that $\tr(\cov^2_{2:d})/\lambda_1^2 = o_d(1)$, and thus the test error of KKT points of Problem~\eqref{eq:margin.maximization.problem} are at most
\[ \eta + C' \sqrt{\f{ \tr(\cov^2_{2:d})}{\lambda_1^2}} \l( 1 + \sqrt{0 \vee \f 12 \log\l(\f{\lambda_1^2}{\tr(\cov^2_{2:d})} \r)}\r) = \eta +  \tilde O(d^{\f 12(1-2\rho)}) = \eta + o_d(1).\]
\end{proof}

\section{Proofs for clustered data}\label{appendix:clustered}
In this section we provide the proofs for Section~\ref{sec:clustered}.  Our proof strategy mirrors that we used for the proof of Theorem~\ref{thm:subgaussian.sumsiyixi} in Appendix~\ref{appendix:subgaussian}, and can be summarized as follows:
\begin{enumerate}
\item We first show that in order for a linear classifier $x\mapsto \sgn(\sip wx)$ to achieve small test error, it suffices to have $\sip{w}{\yq q \muq q}$ be large and positive for each $q\in Q$.
\item Propositions~\ref{prop:bound.lambdas.max.margin.v2} and~\ref{prop:kkt.leaky.nearly.orthogonal} show that the max-margin solutions for linear classifiers and leaky ReLU networks correspond to $\unifclass$-uniform classifiers when the training data is $\orthog$-orthogonal.  To use this result, we thus need to characterize the norms and pairwise correlations of the examples.  Additionally, note that if $w\in \R^d$ is $\unifclass$-uniform w.r.t. $\{(x_i, y_i)\}_{i=1}^n$, then $w = \summ i n s_i y_i x_i$ for some $s_i>0$.  Thus by the first step above, we see it will be helpful to characterize $\sip{y_ix_i}{\yq q\muq q}$ for samples $i\in [n]$ and clusters $q\in Q$.    Lemma~\ref{lemma:cluster.subgaussian.concentration} provides some initial bounds that help us with these goals, and Lemma~\ref{lemma:cluster.norms.correlations} collects all of the important properties of the training data that we will use.  In particular, Lemma~\ref{lem:cluster.nearly.orthogonal} will follow from Lemma~\ref{lemma:cluster.norms.correlations}, and the test error bound in Theorem~\ref{thm:cluster.sumsiyixi} for $\unifclass$-uniform classifiers will crucially rely on this lemma as well.
\item We then prove Theorem~\ref{thm:cluster.sumsiyixi} by utilizing the above properties. 
\item The proofs of Corollaries~\ref{cor:clusterlinearmaxmargin} and~\ref{cor:clusterleakykkt} then follow by a direct calculation. 
\end{enumerate}

\subsection{Preliminary concentration inequalities}
Our first lemma provides a generalization bound for any linear classifier over $\pclust$.  

\begin{lemma}\label{lemma:cluster.concentration.testerror}
There exists an absolute constant $c>0$ such that if $w\in \R^d$ is such that $\sip{w}{\yqi q\muq q} \geq 0$ for each $q\in [k]$, then
\[ \P_{(x,y)\sim \pclustnoise} \big(y \neq \sgn(\sip{w}{x})\big)\leq \eta +  \f 1 k \summ q k \exp \l( - c \f{ \sip{w}{\muq q}^2}{\snorm{w}^2} \r).\]
\end{lemma}
\begin{proof}
We use an identical proof to that of Lemma~\ref{lemma:sg.concentration.testerror}.  
By definition of $\pclustnoise$, we have $y = \tilde y$ (the `clean' label) with probability $1-\eta$ while $y=-\tilde y$ with probability $\eta$.   Thus we can calculate,
\begin{align*}
\P_{(x,y)\sim \pclustnoise} \big(y \neq \sgn(\sip{w}{x})\big)
&=\P_{(x,y)\sim \pclustnoise} (y \sip{w}{x} < 0) \\
&= \P_{(x,y)\sim \pclustnoise}(y \sip{w}{x}<0,\, y =-\tilde y )\\
&\qquad + \P_{(x,y)\sim \pclustnoise}(y \sip{w}{x}<0,\, y = \tilde y)\\
&\leq \eta + \P_{(x,y)\sim \pclustnoise}(y \sip{w}{x}<0,\, y = \tilde y).\numberthis \label{eq:cluster.coupling}
\end{align*}
We can bound the second term above as follows,
\begin{align*}
\P_{(x,y)\sim \pclustnoise}(y \sip{w}{x}<0,\, y = \tilde y) &= \f 1 k \summ q k \P_{z\sim \pzz} ( \sip{w}{\yqi q\muq q + \yqi q z} < 0) \\
 &= \f 1 k \summ q k \P_{z\sim \pzz} ( \sip{w}{\yqi q z} < - \yqi q\sip{w}{\muq q})\\
 &\leq \f 1 k \summ q k \exp \l( - c \f{ \sip{w}{\muq q}^2}{\snorm{w}^2} \r). 
\end{align*}
In the last inequality we have used that $\yqi q \sip{w}{\muq q}\geq 0$, as well as the fact that $\yqi q z$ is sub-Gaussian (with sub-Gaussian norm at most the absolute constant $\sgnormz$) and Hoeffding's inequality.  
Substituting the above into~\eqref{eq:cluster.coupling} completes the proof. 
\end{proof}

Due to Proposition~\ref{prop:bound.lambdas.max.margin.v2} and~\ref{prop:kkt.leaky.nearly.orthogonal}, we are interested in the behavior of classifiers defined in terms of $w\in \R^d$ that are $\unifclass$-uniform w.r.t. the training data.  Such classifiers take the form $\summ i n s_i y_i x_i$, where $s_i>0$.  By Lemma~\ref{lemma:cluster.concentration.testerror}, to show $x\mapsto \sgn(\sip wx)$ has small generalization error, it therefore helpful to characterize $\sip{y_i x_i}{\muq q}$ for different clusters $q$.  We begin to do so with the following lemma. 

\begin{lemma}\label{lemma:cluster.subgaussian.concentration}
Let $\pzz$ be a distribution such that the components of $z\sim \pzz$ are mean-zero, independent, sub-Gaussian random variables with sub-Gaussian norm at most one; and for some absolute constant $\kappa>0$, $\kappa d \leq \E[\snorm{z}^2]\leq d$.  
Let $\delta \in (0,1)$.  Suppose that $\{z_i\}_{i=1}^n\iid \pzz$, and let $v_1, \dots, v_k$ be any collection of vectors in $\R^d$.  There are absolute constants $C, C_1>1$ such that provided $d\geq C \log(n/\delta)$, the following hold with probability at least $1-4\delta$. 
\begin{enumerate}[label=(\roman*)]
\item For all $i$, 
\[ \kappa d \l( 1 - C_1\sqrt{ \f{ \kappa^{-2} \log(2n/\delta)}d}\r) \leq \snorm{z_i}^2 \leq  d \l( 1 + C_1 \sqrt{\f{\log(2n/\delta)}d} \r) .\]
\item For all $i\neq j$, $|\sip{z_i}{z_j}| \leq C_1 \sqrt{d \log(2n/\delta)}$.
\item For all $i=1,\dots, k$ and $j=1,\dots, n$, $|\sip{v_{i}}{z_j}|\leq C_1 \snorm{v_i} \sqrt{\log(2nk/\delta)}$.
\end{enumerate}
\end{lemma} 
\begin{proof}
We prove the lemma in parts.  We use an identical argument to~\citet[Lemma 16]{chatterji2020linearnoise}. 

For the first part, fix $i\in [n]$.  The quantity $\snorm{z_i}^2$ is a sum of $d$ independent random variables that are squares of sub-Gaussian random variables with norm at most one, and thus by~\citet[Lemma 2.7.6]{vershynin.high}, this is the sum of $d$ sub-exponential random variables with sub-exponential norm at most one.  Thus by Bernstein's inequality (see~\cite[Theorem 2.8.1]{vershynin.high}), there is some absolute constant $c>0$ such that for any $t\geq 0$,
\begin{align*}
\P( |\snorm{z_i}^2 - \E\snorm{z_i}^2 | \geq t) \leq 2 \exp \l( - c \l( t \wedge \f{t^2}{d}\r)\r).
\end{align*}
Choosing $t = c^{-1} \sqrt{d\log(2n/\delta)}$, we see that 
\[ d \geq c^{-2} \log(2n/\delta) \implies  t \wedge t^2/d = c^{-2} \log(2n/\delta).\]
Thus, we have
\begin{align*}
\P\l( \exists i:\ \l| \snorm{z_i}^2 - \E[\snorm{z_i}^2] \r| \geq c^{-1} \sqrt{d\log(2n/\delta)}\r) \leq \delta.
\end{align*}
By assumption, $\kappa d \leq \E[\snorm{z_i}^2] \leq d$.  Using 
\begin{align*} 
\kappa d \l( 1 - c^{-1}\sqrt{ \f{ \kappa^{-2} \log(2n/\delta)}d}\r) &= \kappa d - c^{-1} \sqrt{d\log(2n/\delta)},\\
d + c^{-1} \sqrt{d\log(2n/\delta)} &= d \l( 1 + c^{-1} \sqrt{\f{\log(2n/\delta)}d}\r),
\end{align*} 
we thus have
\begin{align*}\numberthis \label{eq:norm.zi.concentration} 
\P\l( \exists i:\ \kappa d \l( 1 - c^{-1}\sqrt{ \f{ \kappa^{-2} \log(2n/\delta)}d}\r) > \snorm{z_i}^2 \text{ or } \snorm{z_i}^2 > d \l( 1 + c^{-1} \sqrt{\f{\log(2n/\delta)}d}\r)  \r) \leq \delta.
\end{align*}

Next, note that for any $i,j\in [n]$, and any $t\geq 0$,
\begin{align*}
\P(|\sip{z_i}{z_j}| \geq t ) &\leq \P\l(|\sip{z_i}{z_j}|\geq t \big| \snorm{z_j} \leq \sqrt{2d} \r) + \P( \snorm{z_j} > \sqrt{2d}). 
\end{align*}
For $i\neq j$, conditional on $z_j$, since $z_i$ has independent sub-Gaussian components with sub-Gaussian norm at most one, the random variable $\sip{z_i}{z_j}$ is mean-zero sub-Gaussian with sub-Gaussian norm at most $c_1\snorm{z_j}$ for an absolute constant $c_1>0$~\citep[Proposition 2.6.1]{vershynin.high}. Thus by Hoeffding's inequality~\citep[Theorem 2.6.3]{vershynin.high} we have for some absolute constant $c_2>0$, 
\begin{align*}
\P\l( |\sip{z_i}{z_j}| \geq t \big|  \snorm{z_j}\leq \sqrt{2d} \r) \leq 2 \exp \l( - c_2 \cdot \f{ t^2}{2d} \r). 
\end{align*}  
Letting $t = c_2^{-1/2} \sqrt{2d \log(2n^2/\delta)}$ and we see that
\begin{align*}
\P\l( |\sip{z_i}{z_j}| \geq c_2^{-1/2} \sqrt{2d \log(2n^2/\delta)} \big|  \snorm{z_j}\leq \sqrt{2d} \r) \leq \f{\delta}{n^2}. 
\end{align*}  
Using this and~\eqref{eq:norm.zi.concentration}, 
\begin{align*}
&\P(\text{ for some $i\neq j$,} |\sip{z_i}{z_j}| \geq c_2^{-1/2} \sqrt{2d \log(2n^2/\delta) } ) \\
&\quad \leq n^2 \P\l(|\sip{z_i}{z_j}|\geq c_2^{-1/2} \sqrt{2d \log(2n^2/\delta)} \big| \snorm{z_j} \leq \sqrt{2d} \r) + \P( \text{for some $j\in [n]$, }\,  \snorm{z_j} > \sqrt{2d})\\
&\quad \leq 2\delta. \numberthis \label{eq:zi.zj.ub}
\end{align*}
In the last inequality we are using the lemma's assumption that $d \geq 4c^{-2} \log(2n/\delta)$ so that $\{ \snorm{z_i}^2 > \sqrt 2 d \} \subset \{ \snorm{z_i}^2 > d(1 + c^{-1} \sqrt{\log(2n/\delta)/d})\}$.

Finally, for $v \in \{v_1, \dots, v_k\}$ and fixed $j$, since $z_j$ has independent sub-Gaussian components we know $\sip{z_j}{v}$ is a sub-Gaussian random variable with sub-Gaussian norm at most $c_1\snorm{v}$.  Therefore, by Hoeffding's inequality we have for some constant $c_3>0$,
\begin{align*}
\P \l( |\sip{z_j}{v}| \geq t \r) \leq 2 \exp \l( - c_3 \cdot \f{ t^2}{\snorm{v}^2}\r).
\end{align*}
Taking $t = c_3^{-1} \snorm{v} \sqrt{\log(2nk/\delta)}$ and a union bound over $j\in [n]$ and the $k$ possible options for $v$, we see that
\begin{align*}
\P\l( \exists j \in [n], v\in \{v_1, \dots, v_k\} \text{ s.t. } |\sip{z_j}{v}| \geq  c_3^{-1} \snorm{v} \sqrt{\log(2n k /\delta)} \r) \leq \delta.
\end{align*}
Using a union bound with~\eqref{eq:norm.zi.concentration},~\eqref{eq:zi.zj.ub} and the above yields a total failure probability of $4\delta$ and completes the proof. 
\end{proof}

Next, we show how to use the above to say something about the training data.  Recall that we observe samples $\{(x_i, y_i)\}_{i=1}^n \iid \pclustnoise$ which are noisy versions of $\{(x_i, \tilde y_i)\}_{i=1}^n$.   We denote by $\calC\subset [n]$ the clean samples and $\calN\subset [n]$ the noisy examples, so that $\calC \cup \calN = [n] = I$.  In particular, for $i\in \calN$, $y_i = -\tilde y_i$, while for $i\in \calC$, $y_i=\tilde y_i$.  We further use the notation $\clusteri i = q_i$ and $\iq q = \{ i\in I: \clusteri i = q\}$ and
\[ \iqc q:= \{ i\in I \cap \calC : \clusteri i = q\},\quad \iqn q := \{ i\in I \cap \calN: \clusteri i = q\},\]
so that $\iq q = \iqc q \cup \iqn q$.  

\begin{lemma} \label{lemma:cluster.norms.correlations}
There is an absolute constant $C_1'>1$ such that the following holds.  For $C>1$ sufficiently large under Assumptions~\ref{a:samples} through~\ref{a:orthog.mu}, with probability at least $1-7\delta$, items (i) through (iii) of Lemma~\ref{lemma:cluster.subgaussian.concentration} hold (with $v_i = \muq i$ for $i=1, \dots, k$), and we have the following. 
\begin{enumerate}[label=(\roman*)]
\item For all $i$, 
\[ d \l( 1 - C_1' \sqrt{\f{  \log(2n/\delta)}d} \r) \leq \snorm{x_i}^2 \leq d\l( 1 + C_1' \sqrt{\f{ \log(2n/\delta)}d} + \f{2}{Cn}\r).\]
\item For each $q\in Q$ and $i\in \iq q$,
\[ \l| \sip{\muq q}{x_i} - \snorm{\muq q}^2 \r|  \leq C_1' \snorm{\muq q} \sqrt{\log(2nk/\delta)}.\] 
\item For each $q\in Q$, if $i,j\in \iq q$ and $i\neq j$, then
\[ \l| \sip{x_i}{x_j} - \snorm{\muq q}^2 \r| \leq  C_1' \sqrt{d \log(2n/\delta)} .\]
\item For each $q,r\in Q$ with $q\neq r$, if $i\in \iq q$ and $j\in \iq r$, then 
\[ |\sip{x_i}{x_j}| \leq \maxcorrmu + C_1' \sqrt{d\log(2n/\delta)}.\]
\item For all $q\in Q$,
\[ \l| \f{ |\iq q| }{n} - \f{1}{k} \r| \leq \sqrt{\f{\log(2k/\delta)}{n}},\]
and
\[ \l| \f{|\iqn q|}{|\iq q|} - \eta \r| \leq \sqrt{\f{ \log(2k/\delta)}{n}}, \quad \l| \f{ |\iqc q|}{|\iq q|} - (1 - \eta)\r| \leq \sqrt{\f{\log(2k/\delta)}{n}}.\]
\end{enumerate}
\end{lemma} 

\begin{proof}
By definition,
\begin{align*}
\snorm{x_i}^2 = \snorm{z_i}^2 + \snorm{\muq {q_i}}^2 + 2 \sip{z_i}{\muq {q_i}}.
\end{align*}
We first note that since $d \geq C n^2 \log(n/\delta)$ by Assumption~\ref{a:dimension}, with probability at least $1-4\delta$, all of the results of Lemma~\ref{lemma:cluster.subgaussian.concentration} hold, where $v_1, \dots, v_k$ are taken to be the cluster means, $v_i = \muq i$.  We work on this high-probability event in the remainder of the proof.  

By definition, since we have assumed $\E[\snorm{z}^2]=d$,
\begin{align*}
\snorm{x_i}^2 &= \snorm{z_i}^2 + \snorm{\muq {q_i}}^2 + 2 \sip{z_i}{\muq {q_i}} \\
&\geq  d \l( 1 - C_1 \sqrt{\f{  \log(2n/\delta)}d} \r)+ \snorm{\muq {q_i}}^2 - 2C_1 \snorm{\muq {q_i}} \sqrt{\log(2nk/\delta)} \\
&\geq  d \l( 1 - C_1 \sqrt{\f{ \log(2n/\delta)}d} \r).
\end{align*}
where we have used Assumption~\ref{a:min.norm.mu} (for $C>1$ large enough) in the last inequality.  
On the other hand, by Lemma~\ref{lemma:cluster.subgaussian.concentration} we also have,
\begin{align*}
\snorm{x_i}^2 &= \snorm{z_i}^2 + \snorm{\muq {q_i}}^2 + 2 \sip{z_i}{\muq {q_i}} \\
&\leq d \l( 1 + C_1 \sqrt{\f{ \log(2n/\delta)}d} \r) + \snorm{\muq {q_i}}^2 + 2 C_1 \snorm{\muq {q_i}} \sqrt{\log(2nk/\delta)} \\
&\overset{(i)}\leq d \l( 1 + C_1 \sqrt{\f{ \log(2n/\delta)}d} \r) + 2 \snorm{\muq {q_i}}^2 \\
&\overset{(ii)}\leq d \l( 1 + C_1 \sqrt{\f{ \log(2n/\delta)}d} + \f{2}{Cn}\r).
\end{align*}
The inequality $(i)$ uses Assumption~\ref{a:min.norm.mu} and inequality $(ii)$ uses Assumption~\ref{a:dimension}.  

For the second part of the lemma, note that for $i\in \iq q$, $\sip{\muq q}{x_i} - \snorm{\muq q}^2 = \sip{z_i}{\muq q}$.  Lemma~\ref{lemma:cluster.subgaussian.concentration} thus bounds the absolute value of this quantity.

For the third part of the lemma, consider those $i\neq j$ that belong to the same cluster.  For these, we have $\muq {q_i} = \muq {q_j}$ so that
\begin{align*}
\sip{x_i}{x_j} &= \sip{\muq {q_i} + z_i}{\muq {q_j} + z_j} \\
&= \snorm{\muq {q_i}}^2 + \sip{z_i}{\muq {q_j}} + \sip{\muq {q_j}}{z_i} + \sip{z_i}{z_j}.
\end{align*}
By Lemma~\ref{lemma:cluster.subgaussian.concentration}, we thus have
\begin{align*}
|\sip{x_i}{x_j} - \snorm{\muq {q_i}}^2| &\leq |\sip{z_i}{\muq {q_j}}| + |\sip{\muq {q_i}}{z_j}| + |\sip{z_i}{z_j}|\\
&\leq 2 C_1 \maxnormmu \sqrt{\log(2nk/\delta)} + C_1 \sqrt{d\log(2n/\delta)}\\
&\overset{(i)}\leq 2 C_1 \sqrt{\f{ d \log(2nk/\delta)}{Cn}} + C_1 \sqrt{d \log(2n/\delta)} \\
&\overset{(ii)}\leq 2 C_1 \sqrt{d\log(2n/\delta)}.\numberthis \label{eq:xi.xj.ip.intermediate}
\end{align*}
The inequality $(i)$ uses Assumption~\ref{a:dimension}.  Inequality $(ii)$ follows since Assumption~\ref{a:samples}
implies that for $C>1$ large enough, we have $n>10 C_1^2 k$ so that
\begin{align*}
    \log(2nk/\delta) < \log(2n^2/\delta) < 2 \log(2n/\delta). 
\end{align*}

For the fourth part of the lemma, if $i\in \iq q$ and $j\in \iq r$ for $q\neq r$, 
\begin{align*}
|\sip{x_i}{x_j}| &= |\sip{ \muq {q_i} + z_i}{ \muq {q_j} + z_j}| \\
&\leq |\sip{\muq {q_i}}{\muq {q_j}}| + |\sip{z_i}{\muq {q_j}}| + |\sip{\muq {q_i}}{z_j}| + |\sip{z_i}{z_j}| \\
&\leq \maxcorrmu + 2 C_1 \maxnormmu \sqrt{\log(2nk/\delta)} + C_1 \sqrt{d\log(2n/\delta)}\\
&\leq \maxcorrmu + 2 C_1 \sqrt{d\log(2n/\delta)}.
\end{align*}
where the second-to-last inequality uses Lemma~\ref{lemma:cluster.subgaussian.concentration}, and the last inequality uses an identical argument to~\eqref{eq:xi.xj.ip.intermediate}. 

For the last part of the lemma, if $q\in Q$ then the quantity
\[ |\iq q| = \summ i n \ind(\clusteri i = q)\]
is a sum of $n$ i.i.d. random variables with mean $1/k$.  
By Hoeffding's inequality, 
for any $u\geq 0$, 
\begin{align*}
\P\l( \l| |\iq q| - \f{n}{k} \r| \geq u\r) \leq 2 \exp \l( - \f{ 2 u^2}{n} \r). 
\end{align*}
In particular, selecting $u = \sqrt{n \log(2k/\delta)}$ and taking a union bound over the $k$ clusters, we see that with probability at least $1-\delta$, for all $q\in Q$, 
\begin{align*}
\l| \f{|\iq q|}{n} - \f 1 k \r| &\leq \sqrt{\f{ \log(2k/\delta)}{n}}.
\end{align*}

Finally, let us denote by $N_q$ the number of noisy examples within cluster $q$,
\[ |\iqn q| = N_q = \sum_{i\in \iq q} \ind(i\in \calN).\]
Condintioned on the value of $|\iq q|$, since we are considering random classification noise, $N_q$ is the sum of $|\iq q|$ independent, identically distributed random variables with mean
\[ m_q := \P(i\in \calN) = \eta. \] 
By Hoeffding's inequality, 
for any $u\geq 0$, 
\begin{align*}
\P\l( \l|  N_q - |\iq q| m_q \r| \geq u\r) \leq 2 \exp \l( - \f{ 2 u^2}{|\iq q|} \r). 
\end{align*}
In particular, selecting $u = \sqrt{|\iq q| \log(2k/\delta)}$ and taking a union bound over the $k$ clusters, we see that with probability at least $1-\delta$, for all $q\in Q$, 
\begin{align*}
\l| \f{|\iqn q|}{|\iq q|} - \eta \r| &\leq \sqrt{\f{ \log(2k/\delta)}{n}}.
\end{align*}

Since samples are `clean' and in cluster $q$ with probability $1-\eta$, a completely identical argument yields the bound for $|\iqc q|$.  Taking a union bound over the event in Lemma~\ref{lemma:cluster.subgaussian.concentration} and the three events above leads to a total failure probability of $7\delta$. 
\end{proof}

\subsection{Proof of Lemma~\ref{lem:cluster.nearly.orthogonal}}
We now show that under our assumptions on the problem parameters, the training data are $\orthog$-orthogonal for large $p$ and the norms of each example are quite close to each other.

\clusternearlyorthogonal*
\begin{proof}
All of the results of
Lemma~\ref{lemma:cluster.norms.correlations} hold with probability at least $1-7\delta$.  We shall show that the lemma is a deterministic consequence of this high-probability event.

First, if $i,j\in \iq q$ and $i\neq j$, then by Lemma~\ref{lemma:cluster.norms.correlations},
\begin{align*}
|\sip{x_i}{x_j}| &\leq \maxnormmu^2 +   C_1' \sqrt{d\log(2n/\delta)} \leq 2 C_1' \max\l(\maxnormmu^2, \sqrt{d\log(2n/\delta)}\r).
\end{align*}
On the other hand, if $i\in \iq q$ and $j\in \iq r$ with $q\neq r$, then
\begin{align*}
|\sip{x_i}{x_j}| &\leq \maxcorrmu + C_1' \sqrt{d\log(2n/\delta)}\\
&\overset{(i)} \leq \minnormmu^2 + C_1' \sqrt{d\log(2n/\delta)}\\
&\leq 2 C_1' \max\l( \maxnormmu^2, \sqrt{d\log(2n/\delta)}\r),
\end{align*}
where in $(i)$ we use Assumption~\ref{a:orthog.mu}. 
Thus for any $i\neq j$ we have
\begin{align*}
|\sip{x_i}{x_j}| &\leq 2 C_1' \max\l( \maxnormmu^2, \sqrt{d\log(2n/\delta)}\r). \numberthis \label{eq:cluster.xixj.corr}
\end{align*}
On the other hand, by Lemma~\ref{lemma:cluster.norms.correlations} we also have
\begin{align*}
d \l( 1 - C_1' \sqrt{\f{ \log(2n/\delta)}d} \r)\leq \min_i \snorm{x_i}^2 \leq 
\max_i \snorm{x_i}^2 &\leq d \l( 1 + C_1' \sqrt{\f{ \log(2n/\delta)}d} + \f{2}{Cn}\r).\numberthis \label{eq:cluster.rratio.prelim}
 \end{align*}
We can thus bound
\begin{align*}
    \f{ \min_i \snorm{x_i}^4}{\max_i \snorm{x_i}^2} &\geq d \cdot \f{  \l( 1 - C_1' \sqrt{\f{ \log(2n/\delta)}d} \r)^2}{ \l( 1 + C_1' \sqrt{\f{ \log(2n/\delta)}d} + \f{2}{Cn} \r) }\\
    &\overset{(i)}\geq d \cdot \l( 1 - C_1' \sqrt{\f{  \log(2n/\delta)}d} \r)^2 \cdot \l( 1 - C_1' \sqrt{\f{ \log(2n/\delta)}d}  - \f{2}{Cn} \r) \\
    &\overset{(ii)} \geq \f 12 d.  \label{eq:normxto4vsto2}\numberthis
\end{align*}
In inequality $(i)$ we have used that $1/(1+x)\geq 1-x$ for $x>0$, and in inequality $(ii)$ we have taken $C>1$ large enough in Assumption~\ref{a:dimension}.  Thus, we have
\begin{align*}
\f{ \min_i \snorm{x_i}^4}{\max_i \snorm{x_i}^2 \max_{i\neq j}|\sip{x_i}{x_j}|} &\geq \f{  d}{ 2 \max_{i\neq j}|\sip{x_i}{x_j}|} \geq \f{ d }{2 C_1' \max \l( \maxnormmu^2, \sqrt{d\log(2n/\delta)}\r)}.
\end{align*}
Rearranging and using Assumption~\ref{a:dimension}, this implies
\begin{align*}
\min_i \snorm{x_i}^2 &\geq \f{ 1}{2C_1'} \cdot \f{\max_i \snorm{x_i}^2}{\min_i \snorm{x_i}^2} \cdot \f{ d}{\max(\maxnormmu^2, \sqrt{d\log(2n/\delta)})} \cdot \max_{i\neq j} |\sip{x_i}{x_j}|\\
&\geq \f{ C}{2 C_1'} \cdot \f{\max_i \snorm{x_i}^2}{\min_i \snorm{x_i}^2} \cdot n \max_{i\neq j} |\sip{x_i}{x_j}|.
\end{align*}
In particular, the training data is $C/C_2$-orthogonal for $C_2 := 2 C_1'$ (see Definition~\ref{def:nearlyorthogonal}).  Moreover, by~\eqref{eq:cluster.rratio.prelim} we have
\begin{align*}
    \rratio^2 &\leq  \l( 1 + C_1' \sqrt{\f{ \log(2n/\delta)}d} + \f{2}{Cn}\r) \cdot \l( 1 - C_1' \sqrt{\f{ \log(2n/\delta)}d}\r)^{-1} \\
  &\overset{(i)}\leq \l( 1 + C_1'/\sqrt{C} + \f{2}{Cn} \r) \cdot \l( 1 - C_1' / \sqrt{C}\r)^{-1} \\
  &\leq \l( 1 + 2C_1'/\sqrt C\r)^{2}.
\end{align*}
The inequality $(i)$ uses Assumption~\ref{a:dimension}.  The final inequality uses that $\f{C_1'}{\sqrt C} + \f{2}{Cn} \leq \f{2 C_1'}{\sqrt C}$ for $C$ large enough and that $1/(1-x)\leq 1+2x$ for $x\in (0,1/2)$. 
\end{proof}

\subsection{Proof of Theorem~\ref{thm:cluster.sumsiyixi}}

We now show that any $\unifclass$-uniform linear classifier projected onto any direction of the form $\yq q \muq q$ is large and positive.  By Lemma~\ref{lemma:cluster.concentration.testerror}, this will be a key ingredient for a test error bound.

\begin{lemma}\label{lemma:cluster.sumsiyixi.corr}
Let $u\in \R^d$ be $\unifclass$-uniform w.r.t. $\{(x_i,y_i)\}_{i=1}^n$ for some absolute constant $\unifclass\geq 1$.   Let $\Delta>0$ be an absolute constant and assume $\eta \leq \f{1}{1+\unifclass} - \Delta$.   
Then under Assumptions~\ref{a:samples} through~\ref{a:orthog.mu}, provided $C>1$ is a large enough absolute constant (depending only on $\eta$, $\unifclass$, and $\Delta$), then with probability at least $1-7\delta$ over $\pclustnoise^n$, 
for each $q\in Q$, 
\[ \f{ \ip{u}{\yq q \muq q}}{\snorm{u}}  \geq  \f{\sqrt 3(1 + \unifclass) \Delta}{4\sqrt{10}  \unifclass} \cdot \f{ \sqrt n \snorm{\muq q}^2}{ k \sqrt d}.\]
\end{lemma}
\begin{proof} 
First note that with probability at least $1-7\delta$, 
the items in 
both Lemma~\ref{lemma:cluster.norms.correlations} and Lemma~\ref{lemma:cluster.subgaussian.concentration} (with the vectors $v_i = \muq i$) hold.  We also showed that Lemma~\ref{lem:cluster.nearly.orthogonal} holds as a deterministic consequence of these lemmas.  
In the remainder of the proof, we will work on this high-probability event and show that Lemma~\ref{lemma:cluster.sumsiyixi.corr} follows as a deterministic consequence of Lemmas~\ref{lemma:cluster.subgaussian.concentration},~\ref{lemma:cluster.norms.correlations}, and~\ref{lem:cluster.nearly.orthogonal} under  Assumptions~\ref{a:samples} through~\ref{a:orthog.mu}. 

Since $u$ is $\unifclass$-uniform, there are strictly positive numbers $s_i$ such that $u = \summ i n s_i y_i x_i$ and $\max_{i,j} \nicefrac{s_i}{s_j} = \unifclass$.  Our proof consists in two parts: first, we want to show that for each $q$, the quantity
\begin{align*}
\sip{u}{\yqi q\muq q} &= \ip{\summ i n s_i y_i x_i}{\yqi q \muq q} \\
&= \summ r k \sum_{i\in \iq r} s_i \sip{y_i x_i}{\yqi q\muq q}\\
&= \sum_{i\in \iq q} s_i \sip{y_i x_i}{\yqi q \muq q} + \sum_{r\neq q} \sum_{i\in \iq r} s_i \sip{y_i x_i}{\yqi q\muq q}
\end{align*}
is large.  We will do so by considering the two terms above.  Intuitively, when $i\in \iq q$ then the summands in the first term $\sip{y_i x_i}{\yq q \muq q}$ will be large and positive for clean points $i\in \iqc q$ and negative for noisy points $i\in \iqn q$, and so as long as there are more clean points than noisy ones, the first term will be large and positive.  For the second term above, this term will not be too large in absolute value since the clusters are nearly-orthogonal.   After we show that the above holds, we then want to provide an upper bound on $\snorm{u}^2$.  

We will first show that the quantity $\sip{u}{\yq q \muq q}$ is large and positive by considering the two terms in the above decomposition separately. 

\paragraph{First term: $i\in \iq q$.} In this case, we have $\muq {q_i} = \muq q$.  If $i\in \iqc q$, then $y_i= \yqi q$, while if $i\in \iqn q$, then $y_i = -\yqi q$.  We will thus show a positive lower bound for clean points and an upper bound on the absolute value of noisy points.  

We first provide a lower bound for clean samples $i\in \calC$.  For such samples, $x_i = \muq {q_i} + z_i$ and $y_i = \tyj {q_i} = \tilde y$ and so
\begin{align*}
\sip{s_i y_i x_i}{ \yq q \muq q} &= s_i \sip{\muq {q} + z_i}{\muq q} \\
&\geq s_i \l[ \snorm{\muq q}^2 - |\sip{z_i}{\muq q}|\r]\\
&= s_i\snorm{\muq q}^2 \l( 1 -  \f{ |\sip{z_i}{\muq q}|}{\snorm{\muq q}^2} \r)\\
&\overset{(i)}\geq s_i \snorm{\muq q}^2 \l( 1 - \f{ C_1 \sqrt{\log(2nk/\delta)}}{\snorm{\muq q}} \r).
\end{align*}  
Inequality $(i)$ uses Lemma~\ref{lemma:cluster.subgaussian.concentration}.
Using an identical sequence of calculations, we can derive a similar upper bound for $|\sip{s_i y_i x_i}{ \tilde y (\muq q + z)}|$ for noisy examples: we have for $i\in \iqn q$,
\begin{align*}
|\sip{s_i y_i x_i}{ \tilde y \muq q}| &=  s_i |\sip{\muq {q} + z_i}{\muq q}| \\
&\leq s_i \l[ \snorm{\muq q}^2 + |\sip{z_i}{\muq q}|\r]\\
&\leq s_i \snorm{\muq q}^2 \l( 1 + \f{ C_1 \sqrt{\log(2nk/\delta)}}{\snorm{\muq q}} \r).
\end{align*}
Putting the two preceding displays together, we get,
\begin{align*}\numberthis \label{eq:cluster.q.equals.qi}
\begin{cases}
s_i\sip{ y_i x_i}{ \yq q \muq q} \geq s_i \snorm{\muq q}^2 \cdot \l( 1 - C_1 \sqrt{ \f{  \log(2nk/\delta)}{\snorm{\muq q}^2}}\r), & i\in \iqc q,\\
|s_i\sip{ y_i x_i}{ \yq q \muq q}| \leq s_i \snorm{\muq q}^2 \cdot \l( 1 + C_1 \sqrt{ \f{ \log(2nk/\delta)}{\snorm{\muq q}^2}}\r), & i\in \iqn q.
\end{cases}
\end{align*}

\paragraph{Second term: $i\in \iq r$, $r\neq q$.}  Since $\muq {q_i} \neq \muq q$, we have for both noisy and clean examples,
\begin{align*}
|s_i \sip{y_i x_i}{ \yq q\muq q}| &= s_i|\sip{\muq {r} + z_i }{\muq q }| \\
&\leq s_i  \l(|\sip{z_i}{\muq q}| + \maxcorrmu \r) \\
&\overset{(i)}\leq s_i \l( C_1 \snorm{\muq q} \sqrt{\log(2nk/\delta)}+ \maxcorrmu\r)  \numberthis 
\label{eq:cluster.q.notequals.qi}
\end{align*} 
Inequality $(i)$ uses Lemma~\ref{lemma:cluster.subgaussian.concentration}.  Putting the above together, we get, 
\begin{align*}
&\sip{u}{\yq q \muq q} \\
&= \sum_{i\in \iq q} s_i \sip{y_i x_i}{\yq q \muq q} + \sum_{r\neq q} \sum_{i\in \iq r} s_i \sip{y_i x_i}{\yq q \muq q} \\
&\geq \sum_{i\in \iqc q} s_i\snorm{\muq q}^2 \l( 1 - C_1\sqrt{ \f{ \log(2nk/\delta)}{\snorm{\muq q}^2}} \r) - \sum_{i\in \iqn q} s_i \snorm{\muq q}^2 \l( 1 + C_1 \sqrt{ \f{ \log(2nk/\delta)}{\snorm{\muq q}^2}} \r) \\
&\qquad - \sum_{r\neq q} \sum_{i\in\iq r} s_i \l( C_1 \snorm{\muq q} \sqrt{\log(2nk/\delta)} + \maxcorrmu \r) \\
&\geq \l(\min_i s_i\r) |\iqc q|  \snorm{\muq q}^2 \l( 1 - C_1 \sqrt{ \f{ \log(2nk/\delta)}{\snorm{\muq q}^2}} \r) - \l(\max_i s_i \r)|\iqn q| \snorm{\muq q}^2 \l( 1 + \sqrt{ \f{ \log(2nk/\delta)}{\snorm{\muq q}^2}} \r) \\
&\qquad -\l(\max_i s_i\r) \cdot \l( n - |\iq q| \r) \cdot \l( C_1 \snorm{\muq q} \sqrt{\log(2nk/\delta)} + \maxcorrmu\r).\numberthis 
\end{align*}
For notational simplicity let us define
\begin{align} \nu := C_1 \sqrt{ \f{ \log(2nk/\delta)}{\snorm{\muq q}^2}} \ll 1, \label{eq:cluster.nu.def}
\end{align} 
where $\nu$ small follows by Assumption~\ref{a:min.norm.mu}.  
Since $\unifclass := \max_i s_i/\min_i s_i$ and $|\iq q| = |\iqc q| + |\iqn q|$, we can then write the above inequality as
\begin{align*}
&\sip{u}{\yq q \muq q} \\
&\geq \l( \min_i s_i \r) \cdot \l( |\iq q| - |\iqn q|\r) \cdot \snorm{\muq q}^2 ( 1 - \nu) - \l( \max_i s_i \r) \cdot |\iqn q| \cdot \snorm{\muq q}^2 (1 + \nu) \\
&\qquad - \l( \max_i s_i \r) \cdot  (n - |\iq q| ) \cdot \l( \snorm{\muq q}^2 \nu  + \maxcorrmu \r) \\
&= \l(\min_i s_i \r)|\iq q|\snorm{\muq q}^2 \Bigg[ 1-\nu -(1-\nu) \f{ |\iqn q| }{|\iq q|} - (1+\nu)\unifclass \cdot \f{ |\iqn q|}{|\iq q|}  \\
&\qquad - \unifclass \l( \f{ n}{|\iq q|} - 1 \r)  \cdot \l( \nu + \f{ \maxcorrmu}{\snorm{\muq q}^2} \r) \Bigg]\\
& = \l(\min_i s_i \r)|\iq q|\snorm{\muq q}^2 \Bigg[ 1 - (1 + \unifclass) \cdot \f{ |\iqn q|}{|\iq q|} - \l( 1 - (1-\unifclass) \cdot \f{ |\iqn q|}{|\iq q|} \r) \nu \\
&\qquad - \unifclass \l( \f{ n}{|\iq q|} -1 \r) \cdot \l( \nu + \f{ \maxcorrmu}{\snorm{\muq q}^2} \r) \Bigg].\numberthis  \label{eq:v.dot.yqmuq.intermediate}
\end{align*}
From here we see we need to control $|\iq q|$ and $|\iqn q|$. 
Using Lemma~\ref{lemma:cluster.norms.correlations}, we have
\begin{align*}
\l| \f{ |\iq q|}{n} - \f 1 k \r| &\leq \sqrt{\f{ \log(2k/\delta)}n},\qquad 
\l| \f{ |\iqn q|}{|\iq q|} - \eta  \r| \leq \sqrt{\f{ \log(2k/\delta)}{n}}.
\end{align*}
In particular, we have
\begin{align*}\numberthis \label{eq:iq.lb}
|\iq q| \geq \f n k - \sqrt{n\log(2k/\delta)} = \f n k \l( 1 - \sqrt{\f{ k^2 \log(2k/\delta)}n} \r) \overset{(i)}\geq \f{ n}{2k},
\end{align*}
where inequality $(i)$ uses Assumption~\ref{a:samples} so that $n \geq 4 k^2 \log(2k/\delta)$. We therefore have
\begin{align*}
\f{ n}{|\iq q|} - 1 \leq 2k.  
\end{align*}
Substituting these inequalities into~\eqref{eq:v.dot.yqmuq.intermediate} and using that $(1 - (1-\unifclass)|\iqn q|/|\iq q|) \leq \tau$,
we get,
\begin{align*}
\sip{u}{\yq q \muq q} &\geq |\iq q| \snorm{\muq q}^2\l(\min_i s_i \r) \cdot \Bigg[ 1 - (1 + \unifclass) \cdot \l( \eta + \sqrt{\f{\log(2k/\delta)}n} \r) -  \unifclass \nu \\
&\qquad - 2k\unifclass \cdot \l( \nu + \f{ \maxcorrmu}{\snorm{\muq q}^2} \r) \Bigg]~. \numberthis \label{eq:u.dot.ymuq.prelim}
\end{align*}

\paragraph{Algebraic calculations to finish the bound on $\sip{u}{\yq q\muq q}$. }  We now want to show that the quantity appearing in the brackets in~\eqref{eq:u.dot.ymuq.prelim} is positive.  
Since by assumption $\eta \leq \f 1{1 + \unifclass} - \Delta$ for some absolute constant $\Delta > 0$, 
\begin{align*}
&1 - (1 + \unifclass) \cdot \l( \eta + \sqrt{\f{\log(2k/\delta)}n} \r) - \unifclass \nu- 2k\unifclass \cdot \l( \nu + \f{ \maxcorrmu}{\snorm{\muq q}^2} \r) \\
&\qquad \geq 1 - (1 + \unifclass) \cdot \l( \f{1}{1+\unifclass} - \Delta + \sqrt{\f{\log(2k/\delta)}n} \r) -\unifclass  \nu - 2k\unifclass \cdot \l( \nu + \f{ \maxcorrmu}{\snorm{\muq q}^2} \r) \\
&\qquad =  (1+\unifclass) \cdot \Bigg[ \Delta - \sqrt{\f{\log(2k/\delta)}n} - \f{ \unifclass \nu }{1+\unifclass} -\f{ 2k\unifclass}{1+\unifclass} \cdot \l( \nu + \f{ \maxcorrmu}{\snorm{\muq q}^2} \r) \Bigg]\\
&\qquad = (1+\unifclass) \Delta  \cdot \Bigg[ 1 - \sqrt{\f{\Delta^{-2} \log(2k/\delta)}n} - \f{ \unifclass \nu\Delta^{-1} }{1+\unifclass} - \f{ 2k\unifclass \Delta^{-1} }{1+\unifclass} \cdot \l( \nu + \f{ \maxcorrmu}{\snorm{\muq q}^2} \r) \Bigg]~. \numberthis\label{eq:v.dot.yqmuq.intermediate2}
\end{align*}
For the second term in the brackets, Assumption~\ref{a:samples} implies
\[ \sqrt{\f{ \Delta^{-2} \log(2k/\delta)}n} \leq \f 1 8.\]
For the next two terms, note that $\f{  \unifclass}{1+\unifclass}  \leq 1$ since $\unifclass \geq 1$. 
Since $\nu = C_1 \sqrt{\log(2nk/\delta) / \snorm{\muq q}^2}$, the second term can be driven to zero by taking $C>1$ sufficiently large by Assumption~\ref{a:min.norm.mu} (namely, $\minnormmu \geq C k \sqrt{ \log(2nk/\delta)}$):
\begin{align*} 
\f{ \tau \nu \Delta^{-1}  }{1+\tau}  &\leq  \Delta^{-1} \cdot C_1 \sqrt{\f{\log(2nk/\delta)}{\snorm{\muq q}^2}}\leq  \f 18.
\end{align*}
Again using Assumption~\ref{a:min.norm.mu},  for $C>1$ large enough we have,
\begin{align*}
\f{ 2k \unifclass \Delta^{-1}  }{1+\unifclass} \cdot\nu \leq  2k\Delta^{-1}\cdot C_1 \sqrt{\f{\log(2nk/\delta)}{\snorm{\muq q}^2}} \leq \f {1}{32}. 
\end{align*}
Finally, Assumption~\ref{a:orthog.mu} implies that for $C>1$ large enough,
\begin{align*}
\f{ 2k \unifclass \Delta^{-1}  }{1+\unifclass} \cdot \f{ \maxcorrmu}{\norm{\muq q}^2} \leq \f 1 {32}.
\end{align*}
Putting the above into~\eqref{eq:v.dot.yqmuq.intermediate2}, we get
\begin{align*}
&1 - (1 + \unifclass) \cdot \l( \eta + \sqrt{\f{\log(2k/\delta)}n} \r) - \unifclass \nu- 2k\unifclass \cdot \l( \nu + \f{ \maxcorrmu}{\snorm{\muq q}^2} \r) \\
&\quad \geq (1+\unifclass)\Delta\l(1 - \f 18- \f 18 -\f 1{32} - \f 1{32}\r) > \f{(1+\unifclass)\Delta}2.
\end{align*}
Substituting this into~\eqref{eq:u.dot.ymuq.prelim}, we get
\begin{align*}
\sip{u}{\yq q \muq q} &\geq |\iq q| \snorm{\muq q}^2\l(\min_i s_i \r) \cdot \Bigg[ 1 - (1 + \unifclass) \cdot \l( \eta + \sqrt{\f{\log(2k/\delta)}n} \r) -  \unifclass \nu \\
&\qquad - 2k\unifclass \cdot \l( \nu + \f{ \maxcorrmu}{\snorm{\muq q}^2} \r) \Bigg]\\
&\geq \f 1 2 (1+\unifclass) \Delta |\iq q| \snorm{\muq q}^2 \l( \min_i s_i \r) \\
&\overset{(i)} \geq \f{ (1+\unifclass)n \snorm{\muq q}^2 \Delta (\min_i s_i)}{4k}.\numberthis \label{eq:v.dot.yqmuq.lb.final}
\end{align*}
The inequality $(i)$ uses~\eqref{eq:iq.lb}.  
This provides the requisite lower bound for $\sip{u}{\yq q\muq q}$.  

\paragraph{Upper bound on $\snorm{u}$.}  Here we use the fact that the samples are nearly-orthogonal: we have,
\begin{align*}
\norm{\summ i n s_i y_i x_i}^2 &\leq \summ i n s_i^2 \snorm{x_i}^2 + \sum_{i\neq j} s_i s_j |\sip{x_i}{x_j}|\\
&\leq n \l( \max_i s_i^2 \r) \l(  \max_i \snorm{x_i}^2 \r) + n^2 \l( \max_i s_i^2\r) \max_{i\neq j} |\sip{x_i}{x_j}| \\
&= n \l( \max_i s_i^2 \r) \l( \max_i \snorm{x_i}^2 + n \max_{i\neq j} |\sip{x_i}{x_j}| \r) \\
&\overset{(i)}\leq \f 5 4 n \l( \max_i s_i^2 \r) \l( \max_i \snorm{x_i}^2\r)\\
&\overset{(ii)}\leq \f 5 4 n \l( \max_i s_i^2 \r) \cdot d \l( 1 + C_1 \sqrt{\f{\log(2n/\delta)}d} + \f{2}{Cn} \r) \\
&\overset{(iii)}\leq \f {10}{3} n d \max_i s_i^2.\numberthis \label{eq:norm.sum.yixi.ub}
\end{align*}
Inequality $(i)$ above uses Lemma~\ref{lem:cluster.nearly.orthogonal}.
Inequality $(ii)$ uses Lemma~\ref{lemma:cluster.norms.correlations}, and inequality $(iii)$ follows by taking $C>1$ large enough by Assumptions~\ref{a:samples} and~\ref{a:dimension}.  Putting~\eqref{eq:norm.sum.yixi.ub} and~\eqref{eq:v.dot.yqmuq.lb.final} together, we get,
\begin{align*}
\f{ \ip{u}{\yq q \muq q}^2}{\snorm{u}^2} &\geq \f{ (1+\unifclass)^2 n^2 \snorm{\muq q}^4 \Delta^2 \min_i s_i^2}{16k^2 \cdot \f{10}{3} nd \max_i s_i^2} = \f{3(1 + \unifclass)^2 \Delta^2}{160 \unifclass^2} \cdot \f{ n \snorm{\muq q}^4}{k^2 d}.
\end{align*}
Taking square roots of the above completes the proof. 
\end{proof}

Putting together Lemma~\ref{lemma:cluster.sumsiyixi.corr} and Lemma~\ref{lemma:cluster.concentration.testerror}, we can derive a generalization bound for the linear classifier $\summ i n s_i y_i x_i$.

\clustersumsiyixi*
\begin{proof}
By a union bound, with probability at least $1-14\delta$, the results of Lemmas~\ref{lemma:cluster.sumsiyixi.corr} and Lemma~\ref{lem:cluster.nearly.orthogonal} hold. 
In the remainder of the proof we will work on this high-probability event and show that the theorem is a deterministic consequence of it and the Assumptions~\ref{a:samples} through~\ref{a:orthog.mu}. 

Since $u$ is $\unifclass$-uniform, there are strictly positive numbers $s_i$ such that $u = \summ i n s_i y_i x_i$.  We shall first show this estimator interpolates the training data.  An identical calculation used as in~\eqref{eq:u.interpolate.intermediate} shows that 
\begin{align*}
\sip{u}{y_k x_k} &= s_k \snorm{x_k}^2 + \sum_{i\neq k} \sip{s_i y_i x_i}{y_k x_k} \\
&\geq s_k \snorm{x_k}^2 \l( 1 - \f{ n \unifclass \max_{i\neq j} |\sip{x_i}{x_j}|}{\snorm{x_k}^2} \r)\\
&\overset{(i)}\geq s_k \snorm{x_k}^2 \l( 1 - \f{ C_2 \unifclass }{C} \r)\\
&\geq \f 12 s_k \snorm{x_k}^2 > 0. 
\end{align*}
The inequality $(i)$ uses that the training data is $C/C_2$-orthogonal by Lemma~\ref{lem:cluster.nearly.orthogonal}, and we took $C$ large relative to the absolute constants $C_2,\tau$.

We now show the generalization error is close to the noise rate.  Since $\eta \leq \f{1}{1+\unifclass}-\Delta$, by Lemma~\ref{lemma:cluster.sumsiyixi.corr}, we know that for each $q$ we have,
\[ \f{ \ip{\summ i n s_i y_i x_i}{\yq q \muq q}}{\snorm{\summ i n s_i y_i x_i}}  \geq  \f{\sqrt 3(1 + \unifclass) \Delta}{4\sqrt{10}  \unifclass} \cdot \f{ \sqrt n \snorm{\muq q}^2}{ k \sqrt d}.\]
Now using Lemma~\ref{lemma:cluster.concentration.testerror}, this implies that
\begin{align*}
\P_{(x,y)\sim\pclustnoise} \big(y \neq \sgn(\sip {\hat \mu}{x}) \big) &\leq \eta + \f 1 k \summ qk  \exp \l( - \f{ 3c(1+\unifclass)^2n \Delta^2 \snorm{\muq q}^4}{160 \unifclass^2 k^2 d}\r) \\
&\leq  \eta + \exp \l( - \f{ n \minnormmu^4}{C'k^2 d}\r),
\end{align*}
where $C'$ is an absolute constant independent of $d$ and  $n$. 
\end{proof}

\subsection{Proof of Corollary~\ref{cor:clusterlinearmaxmargin} and Corollary~\ref{cor:clusterleakykkt}}
In this section we show how to use Theorem~\ref{thm:cluster.sumsiyixi} and Lemma~\ref{lem:cluster.nearly.orthogonal} to prove Corollary~\ref{cor:clusterlinearmaxmargin} and Corollary~\ref{cor:clusterleakykkt}. 

\clusterlinearmaxmargin* 
\begin{proof}
The calculation is essentially identical to that used for the proof of Corollary~\ref{cor:subgaussian.linearmaxmargin}.  By a union bound, the results of Theorem~\ref{thm:cluster.sumsiyixi} and Lemma~\ref{lem:cluster.nearly.orthogonal} hold with probability at least $1-21\delta$, and any $\unifclass$-uniform linear classifier exhibits benign overfitting with noise tolerance determined by $\unifclass$.  We therefore verify that the linear max-margin classifier is $\unifclass$-uniform with small $\unifclass$. 

By Lemma~\ref{lem:cluster.nearly.orthogonal}, the training data is $C/C_2$-orthogonal and $\rratio^2 = \max_{i,j} \nicefrac{\snorm{x_i}^2}{\snorm{x_j}^2} \leq (1 + C_2/\sqrt C)^2 $.  Since for $C$ large enough we have $C/C_2\geq 3$, by Proposition~\ref{prop:bound.lambdas.max.margin.v2} this means
the linear max-margin $w$ is $\unifclass$-uniform with $\unifclass \leq \rratio^2 \l( 1 + \f{2}{C C_2^{-1} \rratio^2 -2}\r)$.  In particular, we have
\begin{align*}
\unifclass \leq \l( 1 + \f{C_2}{\sqrt C}\r)^2 \cdot  \l(1 + \f{2}{C C_2^{-1} \rratio^2 -2}\r)^2 \leq \f{100}{99} \cdot \f{ 201}{200} = \f {201}{198}.
\end{align*}
The final inequality follows by taking $C>1$ a large enough absolute constant.   Thus the max-margin linear classifier is $\unifclass$-uniform where $\unifclass \leq \f{201}{198}$.  Since $\f{1}{1+\unifclass} \geq \f{198}{399} \geq 0.496$, by taking $\eta \leq 0.49 = 0.496 - 0.006$ we may apply Theorem~\ref{thm:cluster.sumsiyixi}.  
\end{proof}

Finally, we prove Corollary~\ref{cor:clusterleakykkt}, again re-stated for convenience. 

\clusterleakykkt*
\begin{proof} 
As in the preceding corollary, with probability at least $1-21\delta$ the results of Theorem~\ref{thm:cluster.sumsiyixi} and Lemma~\ref{lem:cluster.nearly.orthogonal} hold and any $\unifclass$-uniform linear classifier exhibits benign overfitting with noise tolerance determined by $\unifclass$.  By Lemma~\ref{lem:cluster.nearly.orthogonal}, the training data is $C/C_2$-orthogonal for $C>3C_2\gamma^{-3}$ we may apply Proposition~\ref{prop:kkt.leaky.nearly.orthogonal} so that $\sgn(f(x;W)) = \sgn(\sip{z}{x})$ where $z$ is $\unifclass$-uniform w.r.t. the training data for $\unifclass = \rratio^2 \gamma^{-2} \l( 1 + \f{2}{\gamma C \rratio^2/C_2 -2}\r)$.  Lemma~\ref{lem:cluster.nearly.orthogonal} shows that $\rratio^2 \leq \f{100}{99}$ for $C$ large enough, and hence $\unifclass \leq \f{201}{198}\gamma^{-2}$ for large $C$. 
Note that $\f{1}{1+201\gamma^{-2}/198} \geq 0.496\gamma^2$. Hence, 
we may apply Theorem~\ref{thm:subgaussian.sumsiyixi} 
with $\eta \leq 0.49\gamma^2 = 0.496\gamma^2 -0.006\gamma^2$
since $\gamma$ is an absolute constant.  
\end{proof}

\printbibliography
\end{document}